%% file: main.tex
\PassOptionsToPackage{nosubfloats}{jmlrutils}
\documentclass[final,12pt]{colt2026}
\input{macro.tex}
\title[Sign-Based Optimizers Are Effective Under Heavy-Tailed Noise]{Sign-Based Optimizers Are Effective Under Heavy-Tailed Noise}
\usepackage{times}

\coltauthor{
    \Name{Dingzhi Yu}\textsuperscript{1,2} \Email{yudz@lamda.nju.edu.cn}\\
    \Name{Hongyi Tao}\textsuperscript{1} \Email{221220032@smail.nju.edu.cn}\\
    \Name{Yuanyu Wan}\textsuperscript{3} \Email{wanyy@zju.edu.cn}\\
    \Name{Luo Luo}\textsuperscript{4} \Email{luoluo@fudan.edu.cn}\\
    \Name{Lijun Zhang}\textsuperscript{1,2} \Email{zhanglj@lamda.nju.edu.cn}\\
    \addr \textsuperscript{1}State Key Laboratory of Novel Software Technology, Nanjing University\\
    \addr \textsuperscript{2}School of Artificial Intelligence, Nanjing University\\
    \addr \textsuperscript{3}School of Software Technology, Zhejiang University\\
    \addr \textsuperscript{4}School of Data Science, Fudan University
}

\begin{document}

\maketitle

\begin{abstract}%
  While adaptive gradient methods are the workhorse of modern machine learning, sign-based optimization algorithms such as Lion and Muon have recently demonstrated superior empirical performance over AdamW in training large language models (LLM). However, a theoretical understanding of why sign-based updates outperform variance-adapted methods remains elusive. In this paper, we aim to bridge the gap between theory and practice through the lens of heavy-tailed gradient noise, a phenomenon frequently observed in language modeling tasks. Theoretically, we introduce a novel generalized heavy-tailed noise condition that captures the behavior of LLMs more accurately than standard finite variance assumptions. Under this noise model, we establish sharp convergence rates of SignSGD and Lion for generalized smooth function classes, matching or surpassing previous best-known bounds. Furthermore, we extend our analysis to Muon and Muonlight, providing what is, to our knowledge, the first rigorous analysis of matrix optimization under heavy-tailed stochasticity. These results offer a strong theoretical justification for the empirical superiority of sign-based optimizers, showcasing that they are naturally suited to handle the noisy gradients associated with heavy tails. Empirically, LLM pretraining experiments validate our theoretical insights and confirm that our proposed noise models are well-aligned with practice.
\end{abstract}

\begin{keywords}%
  Heavy-tailed noise, sign-based methods, Lion and Muon, matrix optimizers, concentration inequalities, LLM pretraining %
\end{keywords}
\section{Introduction}

The success of large foundation models~\citep{devlin-etal-2019-bert,brown2020gpt3,achiam2023gpt4,touvron2023llama,touvron2023llama2,team2023gemini,guo2025deepseek} requires costly pretraining procedures that are supported by stochastic optimization algorithms~\citep{robbins1951stochastic}. Among them, adaptive gradient methods~\citep{duchi2011adaptive,kingma15adam} prevail, in which AdamW~\citep{loshchilov2019adamw} is the \emph{de facto} optimizer for training large language models (LLM). Recently, two representative sign-based optimizers, namely Lion~\citep{chen2023symbolic} and Muon~\citep{jordan2024muon} have received significant attention due to their consistent speedup over AdamW~\citep{zhaoICLR2025deconstructing,shah2025practical,wen2025fantastic,semenov2025benchmarking}, especially in scaling up both LLM pretraining~\citep{liu2025muon,team2025kimi,zeng2025glm,zeng2026glm,cheng2026ngram} and post-training~\citep{kimi2026kimi2.5}. Despite their huge empirical success, we still lack a complete theoretical understanding of why Lion and Muon outperform AdamW for training LLMs.

In this work, we systematically investigate this phenomenon with theoretical analysis and empirical study. The significant empirical evidence that initiates our work comes from the widely observed heavy-tailed distribution of the stochastic gradients
~\citep{simsekli19tailindex,zhang2020whyheavytail,gurbuzbalaban2021heavy,battash2024revisiting}, especially in language modeling tasks~\citep{kunstner2023noise,kunstner2024heavytail,ahn2024linearattention,kunstner2025scaling,yadav2025provable} due to Zipf's Law~\citep{piantadosi2014zipf}. In sharp contrast to the finite variance condition commonly employed in theoretical studies~\citep{ghadimi2013stochastic,lan2020first}, such a regime only assumes the stochastic noise has a finite $p$th moment where $p\in(1,2]$. Under this more realistic noise model, the challenges underlying both theoretical analysis and algorithmic design arise. For instance, SGD is known to diverge~\citep{zhang2020whyheavytail}, and the convergence theory collapses when $p<2$.\footnote{This does not contradict~\citet{fatkhullin2025can,liu2026old}, as they impose heavy-tailedness on the stochastic gradient itself. This is a strictly stronger condition than the heavy-tailed noise model considered in this work; see~\cref{sec:related-work} for details.} As for the more popular adaptive methods like Adam~\citep{kingma15adam} and AdaGrad~\citep{duchi2011adaptive}, \textit{no rigorous convergence theory has yet been established under heavy-tailed noise.} To make things worse,\textbf{\textit{~\citet{pmlr-v267-chezhegov25a} show that Adam and AdaGrad could suffer from worse convergence when the noise is heavy-tailed.}} Intriguingly, the empirical advantage of sign-based methods like Lion and Muon is most pronounced in LLM training regimes~\citep{liu2025muon,wen2025fantastic}, where heavy-tailed noise is prevalent. In contrast, this advantage diminishes when it comes to computer vision tasks~\citep{yuan25mars,liu2025mars}, which typically exhibit more concentrated gradient noise~\citep{zhang2020whyheavytail,zhou2020towards,gurbuzbalaban2021heavy,ahn2024linearattention}. Regarding this disparity, it is natural to ask the following question: 
\vspace{-0.5em}
\begin{center}
\setlength{\fboxsep}{3pt}
\shadowbox{%
\begin{minipage}[t]{0.9\columnwidth}
Can we obtain theoretical guarantees of sign-based optimizers to explain their practical effectiveness and justify their advantage over AdamW under heavy-tailed gradient noise?
\end{minipage}}
\end{center}
\vspace{-0.8em}
We answer the above question affirmatively by establishing new convergence theories for a series of representative vector- and matrix-based sign descent methods, including vector sign optimizer SignSGD~\citep{bernstein2018signsgd} and Lion~\citep{chen2023symbolic}, matrix sign optimizer Muon~\citep{jordan2024muon} and Muonlight~\citep{liu2025muon}. For a heavy-tailed index $p$ and an iteration number $T$, we show that these optimizers converge at $O(T^{-\frac{p-1}{3p-2}})$, matching or surpassing the previous state-of-the-art. Our convergence analysis is derived under a newly proposed generalized heavy-tailed noise condition, allowing the noise level to grow w.r.t.~the gradient norm, which is strictly weaker than existing assumptions and is further supported by empirical evidence from LLM pretraining. Furthermore, our theory holds for generalized smooth function classes~\citep{Zhang2020Why}, which have become a standard deployment in theory~\citep{zhang2020improved,chen2023generalized,pmlr-v195-faw23a,pmlr-v195-wang23a,Li23adam} and proven to align well with practice~\citep{riabinin2025gluon}. The key technical tools that facilitate our analysis are \textbf{new vector and matrix martingale concentration inequalities in non-Euclidean norms} (cf.~\cref{sec:concentration-inequality}), which may be of independent interest. Our results suggest that the sign operation, whether applied element-wise or via orthogonalization in the matrix case, acts as a natural robustifier against heavy-tailed stochasticity, thus making sign-based optimizers like Lion and Muon a natural fit for training LLMs. 

Our main contributions are summarized as follows.
\begin{itemize}[leftmargin=2.25em,itemsep=0.05em, topsep=0.01em]
    \item We propose a generalized heavy-tailed noise framework for both vector and matrix variables and derive sharp convergence rates of $O(T^{-\frac{p-1}{3p-2}})$ for SignSGD, Lion, Muon, and Muonlight.
    \item We characterize the mechanism by which the sign operator inherently handles heavy-tailed noise, providing a theoretical justification for their superiority over SGD or AdamW in LLMs.
    \item We conduct extensive LLM pretraining experiments on GPT2 to validate our theoretical findings and provide empirical evidence that our proposed noise model accurately reflects the heavy-tailed stochasticity encountered in practice.
\end{itemize}

\section{Related Work\label{sec:related-work}}

\paragraph{Heavy-tailed noise}
The seminal work of~\citet{zhang2020whyheavytail} theoretically explains why adaptive gradient methods like Clipped SGD~\citep{pmlr-v28-pascanu13} have a benefit for Attention models~\citep{NIPS2017transformer}. Inspired by~\citet{simsekli19tailindex}, they identify the heavy-tailed noise distributions in language modeling tasks, and establish an $O(T^{-\frac{p-1}{3p-2}})$ convergence for clipped SGD with a matching lower bound under the standard heavy-tailed gradient noise condition, where the noise is only assumed to have a finite $p$th moment where $1\le p\le 2$. After that, a series of works try to study the behavior of this heavy-tailed stochasticity~\citep{gurbuzbalaban2021heavy,kunstner2024heavytail,kunstner2025scaling}, or leverage this phenomenon to explain the empirical success of popular algorithms~\citep{zhou2020towards,yadav2025provable} or model architectures~\citep{ahn2024linearattention}. As for the theory community, heavy-tailedness has been widely adopted and investigated in multiple realms, such as learning theory~\citep{hsu2014heavy,NeurIPS:2018:Zhang:A}, online learning~\citep{zhang2022parameterfree,liu2026old}, bandits~\citep{bubeck2013bandits,ICML:2019:Lu,IJCAI:2020:Xue,UAI:2023:Gou,NeurIPS:2023:Xue,pmlr-v267-ye25e} and optimization~\citep{hodgkinson2021multiplicative,cutkosky2021high,vural2022mirror,liu2023breaking,liu2023stochastic,nguyen2023improved,pmlr-v202-sadiev23a,konilov2023accelerated,pmlr-v238-puchkin24a,pmlr-v235-gorbunov24a,ICML:2024:Liu,liu2025nonconvex,pmlr-v258-hubler25a,fatkhullin2025can,he2025complexity,liu2025stochastic,wang2026near,chen2026stability,wu2026optimal,liu2026clipped}.

\paragraph{Sign-based methods: vector optimizers}
The simplest signed gradient descent algorithm, namely SignSGD, is proposed by~\citet{bernstein2018signsgd}, where the authors establish convergence rates based on $\ell_1$-norm and discuss its possible complexity improvement over SGD.~\citet{bernstein2019signsgd} later shows that SignSGD with majority vote is communication efficient and robust to noise error. Further attempts have been made to improve the convergence of SignSGD for smooth functions~\citep{safaryan2021stochastic,sun2023momentum,jiang2025improved,kornilov2025signheavytail,JMLR:v26:24-0523}. While for non-smooth functions,~\citet{karimireddy2019error} provide a simple convex counterexample where SignSGD fails to converge, for which they utilize the error feedback techniques~\citep{seide20141} to alleviate the divergence.~\citet{yu2026stosignsgd} propose a new optimizer named StoSignSGD, provably fixing this non-convergence issue by injecting structural noise into the sign operator.~\citet{liu2019signsgd,petrov2025leveraging} combine signed gradient with zeroth-order optimization to reduce memory cost in LLM training. SignSGD has also received great attention due to its close relationship with Adam~\citep{balles2018dissecting}. Numerous efforts have been made to theoretically justify the success of Adam through the lens of SignSGD~\citep{crawshaw2022robustness,kunstner2023noise,peng2025simple}, and of course, study the effectiveness of SignSGD itself~\citep{balles2020geometry,liICLR2025on,ICLR2025adaptive,kunstner2025scaling}. On a parallel path,~\citet{tao2026when} leverage a complexity-theoretical framework to rigorously justify when and why SignSGD may have a provable complexity improvement over vanilla SGD by comparing the $\ell_1$-norm lower bound of SGD to the upper bound of SignSGD. Another important breakthrough is the Lion optimizer, proposed by~\citet{chen2023symbolic} via symbolic search. By incorporating two momentum buffers, it achieves significant empirical success~\citep{zhaoICLR2025deconstructing}. Likewise, there exists a rich body of literature that tries to explain its effectiveness~\citep{chenICLR2024lion,elistratov2024lion,dong2024convergence,jiang2025lion,sfyraki2025lions}.

\paragraph{Sign-based methods: matrix optimizers}
The Muon algorithm is introduced by~\citet{jordan2024muon}, which can be regarded as matrix sign descent. The matrix sign operation $\msign{\cdot}$ is implemented via Newton--Schulz iteration~\citep{kovarik1970some,bjorck1971iterative}. Built upon~\citet{jordan2024muon},~\citet{liu2025muon} incorporate Nesterov momentum and learning rate alignment techniques into the original version of Muon. The modified algorithm Muonlight, has been applied to train massive-scale LLMs~\citep{liu2025muon,zeng2025glm,team2025kimi,kimi2026kimi2.5} and exhibits clear performance gains over AdamW~\citep{wen2025fantastic,semenov2025benchmarking,shah2025practical,ahn2025dion,ahn2025dion2}. Many works have tried to theoretically analyze Muon (or its modified versions), or empirically explain its practical efficiency~\citep{li2025note,shen2025convergence,chang2025convergence,si2025adamuon,sfyraki2025lions,chen2025muon,huang2025limuon,li2025normuon,qian2025muon,tveit2025muon,page2025muonall,mehta2025muon,vasudeva2025muon,vasudeva2025the,frans2025really,pan2025unbiased,wang2025muon,zhang2025provable,zhang2025adagrad,su2025isotropic,crawshaw2025exploration,ma2026preconditioning,du2026newton}.

\paragraph{How to handle heavy-tailed stochasticity?}
It is well-known that incorporating clipping or normalization into the vanilla SGD provably helps to tackle heavy-tailed noise~\citep{zhang2020whyheavytail,cutkosky2021high,liu2025nonconvex,pmlr-v258-hubler25a,JMLR:v26:24-1991,he2025complexity}.~\citet{pmlr-v267-chezhegov25a} recently prove that Adam and AdaGrad might converge poorly in high probability with the presence of heavy-tailed gradient noise, while their clipped versions effectively fix this issue. Another important line of work focuses on vanilla versions of classic algorithms.~\citet{liu2026old} studies SGD~\citep{zinkevich2003oco}, Dual Averaging~\citep{nesterov2009primal}, and AdaGrad~\citep{duchi2011adaptive} from the perspective of online convex optimization. They show that these algorithms can converge optimally under heavy-tailedness.~\citet{fatkhullin2025can} independently derive the same conclusion for SGD. However, a distinct difference between~\citet{fatkhullin2025can,liu2026old} and our work is that they assume the stochastic gradient itself has a $p$th moment, while we make the assumption on the gradient noise. The former is strictly stronger than the latter\footnote{It is also well-known that if the assumption is exerted on the gradient itself, then it would be more natural to consider non-smooth optimization.}, and thus, should be generally regarded as an orthogonal line of work.

\section{Vector Sign Optimizer: SignSGD and Lion\label{sec:vector-sign}}

In this section, we first introduce some notations and assumptions, and then present the convergence guarantee. Proofs of the main theorems are deferred to~\cref{sec:vector-sign-analysis}.

\subsection{Notations and Assumptions\label{sec:vector-assumption}} 

We write $[T]$ for $\{1,2,\dots,T\}$, $\abs{\x}$ for element-wise absolute value of $\x\in\R^d$, $\odot$ for element-wise product between two vectors. The $\lb$-weighted vector norm for $\lb\in\R_{+}^{d}$ is defined as $\norm{\x}_{\lb}^{2}:=\x^{\top}\diag{\lb}\x$. For vector descent methods, we study the optimization problem $\min_{\x\in\R^d}f(\x)$, where $f:\R^d\mapsto\R$ is differentiable. Given a point $\x\in\R^d$, we can only access the gradient $\nabla f(\x)=[\nabla_1f(\x),\cdots,\nabla_df(\x)]\in\R^d$ in a noisy manner, as later specified in~\cref{ass:unbiased}. Below, we list some necessary assumptions.

\renewcommand{\theass}{\arabic{ass}a}
\renewcommand{\theHass}{vec.\arabic{ass}a}
\setcounter{ass}{0} 

\begin{ass}[Lower bounded objective]\label{ass:non-convexity}
    The objective function $f$ is possibly non-convex, almost everywhere twice-differentiable, and bounded from below: $f^*:=\inf_{\x\in\R^d} f(\x)>-\infty$.
\end{ass}

\begin{ass}[$(\lb_0,\lb_1)$-vector smoothness]\label{ass:generalized-smooth}
    There exists non-negative vectors $\lb_0=[\lb_{0,1},\cdots,\lb_{0,d}]\allowdisplaybreaks\\\in\R^d_+$ and $\lb_1=[\lb_{1,1},\cdots,\lb_{1,d}]\in\R^d_+$ such that for all $\norm{\x^\prime-\x}_{\infty}\le1/\norm{\lb_1}_\infty$, it holds that 
    \begin{align}\label{eq:vector-smooth}
    \abs{f(\x^\prime)-\brac{f(\x)+\inner{\nabla f(\x)}{\x^\prime-\x}}}\le\frac{1}{2}\norm{\x^\prime-\x}_{\lb_0+\lb_1\odot\abs{\nabla f(\x)}}^2.
    \end{align} 
\end{ass}
\cref{ass:non-convexity} is standard and necessary for stochastic non-convex optimization~\citep{arjevani2023lower}. We further introduce twice-differentiability as a regularity condition coupled with our smoothness model (see \cref{lem:smooth-grad-sign} for details). \cref{ass:generalized-smooth} is new, which characterizes a class of generalized smooth functions originally proposed in~\citet{Zhang2020Why}. Note that when $\lb_1=\0$,~\cref{ass:generalized-smooth} degenerates to the \emph{coordinate-wise smoothness} condition widely employed in the pioneering works of SignSGD~\citep{bernstein2018signsgd,bernstein2019signsgd,safaryan2021stochastic}, which naturally aligns with sign descent methods due to their coordinate-wise nature~\citep{balles2018dissecting,balles2020geometry}. When $\lb_1\succ\0$,~\cref{ass:generalized-smooth} can be viewed as an extension of the \emph{coordinate-wise smoothness} assumption, which aligns closely with practice~\citep{crawshaw2022robustness,liu2025adagrad}. In general,~\cref{ass:generalized-smooth} is not only weaker than previous assumptions, but also a better fit for sign gradient descent. Detailed discussions are postponed to~\cref{sec:discussion-gs}.

\begin{ass}[Unbiasedness]\label{ass:unbiased}
At step $t$ we observe a mini-batch of mutually independent gradients $G_{t}=\{\g_{t}^{1},\cdots,\g_{t}^{B}\}$ satisfying $\E\sqbrac{\g_t^b|\F_{t-1}}=\nabla f(\x_t),\forall b\in[B]$ where $\mathcal{F}_{t}=\sigma(G_{1},\dots,G_{t})$ denotes the natural filtration.
\end{ass}

\begin{ass}[$(\bsigma_0,\bsigma_1)$-vector heavy-tailed noise]\label{ass:heavy-tailed-noise}
    There exists $p\in(1,2]$, $\bsigma_0=[\bsigma_{0,1},\cdots,\allowdisplaybreaks\\\bsigma_{0,d}]\in\R^d_+$ and $\bsigma_1=[\bsigma_{1,1},\cdots,\bsigma_{1,d}]\in\R^d_+$ such that 
    \begin{align*}
        \E\sqbrac{\left.\abs{\g_{t,i}^b-\nabla_if(\x_t)}^p\right|\F_{t-1}}\le\bsigma_{0,i}^p+\bsigma_{1,i}^p\abs{\nabla_if(\x_t)}^p,\quad\forall i\in[d],\forall b\in[B].
    \end{align*}
\end{ass}
\cref{ass:heavy-tailed-noise} is new, which generalizes the classic \emph{coordinate-wise} heavy-tailed noise assumption~\citep[Assumption~3]{kornilov2025signheavytail} where $\bsigma_1=\0$.~\cref{ass:heavy-tailed-noise} is meaningful in two aspects. Theoretically, we demonstrate that the class heavy-tailed noise model ($\bsigma_1=\0$) might fail even on linear regression tasks (see~\cref{sec:discussion-ht}). Empirically, we draw significant evidence from LLM pretraining to illustrate that~\cref{ass:heavy-tailed-noise} aligns well with practice (see~\cref{sec:experiments}).


\subsection{Convergence Theory for SignSGD}

We formally introduce the SignSGD optimizer~\citep{bernstein2018signsgd} in~\cref{alg:signsgd}.\footnote{\cref{alg:signsgd} is referred to as Signum, signSGD-SIM, or SMM in previous work~\citep{bernstein2018signsgd,sun2023momentum,jiang2025improved}. For simplicity, we keep the name SignSGD hereafter.} At step $t$,~\cref{alg:signsgd} maintains the momentum buffer $\m_t$ using the batched stochastic gradient $\g_t$. Then, the sign of the momentum $\m_t$ is utilized to update the current model $\x_t$. Below, we present the convergence guarantee of~\cref{alg:signsgd}, showcasing that SignSGD remains effective under heavy-tailed gradient noise. 

\begin{thm}\label{thm:signsgd}
Under~\cref{ass:non-convexity,ass:generalized-smooth,ass:unbiased,ass:heavy-tailed-noise}, define $\Delta_f:=f(\x_1)-f^*$. By setting
    \begin{equation}\label{eq:signsgd-params}
        \begin{aligned}
        &B=\max\cbrac{1,\ceil{\brac{64\sqrt{2}\norm{\bsigma_1}_\infty}^{\frac{p}{p-1}}}},\quad \beta=\max\cbrac{0,1-B^{\frac{2p-2}{3p-2}}\brac{\frac{\Delta_f\norm{\lb_0}_1}{\norm{\bsigma_0}_1^2T}}^{\frac{p}{3p-2}}},\\&\eta=\min\cbrac{\sqrt{\frac{2\Delta_f(1-\beta)}{9\norm{\lb_0}_1T}},\frac{1-\beta}{32\norm{\lb_1}_\infty}},
        \end{aligned}        
    \end{equation}
    \cref{alg:signsgd} ensures
    \begin{align*}
        \frac{1}{T}\sum_{t=1}^T\E\sqbrac{\norm{\nabla f(\x_t)}_1}\le
        O\brac{(\Delta_f\norm{\lb_0}_1)^{\frac{p-1}{3p-2}}\norm{\bsigma_0}_1^{\frac{p}{3p-2}}(BT)^{-\frac{p-1}{3p-2}}}.
    \end{align*}
\end{thm}
The convergence rate in~\cref{thm:signsgd} implies a complexity of $O(\epsilon^{-\frac{3p-2}{p-1}})$ to reach an $\epsilon$-stationary point, i.e., $\E\sqbrac{\norm{\nabla f(\x)}_1}\le\epsilon$. Since $\norm{\nabla f(\x)}_1\ge\norm{\nabla f(\x)}_2$, this complexity matches the $\ell_2$-norm lower bound in~\citet[Theorem~3.3]{liu2025nonconvex} and also recovers the well-known $O(\epsilon^{-4})$ lower bound~\citep{arjevani2023lower} when $p=2$. Unlike previous works~\citep{jiang2025improved,kornilov2025signheavytail}, our result does not depend on dimension $d$ explicitly and holds under any $\lb_1\ge\0,\bsigma_1\ge\0$ and $p\in(1,2]$, suggesting that our theory may capture a wider range of objectives. Compared to the well-known gradient normalization or gradient clipping techniques which effectively cope with heavy-tailed noise~\citep{liu2025nonconvex,pmlr-v258-hubler25a}, our results reveal that the sign operator not only remains robust to heavy-tailed stochasticity but also has an advantage over gradient normalization. Due to space limitations, the comprehensive illustrations can be found in~\cref{sec:comparisons}.

\begin{remark}
    When $\bsigma_1=\0$,~\citet{kornilov2025signheavytail} analyze~\cref{alg:signsgd} and derive a complexity of $O(\Delta_fl_0d^{\frac{3p-2}{2(p-1)}}\norm{\bsigma_0}_p^{\frac{p}{p-1}}\epsilon^{-\frac{3p-2}{p-1}})$ under the classic $l_0$-smoothness. Compare their result to ours:
    \begin{align*}
        \text{\citet{kornilov2025signheavytail}}:O\brac{\frac{\Delta_f\norm{\lb_0}_1\norm{\bsigma_0}_1^{\frac{p}{p-1}}\textcolor{red}{d^{\frac{3p-2}{2(p-1)}}}}{\epsilon^{\frac{3p-2}{p-1}}\textcolor{red}{\phi_l\phi_\sigma^{\frac{p}{p-1}}}}}\texttt{vs}\text{ Ours}:O\brac{\frac{\Delta_f\norm{\lb_0}_1\norm{\bsigma_0}_1^{\frac{p}{p-1}}}{\epsilon^{\frac{3p-2}{p-1}}}},
    \end{align*}
    where $\phi_l:=\norm{\lb_0}_1/l_0\in[1,d]$ and $\phi_\sigma:=\norm{\bsigma_0}_1/\norm{\bsigma_0}_p\in[1,d^{1-\frac{1}{p}}]$ are two quantities reflecting the problem geometry. When the curvature vector $\lb_0$ is sparse and exhibits axis-alignment phenomenon~\citep{balles2020geometry}, $\phi_l$ becomes close to $1$. If $\bsigma_0$ is also sparse~\citep{yuan2017ell,pmlr-v119-wu20c}, then it holds $\phi_\sigma=O(1)$. In this case,~\cref{thm:signsgd} achieve a remarkable complexity improvement of $d^{\frac{3p-2}{2(p-1)}}$. On the other hand, taking the maximum over $\phi_l,\phi_\sigma$ gives the \emph{best case} complexity of~\citet{kornilov2025signheavytail}, which still suffers from an extra $d^{\frac{2-p}{2(p-1)}}$ factor compared to ours. 
\end{remark}


\begin{figure}[t]
    \begin{minipage}[t]{0.49\textwidth}
    \linespread{1.2}
        \begin{algorithm}[H] 
            \caption{SignSGD~\citep{bernstein2018signsgd}}
            \label{alg:signsgd}
            \begin{algorithmic}[1]
                \STATE {\bfseries Input:} $T\in\N$, $\x_1\in\R^d$, $\eta\in\R_+$
                \FOR{$t = 1$ {\bfseries to} $T$}
                    \STATE $\g_t=\frac{1}{B}\sum_{b=1}^B\g_t^b$
                    \STATE $\m_t = \beta \m_{t-1 } + (1-\beta) \g_t$ \hfill\COMMENT{$\m_0:=\g_1$}
                    \STATE $\x_{t+1} = \x_t - \eta \sign{\m_t}$
                \ENDFOR
            \end{algorithmic}
        \end{algorithm}
    \end{minipage}
    \hfill
    \begin{minipage}[t]{0.49\textwidth}
        \begin{algorithm}[H] 
            \caption{Lion~\citep{chen2023symbolic}}
            \label{alg:lion}
            \begin{algorithmic}[1]
                \STATE {\bfseries Input:} $T\in\N$, $\x_1\in\R^d$, $\eta\in\R_+$, $\lambda\in\R_+$
                \FOR{$t = 1$ {\bfseries to} $T$}
                    \STATE $\g_t=\frac{1}{B}\sum_{b=1}^B\g_t^b$
                    \STATE $\v_t = \beta_1 \m_{t-1} + (1-\beta_1) \g_t$ \hfill\COMMENT{$\m_0:=\g_1$}
                    \STATE $\m_t = \beta_2 \m_{t-1} + (1-\beta_2) \g_t$
                    \STATE $\x_{t+1} = \x_t - \eta \sign{\v_t}-\eta\lambda\x_t$                 
                \ENDFOR
            \end{algorithmic}
        \end{algorithm}
    \end{minipage}
\end{figure}

\subsection{Convergence Theory for Lion}

The Lion optimizer is discovered by~\citet{chen2023symbolic} through program search. On top of SignSGD, it incorporates (i) different $\beta$ values to track the exponential moving average of past gradients $\cbrac{\g_t}_{t\in[T]}$ thus decoupling the state between stored momentum $\m_t$ and the update momentum $\v_t$, and (ii) decoupled weight decay into the final update. The whole procedure is summarized in~\cref{alg:lion}. The following theorem prescribes the convergence rate of Lion.

\begin{thm}\label{thm:lion}
Under~\cref{ass:non-convexity,ass:generalized-smooth,ass:unbiased,ass:heavy-tailed-noise}, define $\Delta_f:=f(\x_1)-f^*$. Set
    \begin{equation}\label{eq:lion-params}
        \begin{aligned}
        &B=\max\cbrac{1,\ceil{\brac{144\sqrt{2}\norm{\bsigma_1}_\infty}^{\frac{p}{p-1}}}},\quad \beta_2=\max\cbrac{0,1-B^{\frac{2p-2}{3p-2}}\brac{\frac{\Delta_f\norm{\lb_0}_1}{\norm{\bsigma_0}_1^2T}}^{\frac{p}{3p-2}}},\\&\eta=\min\cbrac{\sqrt{\frac{8\Delta_f(1-\beta_2)}{33\norm{\lb_0}_1T}},\frac{1-\beta_2}{120\norm{\lb_1}_\infty}}, \beta_1\in\sqbrac{1-\brac{1-\beta_2}^{\frac{p-1}{p}},1}, \lambda\in\sqbrac{0,\frac{1-2^{-\frac{1}{T}}}{\eta}}.
        \end{aligned}        
    \end{equation}
    If $\norm{\x_1}_\infty\le1/(3\lambda)$, then~\cref{alg:lion} ensures $\norm{\x_t}_\infty\le2/(3\lambda),\forall t\in[T]$, and
    \begin{align*}
        \frac{1}{T}\sum_{t=1}^T\E\sqbrac{\norm{\nabla f(\x_t)}_1}\le
        O\brac{(\Delta_f\norm{\lb_0}_1)^{\frac{p-1}{3p-2}}\norm{\bsigma_0}_1^{\frac{p}{3p-2}}(BT)^{-\frac{p-1}{3p-2}}}.
    \end{align*}
\end{thm}
The above convergence rate matches that of~\cref{thm:signsgd}, since Lion reduces to SignSGD when $\beta_1=\beta_2$. As far as we know,~\cref{thm:lion} is the first rigorous theoretical guarantee for Lion under heavy-tailed noise. The implied $O(\epsilon^{-\frac{3p-2}{p-1}})$ complexity, as aforementioned SignSGD, is also optimal.

\begin{remark}
    In the special case of $\lb_1=\bsigma_1=\0$ and $p=2$ (finite-variance regime),~\citet{jiang2025lion} establish a convergence rate of $O(\sqrt{d}T^{-1/4})$ for Lion, which our result strictly improves upon. Firstly, our result removes the explicit dimensional dependence. Secondly, we relax their $\beta_1\in[1-\sqrt{1-\beta_2},1-(1-\beta_2)^2]$ requirement as shown in~\eqref{eq:lion-params}. Thirdly, by a finer analysis of the weight decay dynamics in~\cref{lem:stability-wd}, we relax their $\lambda\le1/(2\eta T)$ requirement to the milder $\lambda\le(1-2^{-1/T})/\eta$. Lastly, our initialization condition $\norm{\x_1}_\infty\le1/(3\lambda)$ is more general than their $\norm{\x_1}_\infty\le\eta$, as a large $\lambda$ ($\sim10^{-1}$) and a small $\eta$ ($\sim10^{-5}$) is typically chosen for Lion~\citep{wen2025fantastic}. Furthermore, we establish a uniform constant bound on the iterate trajectory $\norm{\x_t}_\infty \le 2/(3\lambda)$, which is significantly tighter than the $\norm{\x_t}_\infty\le\eta t)$ bound in~\citet{jiang2025lion} that grows over time. Our result also aligns with the bounded constraints identified for Lion-$\K$ dynamics in~\citet{chenICLR2024lion}.
\end{remark}

\section{Matrix Sign Optimizer: Muon and Muonlight\label{sec:matrix-sign}}

The organization of this section follows from~\cref{sec:vector-sign}. Proofs can be found at~\cref{sec:matrix-sign-analysis}.

\subsection{Notations and Assumptions} 

We denote the set of $m \times m$ positive semi-definite (PSD) matrices by $\SBB^m$. For any $\XB \in \R^{m \times n}$, its modulus (or absolute value) is defined as $\mabs{\XB} := \brac{\XB\XB^\top}^{1/2} \in \SBB^m$. The matrix sign operator is defined as $\msign{\XB} := \U\V^\top$, where $\XB = \U\bSigma\V^\top$ is the singular value decomposition (SVD) of $\XB \in \R^{m \times n}$. The inner product between matrices $\XB,\YB\in\R^{m\times n}$ is denoted by $\inner{\YB}{\XB}:=\tr{\YB^\top\XB}$. We let $\norm{\cdot}_\op$, $\norm{\cdot}_*$, and $\norm{\cdot}_\Fn$ denote the matrix operator norm, nuclear norm, and Frobenius norm, respectively. Finally, for any $\Lb \in \SBB^m$, the $\Lb$-weighted matrix norm is defined as $\norm{\XB}_{\Lb}^{2} := \tr{\XB^\top \Lb \XB}$. For Muon and Muonlight, we consider the matrix optimization problem $\min_{\XB\in\R^{m\times n}}f(\XB)$ and list some necessary assumptions in the sequel.

\renewcommand{\theass}{\arabic{ass}b}
\renewcommand{\theHass}{mat.\arabic{ass}b}
\setcounter{ass}{0} 

\begin{ass}[Lower bounded objective]\label{ass:non-convexity-matrix}
    The objective function $f$ is possibly non-convex, and bounded from below: $f^*:=\inf_{\XB\in\R^{m\times n}} f(\XB)>-\infty$.
\end{ass}

\begin{ass}[$(\Lb_0,\Lb_1)$-matrix smoothness]\label{ass:generalized-smooth-matrix}
    There exists matrices $\Lb_0,\Lb_1\in\R^{m\times n}$ where $\Lb_0$ has full row rank, such that for all $\norm{\XB^\prime-\XB}_{\op}\le1/\norm{\Lb_1}_\op$, it holds that 
    \begin{align*}
        \norm{\nabla f(\XB^\prime)-\nabla f(\XB)}_{\brac{\Lb(\XB)}^{-1}}\le\norm{\XB^\prime-\XB}_{\Lb(\XB)},\text{ where }\Lb(\XB)=\mabs{\Lb_0}+\mabs{\nabla f(\XB)\Lb_1^\top}.
    \end{align*} 
\end{ass}
\cref{ass:generalized-smooth-matrix} serves as the matrix-based counterpart to~\cref{ass:generalized-smooth} and represents, to the best of our knowledge, the first formulation of a matrix generalized smoothness model. Furthermore, when $\Lb_1=\0$, this framework encompasses classical matrix $L$-smoothness, as defined in~\citet[Assumption~2]{an2025asgo} and~\citet[Assumption~1]{kovalev2025non}.
As pointed out by~\citet{an2025asgo}, this formulation effectively models the block-diagonal Hessian structure of transformer-based modern neural networks~\citep{ zhang2024adam_hessian, zhang2025adammini}. 

\begin{ass}[Unbiasedness]\label{ass:unbiased-matrix}
At step $t$ we observe a mini-batch of mutually independent gradients $G_{t}=\{\Gb_{t}^{1},\cdots,\Gb_{t}^{B}\}$ satisfying $\E\sqbrac{\Gb_t^b|\F_{t-1}}=\nabla f(\XB_t),\forall b\in[B]$ where $\mathcal{F}_{t}=\sigma(G_{1},\dots,G_{t})$ denotes the natural filtration.
\end{ass}

\begin{ass}[$(\V_0,\V_1)$-matrix heavy-tailed noise]\label{ass:heavy-tailed-noise-matrix}
    There exists $p\in(1,2]$, $\V_0,\V_1\in\R^{m\times n}$ where $\V_0$ has full row rank, such that 
    \begin{align*}
        \E\sqbrac{\norm{\V_0}_*^{p/2}\cdot\left.\norm{\Gb_t^b-\nabla f(\XB_t)}^{p}_{\mabs{\V_0}^{-1}}\right|\F_{t-1}}\le\norm{\V_0}_*^{p}+\abs{\inner{\V_1}{\nabla f(\XB_t)}}^p,\quad\forall b\in[B].
    \end{align*}  
\end{ass}
\cref{ass:heavy-tailed-noise-matrix} introduces our newly proposed matrix-based generalized heavy-tailed noise model, which generalizes canonical assumptions in several key dimensions (see~\cref{sec:discussion-ht-matrix} for details). First, when $\V_1=\0$ and $p=2$, our condition recovers the adaptive variance assumption $\E\sqbrac{\left. \norm{\Gb_t^b - \nabla f(\XB_t)}^2_{\mabs{\V_0}^{-1}} \right| \F_{t-1}} \le \norm{\V_0}_*$ in~\citet[Assumption~3]{kovalev2025sgd},~\citet[Assumption~2]{kovalev2025non}, and~\citet[Definition~4.1]{xie2025tale}. Similar to coordinate-wise bounds in vector settings~\citep{bernstein2018signsgd,liu2025adagrad}, this formulation and its variants in~\citet[Assumption~3]{an2025asgo} and~\citet[Assumption~3]{pan2025unbiased} better capture the structural properties of the noise matrix. Second, by extending the tail-index to $p \in (1,2]$, we provide a natural characterization of the heavy-tailed regime. Finally, \cref{ass:heavy-tailed-noise-matrix} allows the $p$th moment of the noise to grow with the gradient, extending recent vector-based models~\citep[Assumption~2.4]{liu2025nonconvex} to matrix optimization---a property we empirically validate in~\cref{sec:experiments}.

\begin{figure}[t]
    \begin{minipage}[t]{0.49\textwidth}
    \linespread{1.23}
        \begin{algorithm}[H]
            \caption{Muon~\citep{jordan2024muon}}
            \label{alg:muon}
            \begin{algorithmic}[1]
                \STATE {\bfseries Input:} $T\in\N$, $\XB_1\in\R^{m\times n}$, $\eta\in\R_+$
                \FOR{$t = 1$ {\bfseries to} $T$}
                    \STATE $\Gb_t=\frac{1}{B}\sum_{b=1}^B\Gb_t^b$
                    \STATE $\BB_t = \beta \BB_{t-1 } + \Gb_t$ \hfill\COMMENT{$\BB_0:=\0$}
                    \STATE $\OB_t=\texttt{NewtonSchulz}(\BB_t)$
                    \STATE $\XB_{t+1} = \XB_t - \eta \OB_t$
                \ENDFOR
            \end{algorithmic}
        \end{algorithm}
    \end{minipage}
    \hfill
    \begin{minipage}[t]{0.49\textwidth}
        \begin{algorithm}[H]
            \caption{Muonlight~\citep{liu2025muon}}
            \label{alg:muonlight}
            \begin{algorithmic}[1]
                \STATE {\bfseries Input:} $T\in\N$, $\XB_1\in\R^{m\times n}$, $\eta,\lambda\in\R_+$
                \FOR{$t = 1$ {\bfseries to} $T$}
                    \STATE $\Gb_t=\frac{1}{B}\sum_{b=1}^B\Gb_t^b$
                    \STATE $\BB_t = \beta_2 \BB_{t-1} + \Gb_t$ \hfill\COMMENT{$\BB_0:=\0$}
                    \STATE $\BBT_t = \beta_1 \BB_{t} + \Gb_t$
                    \STATE $\OB_t=\texttt{NewtonSchulz}(\BBT_t)$
                    \STATE $\XB_{t+1} = \XB_t - \eta \OB_t-\eta\lambda\XB_t$
                \ENDFOR
            \end{algorithmic}
        \end{algorithm}
    \end{minipage}
\end{figure}

\subsection{Convergence Theory for Muon}
Muon~\citep{jordan2024muon} (\cref{alg:muon}) has recently emerged as a powerful optimizer for the hidden layers of large-scale neural networks, utilizing the principle of orthogonalized momentum.~\cref{alg:muon} first tracks a momentum estimate $\BB_t$ and subsequently employs the Newton--Schulz iteration~\citep{kovarik1970some,bjorck1971iterative} to compute the matrix sign, $\msign{\BB_t}$, which is further utilized to update the parameter matrix $\XB_t$. Following common practice~\citep{li2025note,shen2025convergence,sato2025convergence,chang2025convergence,pan2025unbiased}, we assume zero numerical error in Newton--Schulz algorithm, i.e., $\texttt{NewtonSchulz}(\XB)=\msign{\XB}=\U\V^\top$ with $\XB=\U\bSigma\V^\top$ as the SVD of $\XB\in\R^{m\times n}$. \emph{We remark that it is introduced for simplicity rather than necessity, with details deferred to~\cref{sec:NewtonSchulz}.}
In the following theorem, we establish the first convergence guarantee for Muon under our generalized heavy-tailed noise framework and matrix-based generalized smoothness conditions.
\begin{thm}\label{thm:muon}
    Under~\cref{ass:non-convexity-matrix,ass:generalized-smooth-matrix,ass:unbiased-matrix,ass:heavy-tailed-noise-matrix}, define $\Delta_f:=f(\XB_1)-f^*$. By setting
    \begin{equation}\label{eq:muon-params}
        \begin{aligned}
        &B=\max\cbrac{1,\ceil{\brac{64\sqrt{2}\norm{\V_1}_\op}^{\frac{p}{p-1}}}},\quad \beta=\max\cbrac{0,1-B^{\frac{2p-2}{3p-2}}\brac{\frac{\Delta_f\norm{\Lb_0}_*}{\norm{\V_0}_*^2T}}^{\frac{p}{3p-2}}},\\&\eta=\min\cbrac{\sqrt{\frac{2\Delta_f(1-\beta)}{5\norm{\Lb_0}_*T}},\frac{1-\beta}{16\norm{\Lb_1}_\op}},
        \end{aligned}        
    \end{equation}
    \cref{alg:muon} ensures
    \begin{align*}
        \frac{1}{T}\sum_{t=1}^T\E\sqbrac{\norm{\nabla f(\XB_t)}_*}\le O\brac{(\Delta_f\norm{\Lb_0}_*)^{\frac{p-1}{3p-2}}\norm{\V_0}_*^{\frac{p}{3p-2}}(BT)^{-\frac{p-1}{3p-2}}}.
    \end{align*}
\end{thm}
Notably, the convergence rate in~\cref{thm:muon} is free of the explicit dimensional dependence, representing a significant advancement over the existing literature on Muon~\citep{li2025note,shen2025convergence,chang2025convergence,huang2025limuon}. While prior works are restricted to the simpler finite-variance setting ($p=2, \Lb_1=\V_1=\0$), they still suffer from an explicit dependence on the matrix dimensions through the term $\sqrt{\min\{m,n\}}\norm{\V_0}_\Fn$. In contrast, our results not only provide the first convergence guarantee for the general heavy-tailed regime ($p\in(1,2)$) but also yield a sharper bound in the $p=2$ case, as the nuclear norm $\norm{\V_0}_*$ is significantly smaller than the \emph{dimension-dependent} factors in previous bounds. This improvement underscores the strength of our refined noise model in~\cref{ass:heavy-tailed-noise-matrix}, which more effectively captures the structural noise properties inherent to matrix optimization.

\begin{remark}
\citet{sfyraki2025lions} recently established a suboptimal $O(T^{-\frac{p-1}{3p-2}}\log{T})$ convergence rate for \emph{clipped} versions of Lion and Muon under heavy-tailed noise. Their analysis exhibits several limitations compared to our work: (i) their assumptions and convergence criteria rely on Euclidean norms ($\norm{\cdot}_2, \norm{\cdot}_\Fn$), which fail to capture the intrinsic geometry of sign-based methods, particularly from their Frank-Wolfe perspective; (ii) they require an additional bounded gradient assumption ($\norm{\nabla f(\XB)}_\Fn \le G$) and still incur extraneous logarithmic factors; and (iii) their theory is restricted to settings with a strictly positive weight decay parameter ($\lambda > 0$) and provides no guarantees for the standard algorithms \emph{without clipping}.
\end{remark}

\subsection{Convergence Theory for Muonlight}
Proposed by~\citet{liu2025muon}, Muonlight has achieved remarkable success in large-scale LLM pretraining, as evidenced by the \href{https://huggingface.co/moonshotai/Moonlight-16B-A3B}{Muonlight-16B-A3B}, \href{https://huggingface.co/moonshotai/Kimi-K2-Base}{Kimi-K2}, \href{https://huggingface.co/zai-org/GLM-4.5}{GLM-4.5} models. It introduces two primary modifications to the original Muon algorithm: (i) the integration of Nesterov momentum and (ii) the addition of decoupled weight decay. The former is a standard technique in modern optimization~\citep{dozat2016nadam,xie2024adan}---already a key feature in the popular PyTorch implementation of Adam~\citep{NEURIPS2019PYTORCH}---whose effectiveness is often attributed to implicit variance reduction~\citep{wen2025fantastic,yuan25mars,liu2025mars}. The latter, decoupled weight decay, has proven essential for stabilizing the training trajectories of large foundation models~\citep{team2025kimi}. The complete procedure is detailed in~\cref{alg:muonlight}. Notably, while original implementations often restrict Nesterov momentum to $\beta_1 = \beta_2$, our framework supports a general $(\beta_1, \beta_2)$ configuration. We present the theoretical analysis of Muonlight below.
 
\begin{thm}\label{thm:muonlight}
    Under~\cref{ass:non-convexity-matrix,ass:generalized-smooth-matrix,ass:unbiased-matrix,ass:heavy-tailed-noise-matrix}, define $\Delta_f:=f(\XB_1)-f^*$. Set
    \begin{equation}\label{eq:muonlight-params}
        \begin{aligned}
        &B=\max\cbrac{1,\ceil{\brac{5958\norm{\V_1}_\op}^{\frac{p}{p-1}}}},\quad \beta_2=\max\cbrac{0,1-B^{\frac{2p-2}{3p-2}}\brac{\frac{\Delta_f\norm{\Lb_0}_*}{\norm{\V_0}_*^2T}}^{\frac{p}{3p-2}}},\\&\eta=\min\cbrac{\sqrt{\frac{4\Delta_f(1-\beta_2)}{15\norm{\Lb_0}_*T}},\frac{3(1-\beta_2)}{625\norm{\Lb_1}_\op}},\ \beta_1\in\sqbrac{\max\cbrac{0,\beta_2-\frac{3}{20}},\min\cbrac{1,\beta_2+\frac{3}{20}}},
        \end{aligned}        
    \end{equation}
    and $\lambda\in\sqbrac{0,\brac{1-2^{-\frac{1}{T}}}/\eta}$. If $\norm{\XB_1}_\op\le1/(3\lambda)$, then~\cref{alg:muonlight} ensures $\norm{\XB_t}_\op\le2/(3\lambda),\forall t\in[T]$, and
    \begin{align*}
        \frac{1}{T}\sum_{t=1}^T\E\sqbrac{\norm{\nabla f(\XB_t)}_*}\le O\brac{(\Delta_f\norm{\Lb_0}_*)^{\frac{p-1}{3p-2}}\norm{\V_0}_*^{\frac{p}{3p-2}}(BT)^{-\frac{p-1}{3p-2}}}.
    \end{align*}
\end{thm}
Several remarks are in order: (i) Muonlight achieves the same convergence bounds as Muon under generalized heavy-tailed noise. (ii) Our result provides the first analysis for $p\in(1,2)$ and improves upon previous \emph{dimension-dependent} rates for the $p=2, \Lb_1=\V_1=\0$ case~\citep{sato2025convergence,chang2025convergence}. (iii) We establish a uniform trajectory bound $\norm{\XB_t}_\op \le 2/(3\lambda)$, mirroring the stability of Lion in~\cref{thm:lion}. (iv) The theory supports a broad range for the Nesterov momentum $\beta_1$, reinforcing the empirical flexibility and practicality of the algorithm.

\section{How Sign Operator Works}

In this section, we elucidate why the sign operator naturally accommodates heavy-tailed noise, while concurrently presenting our proof sketch and novel analysis techniques. To clearly illustrate the core intuition, we restrict our focus to the simplified setting where $\lb_1=\Lb_1=\0$, considering only SignSGD and Muon (\cref{thm:signsgd,thm:muon}). The core technical intuition is that the sign operator $\sign{\cdot}$ acts as \emph{coordinate-wise} normalization, inheriting the robust properties of gradient normalization under heavy-tailed noise, while \textbf{shifting the optimization geometry from Euclidean to non-Euclidean space}. For the matrix sign $\msign{\cdot}$, its effect mirrors the vector case by performing normalization on each singular value.

\subsection{Signed Gradient As Non-Euclidean Normalized Gradient\label{sec:proof-sketch}}

We begin by reviewing how the normalized gradient method (\cref{alg:nsgd} in~\cref{sec:comparisons}) mitigates heavy-tailed noise. We invoke the following well-established lemma, attributed to~\citet{cutkosky2020momentum,jin2021nonconvexdro,liu2025nonconvex,pmlr-v258-hubler25a}.
\begin{lem}\label{lem:nsgd-basic}
    Under the conditions in~\cref{thm:signsgd},~\cref{alg:nsgd} ensures that
    \begin{align*}
        &\frac{1}{T}\sum_{t=1}^T\E\sqbrac{\norm{\nabla f(\x_t)}_2}\overset{\hypertarget{nsgd-descent}{\textcolor{blue}{\textnormal{(a)}}}}{\le}\frac{\Delta_f}{\eta T}+\frac{\eta\norm{\Lb_0}_\infty}{2}+\frac{2}{T}\sum_{t=1}^T\E\sqbrac{\norm{\eps_t}_2}\\\overset{\hypertarget{nsgd-noise-expand}{\textcolor{blue}{\textnormal{(b)}}}}{\le}&\frac{\Delta_f}{\eta T}+\frac{\eta\norm{\Lb_0}_\infty}{2}+\frac{2}{T}\sum_{t=1}^T \E\sqbrac{\norm{\sum_{k=2}^t\beta^{t-k+1}\s_k}_2 + \beta^{t-1}\norm{\n_1}_2+ (1-\beta)\norm{\sum_{k=2}^T\beta^{t-k}\n_k}_2 }  ,
    \end{align*}
    where error term $\eps_t$, noise term $\n_t$, and curvature term $\s_t$ are defined in~\eqref{eq:momentum-def}.
\end{lem}
With~\cref{lem:nsgd-basic}, our goal is to bound the three terms in expectation. The first two term, namely the cumulative curvature term $\E\sqbrac{\norm{\sum_{k=2}^t\beta^{t-k+1}\s_k}_2}$ and the initial noise $\E\sqbrac{\beta^{t-1}\norm{\n_1}_2}$ admit straightforward upper bounds under~\cref{ass:generalized-smooth,ass:heavy-tailed-noise}. The crux of the analysis lies in controlling the cumulative noise $\E\sqbrac{\norm{\sum_{k=2}^T\beta^{t-k}\n_k}_2}$. This term is bounded by applying Hölder's inequality followed by standard von Bahr-Esseen type concentration inequalities~\citep{bahr1965inequalities,konilov2023accelerated,pmlr-v258-hubler25a}: $\E\sqbrac{\norm{\sum_{k=2}^T\beta^{t-k}\n_k}_2}\le\brac{\E\sqbrac{\norm{\sum_{k=2}^T\beta^{t-k}\n_k}_2^p}}^{\frac{1}{p}}\le\brac{2\sum_{k=2}^T\E\sqbrac{\norm{\beta^{t-k}\n_k}_2^p}}^{\frac{1}{p}}$. Invoking the heavy-tailed noise assumptions yields the optimal sublinear rate.

We now turn to sign-based gradient methods. The sign operator $\sign{\x}$ effectively performs \emph{coordinate-wise} normalization, as it satisfies $\norm{\sign{\x}}_\infty=1$ and $\norm{\sign{\x}}_1=d$. Indeed, signed gradient and normalized gradient methods can be viewed as instances of normalized steepest descent with respect to the $\ell_\infty$- and $\ell_2$-norms, respectively~\citep{bernstein2024old,yadav2025provable}. This fundamental shift in algorithmic geometry is formalized in the following lemma, with the proof provided in~\cref{sec:proof-signsgd}.
\begin{lem}\label{lem:signsgd-basic}
    Under the same conditions and notations as in~\cref{lem:nsgd-basic},~\cref{alg:signsgd} ensures that 
    \begin{align*}
        &\frac{1}{T}\sum_{t=1}^T\E\sqbrac{\norm{\nabla f(\x_t)}_1}\overset{\hypertarget{signsgd-descent}{\textcolor{blue}{\textnormal{(a)}}}}{\le}\frac{\Delta_f}{\eta T}+\frac{\eta\norm{\Lb_0}_1}{2}+\frac{2}{T}\sum_{t=1}^T\E\sqbrac{\norm{\eps_t}_1}\\\overset{\hypertarget{signsgd-noise-expand}{\textcolor{blue}{\textnormal{(b)}}}}{\le}&\frac{\Delta_f}{\eta T}+\frac{\eta\norm{\Lb_0}_1}{2}+\frac{2}{T}\sum_{t=1}^T \E\sqbrac{\norm{\sum_{k=2}^t\beta^{t-k+1}\s_k}_1 + \beta^{t-1}\norm{\n_1}_1+ (1-\beta)\norm{\sum_{k=2}^T\beta^{t-k}\n_k}_1 }  .
    \end{align*}
\end{lem}
Compared to~\cref{lem:nsgd-basic}, the primary distinction lies in the shift from the $\ell_2$-norm to the $\ell_1$-norm for both the convergence criterion $\norm{\nabla f(\x_t)}$ and the error term $\norm{\eps_t}$ (along with the three decomposed components).~\cref{lem:signsgd-basic} demonstrates that \emph{the signed gradient operates as a non-Euclidean form of gradient normalization, fundamentally altering the problem's geometry and the corresponding analysis}. A more comprehensive discussion of these geometric implications is provided in~\cref{sec:comparisons}. Similar to standard gradient normalization, the principal analytical challenge for~\cref{alg:signsgd} is bounding the cumulative noise in the $\ell_1$-norm, i.e., $\E\sqbrac{\norm{\sum_{k=2}^T\beta^{t-k}\n_k}_1}$. However, since standard concentration inequalities are typically restricted to Euclidean norms, this necessitates the development of novel techniques to control this critical term.

\subsection{New Martingale Concentration Inequalities\label{sec:concentration-inequality}}

To address the major technical challenge identified in~\cref{sec:proof-sketch}, we seek to bound the (discounted) sum of vector martingales in the $\ell_1$-norm, i.e., $\E\sqbrac{\norm{\sum_{k=2}^T\beta^{t-k}\n_k}_1}$. A close inspection of the von Bahr-Esseen inequality (e.g., Lemma~10 in~\citet{pmlr-v258-hubler25a}) reveals its inherent reliance on Euclidean geometry, rendering it ineffective for non-Euclidean norms like $\norm{\cdot}_1$. Furthermore, the standard approach of applying Hölder's inequality complicates the analysis by introducing $\norm{\cdot}^p_1$. This prompts a key question: \emph{Can we bypass Hölder's inequality and derive a concentration inequality tailored specifically to the $\ell_1$-norm?}

We provide an affirmative answer by establishing a new vector martingale concentration inequality in~\cref{lem:vector-concentration}, which bounds the $\ell_1$-norm of a vector martingale via its \emph{coordinate-wise} variance. By applying~\cref{lem:vector-concentration} with $\g_k=\beta^{t-k}\n_k$, we recover the tractable term $\E\sqbrac{\beta^{p(t-k)}\abs{\n_{k,i}}^p}$ (following ~\eqref{eq:Bt-initial-bound} and~\eqref{eq:recursion}), which is directly controllable under~\cref{ass:heavy-tailed-noise}.

\begin{lem}[Concentration in $\ell_1$-norm]\label{lem:vector-concentration}
    Let $\cbrac{\g_t}_{t=1}^T\subset\R^d$ be a vector martingale difference sequence such that $\E\sqbrac{\g_t|\F_{t-1}}=\0$, where $\F_t=\sigma\brac{\g_1,\cdots,\g_t}$ is the natural filtration. Then  
    \begin{align*}
        \E\sqbrac{\norm{\sum_{t=1}^T\g_t}_1}\le2\sqrt{2}\sum_{i=1}^d\E\sqbrac{\norm{\g_{1:T,i}}_2}\le2\sqrt{2}\sum_{i=1}^d\E\sqbrac{\norm{\g_{1:T,i}}_p},\quad\forall p\in[1,2],
    \end{align*}
    where $\g_{1:t,i}:=[\g_{1,i},\cdots,\g_{t,i}]\in\R^t$.
\end{lem}
Inspired by~\citet{rakhlin2017equivalence,liu2025nonconvex}, we delve deeply into the regret analysis of diagonal AdaGrad~\citep{McMahanS10adagrad,duchi2011adaptive} to prove this lemma (cf.~\cref{lem:adagrad-regret,lem:MDS-L1-concentration}). This conceptual strategy---\emph{concentrating vector martingales through the lens of regret analysis in online learning}---will be further leveraged in the matrix setting. Consider the matrix normalized gradient method (\cref{alg:mnsgd}) and Muon (\cref{alg:muon}), where the effect of matrix orthogonalization $\msign{\cdot}$ in Muon plays a role analogous to the vector sign operator $\sign{\cdot}$ discussed in~\cref{sec:proof-sketch}. Concretely, the norms highlighted in~\cref{lem:nsgd-basic,lem:signsgd-basic} map directly to their matrix Schatten $p$-norm counterparts: $\norm{\cdot}_2\to\norm{\cdot}_\Fn,\norm{\cdot}_1\to\norm{\cdot}_*$. Consequently, we face the similar challenge of bounding $\E\sqbrac{\norm{\sum_{k=2}^T\beta^{t-k}\NB_k}_*}$ ($\NB_k=\Gb_k-\nabla f(\XB_k)$), necessitating a novel matrix concentration inequality. To resolve this, we present the following lemma to bound $\norm{\sum_{t=1}^T\Gb_t}_*$ by its covariance.

\begin{lem}[Concentration in nuclear norm, full version in~\cref{lem:asgo-regret}]\label{lem:matrix-concentration}
    Let $\cbrac{\Gb_t}_{t=1}^T\subset\R^{m\times n}$ be a matrix martingale difference sequence such that $\E\sqbrac{\Gb_t|\F_{t-1}}=\0$, where $\F_t=\sigma\brac{\Gb_1,\cdots,\Gb_t}$ is the natural filtration. Then, it holds that
    \begin{align*}
        \E\sqbrac{\norm{\sum_{t=1}^T\Gb_t}_*}\le2\sqrt{2}\E\sqbrac{\norm{\brac{\sum_{t=1}^T\Gb_t\Gb_t^\top}^{1/2}}_*}.
    \end{align*}
\end{lem}
\cref{lem:matrix-concentration} stands in sharp contrast to existing matrix concentration inequalities, which typically rely on $\norm{\cdot}_\Fn, \norm{\cdot}_\op$ or suffer from dimension-dependent factors~\citep{nemirovski2007sums,so2011moment,TIT:2014:Zhang,tropp2015introduction}. At first glance, this result appears elusive. To derive it, we seek to extend the regret-based concentration argument from~\cref{lem:vector-concentration} to the matrix setting. The most natural candidate for this extension is Shampoo~\citep{gupta18shampoo}, which serves as the matrix counterpart to AdaGrad. However, this direct path is blocked because the standard Shampoo regret bound incurs explicit dimension dependence. To overcome this barrier, we instead conduct a new regret analysis of \emph{one-sided} Shampoo~\citep{an2025asgo,xie2025structured}. This critical pivot allows us to establish a desirable bound that, equipped with~\cref{lem:matrix-concentration}, enables the clean cancellation of terms (cf.~\eqref{eq:Bt-cancellation}) necessary to handle the heavy-tailed noise.

\begin{figure}[t]
\centering
\includegraphics[width=\columnwidth]{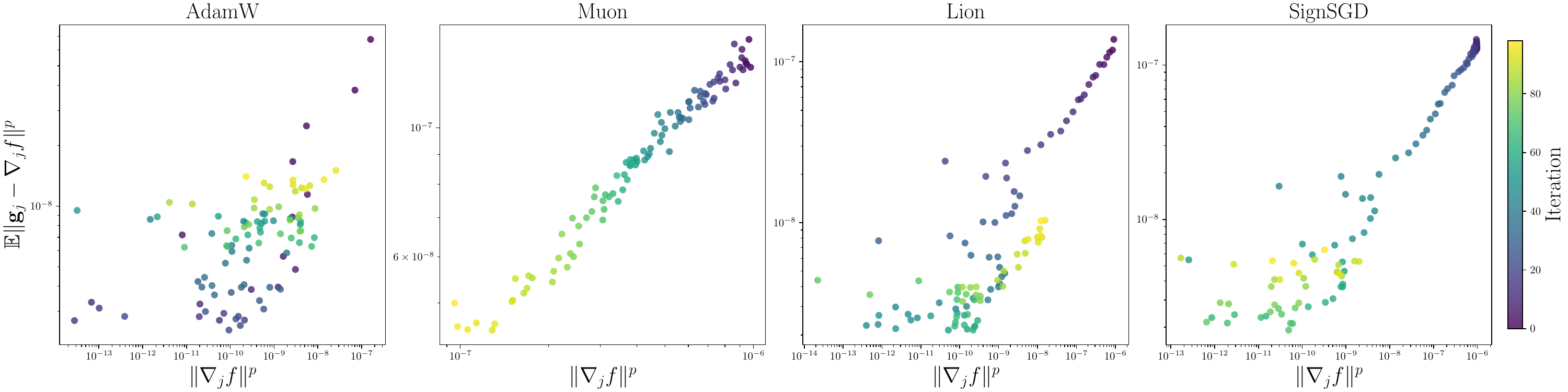}
\caption{Verification of~\cref{ass:heavy-tailed-noise}. \textbf{x-axis}: $|\nabla_j f|^p$, \textbf{y-axis}: $\E \sqbrac{\abs{\g_j - \nabla_j f}^p}$.}
\label{figs:assumption_4a_verification_start=0_j=1555144}
\end{figure}
\begin{figure}[t]
\centering
\includegraphics[width=\columnwidth]{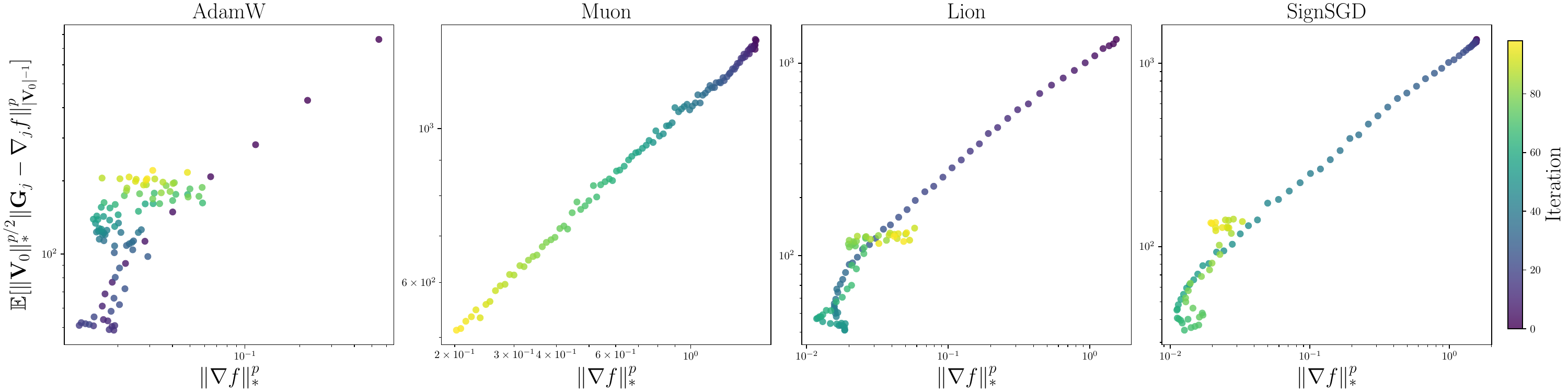}
\caption{Verification of~\cref{ass:heavy-tailed-noise-matrix}. \textbf{x-axis}: $\norm{\nabla f}_*^p$, \textbf{y-axis}: $\E \sqbrac{\norm{\V_0}_*^{p/2} \norm{\Gb - \nabla f}^p_{\mabs{\V_0}^{-1}}}$.}
\label{figs:assumption_4c_verification_start=0}
\end{figure}

\section{Empirical Study\label{sec:experiments}}
In this section, we provide empirical evidence to (i) demonstrate that our proposed noise models in~\cref{ass:heavy-tailed-noise,ass:heavy-tailed-noise-matrix} accurately reflect the stochasticity in LLM training, and (ii) evaluate the efficacy of sign-based optimizers under these heavy-tailed regimes. Code is available at \href{https://github.com/Dingzhen230/Heavy-tailed-Noise-in-LLMs}{https://github.com/Dingzhen230/Heavy-tailed-Noise-in-LLMs}. Further details regarding hyperparameters and the complete experimental setup are available in~\cref{sec:experiments-add}. \cref{figs:assumption_4a_verification_start=0_j=1555144,figs:assumption_4c_verification_start=0} shows that our proposed generalized heavy-tailed noise model closely reflects the noise pattern in practice. \cref{figs:train_val_and_speedup} verifies the practical efficiency of sign-based optimizers in LLM pretraining, where Lion and Muon achieve \textcolor{red}{1.07$\times$} and \textcolor{red}{1.32$\times$ \faRocket} speedup over AdamW, respectively. 

\begin{figure}[h]
    \centering
    \captionsetup[subfloat]{font=small, justification=centering}
    \vspace{-20pt}
    \subfloat[Train loss]{
        \includegraphics[width=.175\textwidth]{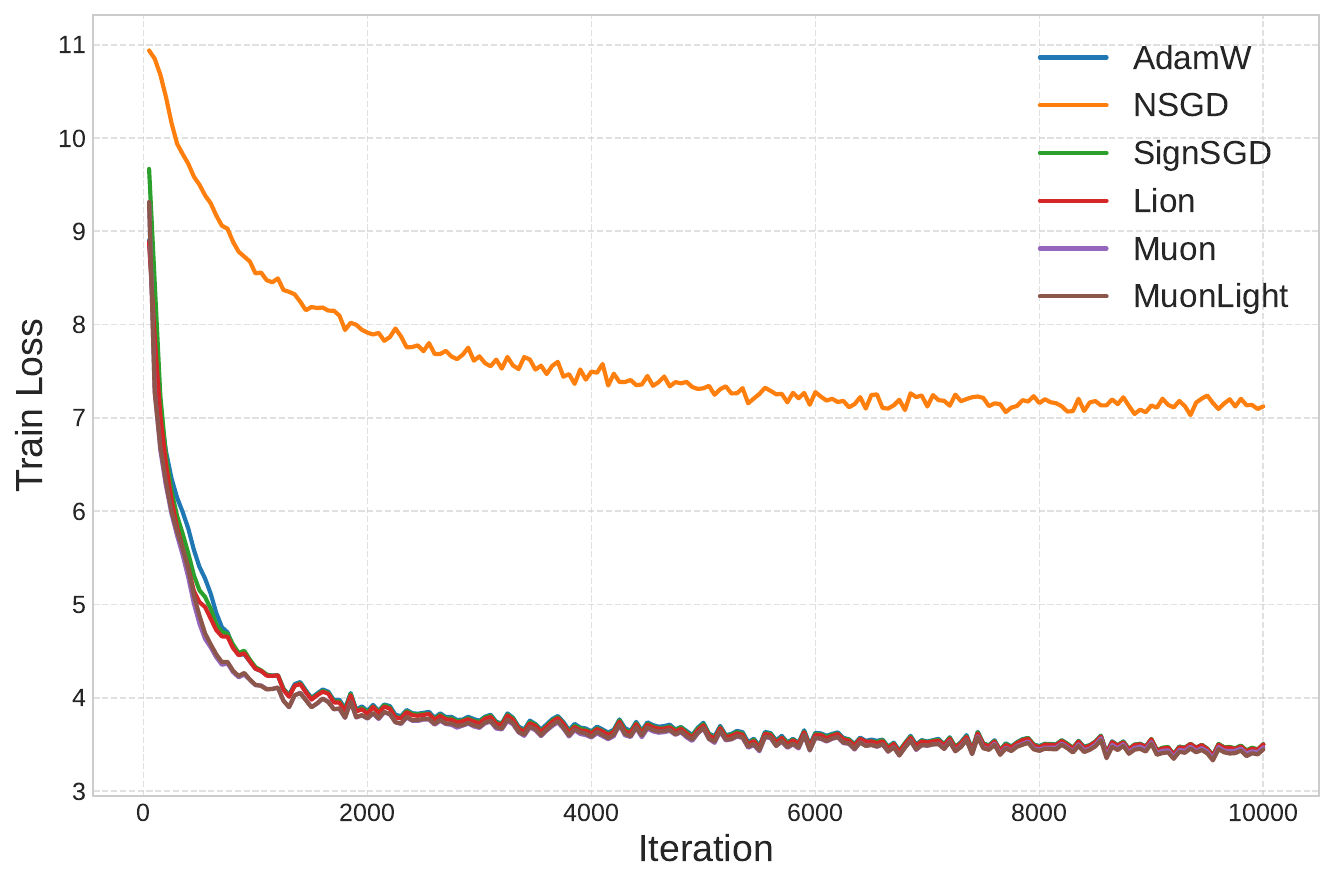}
        \label{fig:train_loss}
    }
    \hfill
    \subfloat[Val loss]{
        \includegraphics[width=.175\textwidth]{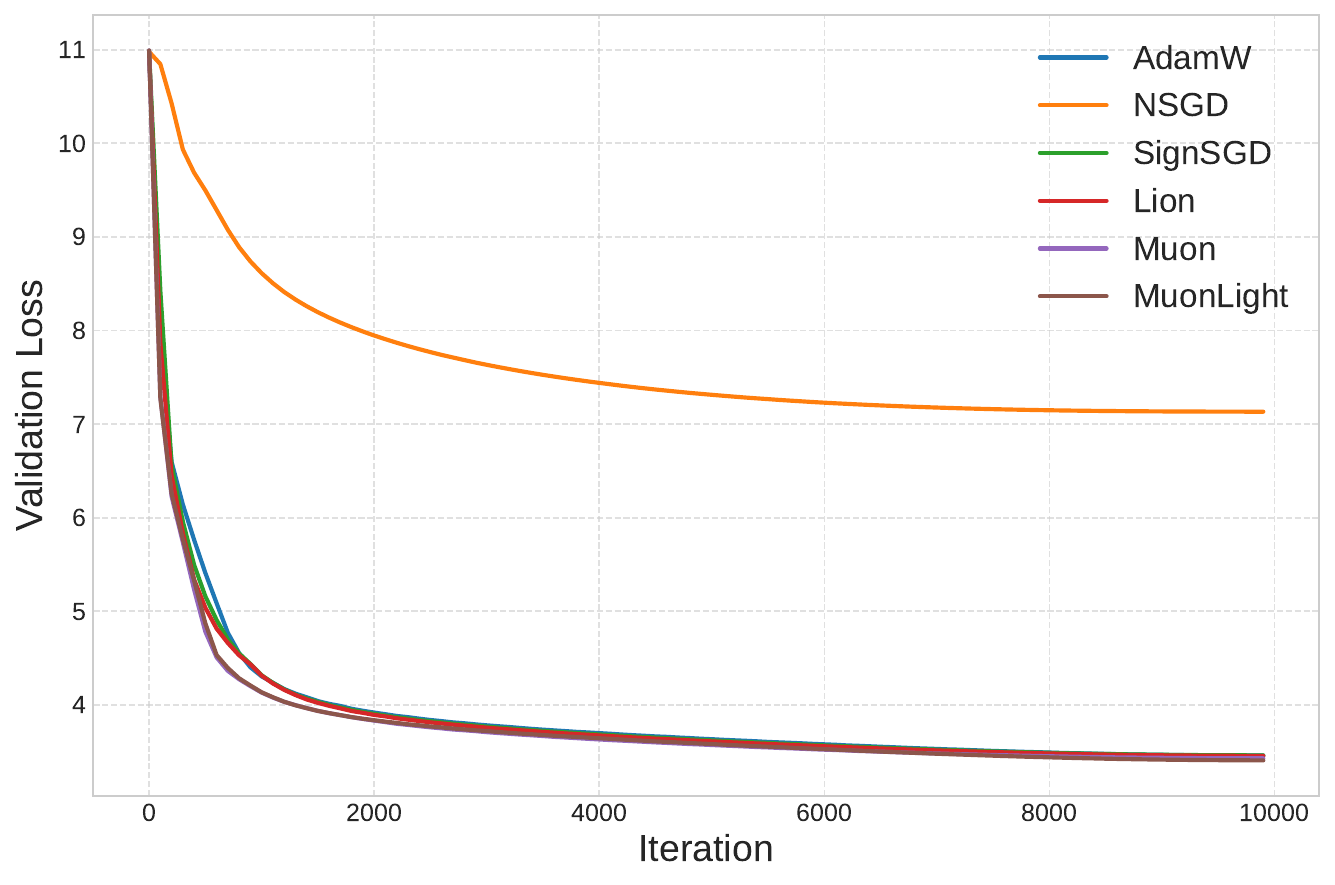}
        \label{fig:val_loss}
    }
    \hfill
    \subfloat[Val acc]{
        \includegraphics[width=.175\textwidth]{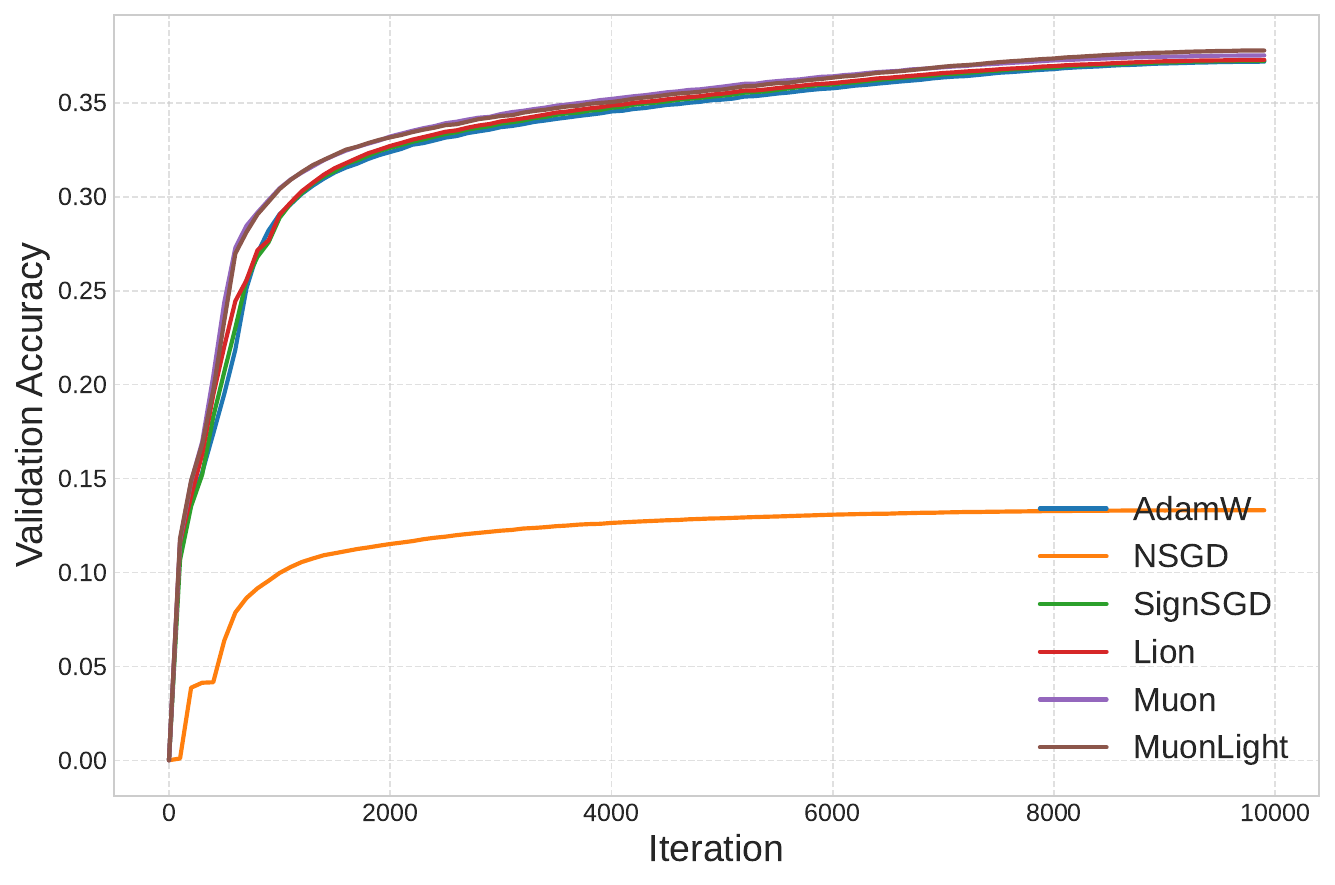}
        \label{fig:val_acc}
    }
    \hfill
    \subfloat[Lion: \textcolor{red}{1.07$\times$ \faRocket}]{
        \includegraphics[width=.19\textwidth]{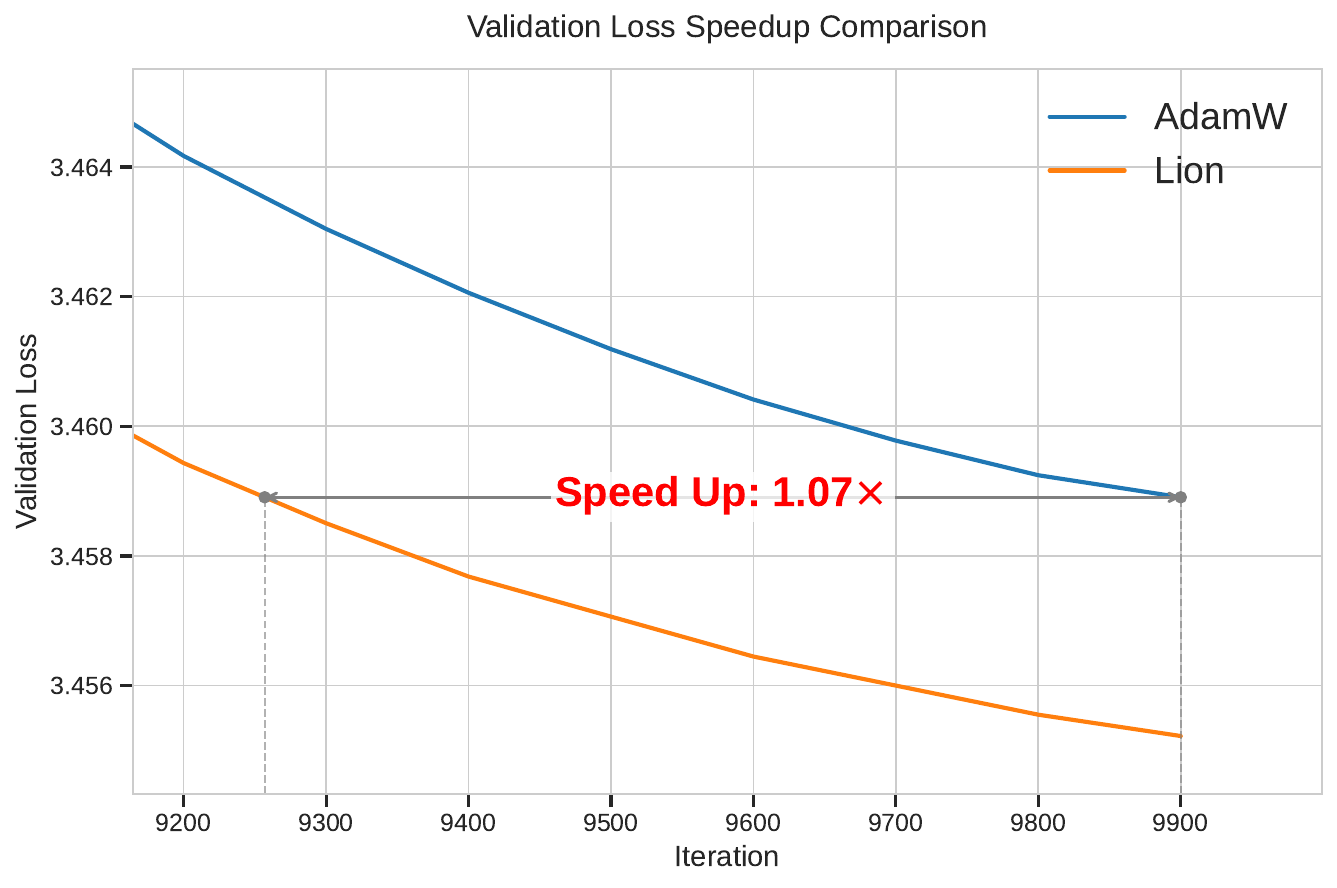}
        \label{fig:lion_speedup}
    }
    \hfill
    \subfloat[Muon: \textcolor{red}{1.32$\times$ \faRocket}]{
        \includegraphics[width=.19\textwidth]{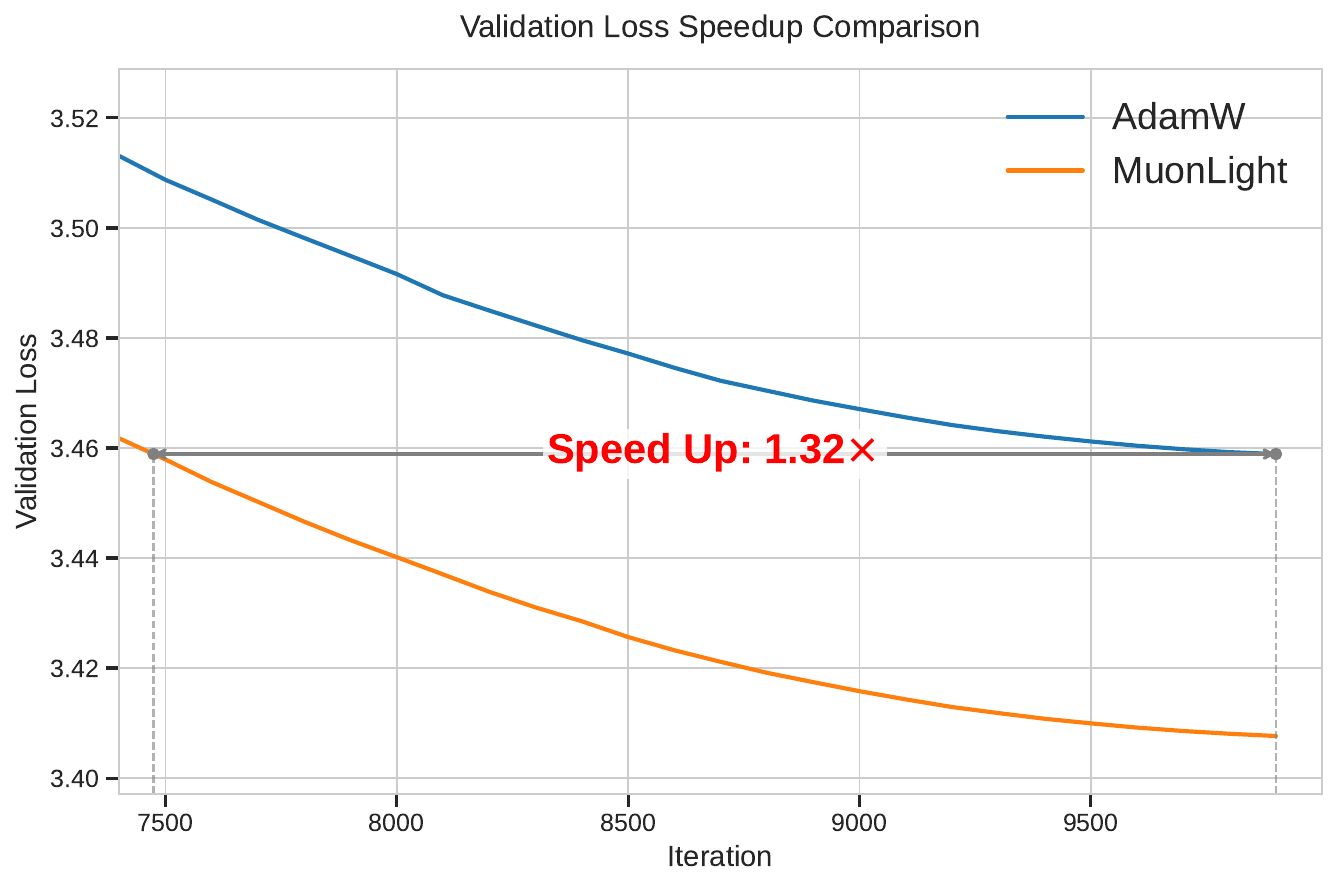}
        \label{fig:muon_speedup}
    }
    \caption{Training and validation curves for nanoGPT on C4, together with zoom-in speedup comparisons between sign-based methods and AdamW. The first three panels show training loss, validation loss, and validation accuracy. The last two panels show that Lion and Muon achieve \textcolor{red}{1.07$\times$} and \textcolor{red}{1.32$\times$ \faRocket} speedup over AdamW, respectively, following the common speedup-ratio computation practice; see, e.g.,~\citet{wen2025fantastic}.}
    \label{figs:train_val_and_speedup}
\end{figure}

\section{Conclusion}

This work aims to provide theoretical justifications for the empirical success of sign-based optimization algorithms such as Lion and Muon, particularly in the context of training LLMs. We introduce generalized heavy-tailed gradient noise assumptions that allow the noise magnitude to scale linearly with the gradient norm, a model we empirically validate as accurately reflecting the stochastic dynamics of real-world LLM training. Under this regime, we establish sharp convergence guarantees for SignSGD, Lion, Muon, and Muonlight, achieving rates that match or surpass the best-known bounds. Crucially, our theory indicates provable advantages of sign-based methods over NSGD and AdamW in the presence of heavy-tailed stochasticity. These findings are corroborated by LLM pretraining experiments, which confirm the practical superiority of sign-based optimizers.

\acks{We sincerely appreciate the valuable feedback provided by Weizhong Zhang and Yuxing Liu.}

\bibliography{ref}

\appendix

\crefalias{section}{appendix}
\crefalias{subsection}{appendix}
\crefalias{subsubsection}{appendix}

\section{Provable Complexity Improvement of Sign Gradient Descent over Normalized Gradient Descent\label{sec:comparisons}}

This section provides a theoretical comparison between the complexities of sign-based methods and Normalized Stochastic Gradient Descent (NSGD)~\citep{nesterov1984minimization,hazan2015beyond,you2017large,You2020Large,cutkosky2020momentum}, highlighting scenarios where the former offers provable advantages. The comparison is motivated by two key observations: (i) NSGD is a robust baseline for heavy-tailed noise~\citep{liu2025nonconvex, pmlr-v258-hubler25a}, and (ii) both methods can be interpreted as steepest descent algorithms—or Linear Minimization Oracles (LMOs)—operating under distinct geometric constraints~\citep{bernstein2024old, pethick2025training, sfyraki2025lions}. Consistent with the main text, both methods incorporate momentum, which is essential for tight convergence rates~\citep{cutkosky2020momentum}.

To quantify these geometric differences, we utilize the notion of \emph{vector density function} following~\citet{bernstein2018signsgd, jiang2024convergence}:
\begin{align*}
    \phi_q(\v):=\frac{\norm{\v}_1}{\norm{\v}_q}\in[1,d^{1-\frac{1}{q}}],\forall \v\in\R^d,q\in[1,\infty],
\end{align*}
where a higher $\phi_q(\v)$ indicates a denser vector. To compare different stationary measures, we define the trajectory-wide density as $\phi_2(\nabla_T):=\min_{t\in[T]}\phi_2(\nabla f(\x_t))$. This concept extends naturally to the matrix setting via the \emph{matrix density function}:
\begin{align*}
    \psi_q(\YB):=\frac{\norm{\YB}_{S_1}}{\norm{\YB}_{S_q}}\in[1,\brac{\rank{\YB}}^{1-\frac{1}{q}}],\forall \YB\in\R^{m\times n},q\in[1,\infty],
\end{align*}
where $\norm{\cdot}_{S_q}$ denotes the Schatten $q$-norm (the $\ell_q$-norm of the singular values). Notably, $\norm{\cdot}_{S_1}, \norm{\cdot}_{S_2}$, and $\norm{\cdot}_{S_\infty}$ correspond to the nuclear norm $\norm{\cdot}_*$, Frobenius norm $\norm{\cdot}_\Fn$, and operator norm $\norm{\cdot}_\op$, respectively. Intuitively, the matrix density function $\psi_q(\cdot)$ measures the distribution of singular values, capturing the ``flatness" of the matrix spectrum. We define the matrix trajectory density as $\psi_2(\bnabla_T) := \min_{t\in[T]} \psi_2(\nabla f(\XB_t))$. To simplify the comparison, we focus on the case where $\lb_1=\Lb_1 = \bsigma_1 = \V_1 = \0$. The complete versions of normalized gradient descent for vectors and matrices are listed in~\cref{alg:nsgd,alg:mnsgd}, respectively.

\begin{figure}[t]
    \begin{minipage}[t]{0.47\textwidth}
        \begin{algorithm}[H] 
            \caption{Normalized Stochastic Gradient Descent (NSGD)}
            \label{alg:nsgd}
            \begin{algorithmic}[1]
                \STATE {\bfseries Input:} $T\in\N$, $\x_1\in\R^d$, $\eta\in\R_+$
                \FOR{$t = 1$ {\bfseries to} $T$}
                    \STATE $\g_t=\frac{1}{B}\sum_{b=1}^B\g_t^b$
                    \STATE $\m_t = \beta \m_{t-1 } + (1-\beta) \g_t$ \hfill\COMMENT{$\m_0:=\g_1$}
                    \STATE $\x_{t+1} = \x_t - \eta \frac{\m_t}{\norm{\m_t}_2}$
                \ENDFOR
            \end{algorithmic}
        \end{algorithm}
    \end{minipage}
    \hfill
    \begin{minipage}[t]{0.53\textwidth}
        \begin{algorithm}[H] 
            \caption{Matrix Normalized Stochastic Gradient Descent (MNSGD)}
            \label{alg:mnsgd}
            \begin{algorithmic}[1]
                \STATE {\bfseries Input:} $T\in\N$, $\x_1\in\R^d$, $\eta\in\R_+$, $\lambda\in\R_+$
                \FOR{$t = 1$ {\bfseries to} $T$}
                    \STATE $\Gb_t=\frac{1}{B}\sum_{b=1}^B\Gb_t^b$
                    \STATE $\MB_t = \beta_1 \MB_{t-1} + (1-\beta_1) \Gb_t$ \hfill\COMMENT{$\MB_0:=\Gb_1$}
                    \STATE $\XB_{t+1} = \XB_t - \eta \frac{\MB_t}{\norm{\MB_t}_\Fn}$                 
                \ENDFOR
            \end{algorithmic}
        \end{algorithm}
    \end{minipage}
\end{figure}

\paragraph{Vector Optimization}
Under the conditions of~\cref{ass:generalized-smooth,ass:heavy-tailed-noise}, the objective $f$ is effectively $\norm{\lb_0}_\infty$-smooth with $\norm{\bsigma_0}_2$-heavy-tailed noise\footnote{Technically speaking, we reparameterize the original $\sigma_0$ by $\norm{\bsigma_0}_2$, which shares the same range as $\sigma_0$ in $\phi_2(\bsigma_0)$ and does not affect our later discussions.}. As established in~\citet{liu2025nonconvex},~\cref{alg:nsgd} achieves a complexity of $O(\Delta_f\norm{\lb_0}_\infty\norm{\bsigma_0}_2^{\frac{p}{p-1}}\epsilon^{-\frac{3p-2}{p-1}})$ for finding $\ell_2$-stationary points, i.e., $\E\sqbrac{\norm{\nabla f(\x)}_2}\le\epsilon$. Using density functions to convert the complexity notion, we obtain the following bounds for identifying $\ell_2$-stationary points.
\begin{align*}
    \text{SignSGD \& Lion (\cref{thm:signsgd,thm:lion}): }\quad&O\brac{\frac{\Delta_f\norm{\lb_0}_\infty\norm{\bsigma_0}_2^{\frac{p}{p-1}}\textcolor{red}{\phi_\infty(\lb_0)\brac{\phi_2(\bsigma_0)}^{\frac{p}{p-1}}}}{\epsilon^{\frac{3p-2}{p-1}}\textcolor{red}{\brac{\phi_2(\nabla_T)}^{\frac{3p-2}{p-1}}}}}\\
    \text{NSGD~\citep{liu2025nonconvex}: }\quad&O\brac{\frac{\Delta_f\norm{\lb_0}_\infty\norm{\bsigma_0}_2^{\frac{p}{p-1}}}{\epsilon^{\frac{3p-2}{p-1}}}}
\end{align*}
Evidently, the relative complexity is governed by the ratio
\begin{align*}
    R=R_1(R_2)^{\frac{p}{2(p-1)}},\quad\text{where }R_1:=\frac{\phi_\infty(\lb_0)}{\phi_2^2(\nabla_T)}\text{ and }R_2=\frac{\phi_2^2(\bsigma_0)}{\phi_2^2(\nabla_T)}.
\end{align*}
Based on the empirical evidence in~\citet[Figure~1]{bernstein2018signsgd}, we shall see that $R_2$ is a mild constant ($R_2\le5$). For the dominant factor $R_1$, the extensive experiments on language modeling as well as computer vision tasks in~\citet{bernstein2018signsgd,dong2024convergence,JMLR:v26:24-0523} show that the gradients $\cbrac{\nabla f(\x_t)}_{t\in[T]}$ along the optimization trajectory remain dense, and the ratio $\phi_2(\nabla_T)$ is close to $\Theta(\sqrt{d})$. Since $\phi_\infty(\lb_0)\in[1,d]$, so we have that $R_1=O(1)$ in the worst case. When the curvature vector $\lb_0$ exhibits axis-alignment properties~\citep{balles2020geometry} such that $\phi_\infty(\lb_0)\approx1$, SignSGD and Lion achieve a remarkable dimension-wise speedup of factor $d$ over NSGD.

\paragraph{Matrix Optimization}
Consider the MNSGD method shown in~\cref{alg:mnsgd}. The trajectory of~\cref{alg:mnsgd} is equivalent to~\cref{alg:nsgd} since $\norm{\YB}_\Fn = \norm{\text{vec}(\YB)}_2$. Under~\cref{ass:generalized-smooth-matrix,ass:heavy-tailed-noise-matrix} and by the same reparameterization trick as in the vector case, we can deduce that the objective $f$ is $\norm{\Lb_0}_\op$-smooth and $\norm{\V_0}_\Fn$-heavytailed. The results in~\citet{liu2025nonconvex} suggest that MNSGD require $O(\Delta_f\norm{\Lb_0}_\op\norm{\V_0}_\Fn^{\frac{p}{p-1}}\epsilon^{-\frac{3p-2}{p-1}})$ iterations to reach $\E\sqbrac{\norm{\nabla f(\XB)}_\Fn}\le\epsilon$. Leveraging matrix density functions, we obtain the following complexity comparison:
\begin{align*}
    \text{Muon \& Muonlight (\cref{thm:muon,thm:muonlight}): }\quad&O\brac{\frac{\Delta_f\norm{\Lb_0}_\op\norm{\V_0}_\Fn^{\frac{p}{p-1}}\textcolor{red}{\psi_\infty(\Lb_0)\brac{\psi_2(\V_0)}^{\frac{p}{p-1}}}}{\epsilon^{\frac{3p-2}{p-1}}\textcolor{red}{\brac{\psi_2(\bnabla_T)}^{\frac{3p-2}{p-1}}}}}\\
    \text{MNSGD~\citep{liu2025nonconvex}: }\quad&O\brac{\frac{\Delta_f\norm{\Lb_0}_\op\norm{\V_0}_\Fn^{\frac{p}{p-1}}}{\epsilon^{\frac{3p-2}{p-1}}}}
\end{align*}
As in the vector case, the comparison is dominated by the ratios below:
\begin{align*}
    R=R_1(R_2)^{\frac{p}{2(p-1)}},\quad\text{where }R_1:=\frac{\psi_\infty(\Lb_0)}{\psi_2^2(\bnabla_T)}\text{ and }R_2=\frac{\psi_2^2(\V_0)}{\psi_2^2(\bnabla_T)}.
\end{align*}
The Hessians in modern deep neural networks (DNN) are typically low-rank~\citep{sagun2016eigenvalues,sagun2017empirical,wu2020dissecting,an2025asgo}, implying $\psi_\infty(\Lb_0)$ is close to $1$. Furthermore, while gradients in DNNs frequently exhibit low-rank structures~\citep{gur2018gradient,zhao2022zero,cosson2023low,yang2023spectral}, recent evidence~\citep{songICLR2025does} suggests that gradient components aligning with the low-rank eigenspace may not effectively reduce training loss. When the variance and gradient matrices share a similar effective rank---particularly under Muon or Muonlight, which balance singular values via orthogonalization---the gradients can maintain a high effective rank~\citep{pan2025unbiased}. In this scenario, $R_1 \approx 1/\min\{m,n\}$ and $R_2$ remains a mild constant, highlighting a significant $\min\{m,n\}$ complexity improvement of Muon over MNSGD.

\begin{table}[t]
\centering
\caption{Complexity comparison to find an $\epsilon$-stationary point.}
\label{tab:complexity-comparison}
\renewcommand{\arraystretch}{2.0}
\begin{tabular}{@{}ccccc@{}}
\toprule
\textbf{Setting} & \textbf{Algorithm} & \textbf{Criterion} & \textbf{Complexity} & \textbf{Improvement} \\ \midrule
\multirow{2}{*}[-0.75em]{Vector} & NSGD & $\E[\norm{\nabla}_2] \le \epsilon$ & $O\left(\frac{\Delta_f \norm{\lb_0}_\infty \norm{\bsigma_0}_2^{\frac{p}{p-1}}}{\epsilon^{\frac{3p-2}{p-1}}}\right)$ & $=$ lower bound \\ \cmidrule{2-5}
 & \multicolumn{1}{c}{%
   {\renewcommand{\arraystretch}{1.0}%
   \begin{tabular}{@{}c@{}}
     SignSGD\\ 
     \& Lion
   \end{tabular}}} & $\E[\norm{\nabla}_1] \le \epsilon$ & $O\left(\frac{\Delta_f \norm{\lb_0}_1 \norm{\bsigma_0}_1^{\frac{p}{p-1}}}{\epsilon^{\frac{3p-2}{p-1}} }\right)$ & Up to \textcolor{red}{$d$} \\ \midrule
\multirow{2}{*}[-0.75em]{Matrix} & MNSGD & $\E[\norm{\bnabla}_\Fn] \le \epsilon$ & $O\left(\frac{\Delta_f \norm{\Lb_0}_\op \norm{\V_0}_\Fn^{\frac{p}{p-1}}}{\epsilon^{\frac{3p-2}{p-1}}}\right)$ & $=$ lower bound \\ \cmidrule{2-5}
 & \multicolumn{1}{c}{%
   {\renewcommand{\arraystretch}{1.0}%
   \begin{tabular}{@{}c@{}}
     Muon \&\\ 
    Muonlight
   \end{tabular}}} & $\E[\norm{\bnabla}_*] \le \epsilon$ & $O\left(\frac{\Delta_f \norm{\Lb_0}_* \norm{\V_0}_*^{\frac{p}{p-1}}}{\epsilon^{\frac{3p-2}{p-1}} }\right)$ & Up to \textcolor{red}{$\min\{m,n\}$} \\ \bottomrule
\end{tabular}
\end{table}

\paragraph{Lower Bounds}
The comparisons above follow the approach of \citet{bernstein2018signsgd} by evaluating the upper bounds of both method classes. However, this argument can be significantly strengthened by comparing the \emph{upper bounds} of sign-based optimizers directly against the \emph{lower bounds} of gradient normalization. Given that the complexity results for NSGD and MNSGD are tight~\citep{liu2025nonconvex}, this comparison highlights a fundamental, provable advantage of sign-descent methods. Similar to the approach in~\citet{jiang2024convergence}, evaluating our upper bounds against established lower bounds rigorously demonstrates the superiority of sign-based optimization in these specific geometric settings.

We consolidate these findings in~\cref{tab:complexity-comparison}. This table explicitly illustrates how sign-based algorithms leverage problem geometry to achieve superior convergence rates compared to the established lower bounds of normalized gradient methods.

\section{Further Discussions on Assumptions\label{sec:discussion}}

In this section, we complete the omitted discussions in~\cref{sec:vector-sign,sec:matrix-sign}.

\subsection{Generalized Smoothness: Assumption~\ref{ass:generalized-smooth}\label{sec:discussion-gs}}
\renewcommand{\theass}{\arabic{ass}c}
\renewcommand{\theHass}{vecsmooth.\arabic{ass}c}
\setcounter{ass}{1}

Below, we briefly compare~\cref{ass:generalized-smooth} with other $(\lb_0,\lb_1)$-\emph{coordinate-wise smooth} conditions. Consider the generalized smoothness model proposed in~\citet[Assumption~2]{crawshaw2022robustness}:
\begin{align}
    \abs{\nabla_if(\x^\prime)-\nabla_if(\x)}\le\brac{\lb_{0,i}+\lb_{1,i}\abs{\nabla_if(\x)}}\norm{\x^\prime-\x}_2,\forall\ i\in[d],\norm{\x^\prime-\x}_2\le\frac{1}{\norm{\lb_1}_\infty}.\label{eq:coordinate-smooth}
\end{align}
Evidently, the above assumption is stronger than~\cref{ass:generalized-smooth} in the sense that it requires~\eqref{eq:coordinate-smooth} to hold for all coordinates, and that the requested domain is larger as $\norm{\x^\prime-\x}_\infty\le\norm{\x^\prime-\x}_2$. Also,~\eqref{eq:coordinate-smooth} will incur an explicit dimensional factor of $d$, which is unfavorable for LLMs.~\citet{liu2025adagrad} further refined the condition in~\eqref{eq:coordinate-smooth} into
\begin{equation}
    \norm{\nabla f(\x^\prime)-\nabla f(\x)}_{\brac{\lb_0+\lb_1\odot\abs{\nabla f(\x)}}^{-1}}\le\norm{\x^\prime-\x}_{\lb_0+\lb_1\odot\abs{\nabla f(\x)}},\quad\forall\norm{\x^\prime-\x}_{\lb_1^2}\le\sqrt{d}.\label{eq:anisotropic-smooth}
\end{equation}
By standard calculus, it's easy to see that~\eqref{eq:anisotropic-smooth} implies~\cref{ass:generalized-smooth-variant}, a variant of~\cref{ass:generalized-smooth} to be introduced later. Note that all of our theoretical results also hold under~\cref{ass:generalized-smooth-variant} (see discussions below~\cref{ass:generalized-smooth-variant}). We deduce that our smoothness model based on the quadratic bound~\eqref{eq:vector-smooth} is generally weaker than the gradient curvature bound in~\eqref{eq:coordinate-smooth} and~\eqref{eq:anisotropic-smooth}. On the other hand,~\eqref{eq:coordinate-smooth} leverages $\ell_2$-norm while~\eqref{eq:anisotropic-smooth} utilizes $\lb_1^2$-weighted norm, both of which fail to align with the geometry of sign gradient descent~\citep{bernstein2024old,bernstein2025modular}. On the contrary, our $\ell_\infty$-norm formulation perfectly matches the interpretation of sign descent from the perspective of linear minimization oracle~\citep{pethick2025training}. Lastly, the empirical validations in~\citet{crawshaw2022robustness,liu2025adagrad} confirms the practical value of~\cref{ass:generalized-smooth}.

The following condition is a variant of~\cref{ass:generalized-smooth}, and can be inferred from~\eqref{eq:anisotropic-smooth}.
\begin{ass}[Variant of~\cref{ass:generalized-smooth}]\label{ass:generalized-smooth-variant}
    There exists non-negative vectors $\lb_0=[\lb_{0,1},\cdots,\allowdisplaybreaks\\\lb_{0,d}]\in\R^d_+$ and $\lb_1=[\lb_{1,1},\cdots,\lb_{1,d}]\in\R^d_+$ such that for all $\norm{\x^\prime-\x}_{\lb_1^2}\le1/\sqrt{d}$\footnote{The requirement $\norm{\x^\prime-\x}_{\lb_1^2}\le\sqrt{d}$ here as well as in~\eqref{eq:anisotropic-smooth} is stated as $\norm{\x^\prime-\x}_{\lb_1}\le\sqrt{d}$ in~\citet{liu2025adagrad}, which is not correct after private communications with the authors of~\citet{liu2025adagrad}. Thus, we discuss under the right condition here.}: 
    \begin{align}\label{eq:generalized-smooth-variant}
    \abs{f(\x^\prime)-\brac{f(\x)+\inner{\nabla f(\x)}{\x^\prime-\x}}}\le\frac{1}{2}\norm{\x^\prime-\x}_{\lb_0+\lb_1\odot\abs{\nabla f(\x)}}^2.
    \end{align} 
\end{ass}
Ignoring the absolute value on the LHS of~\eqref{eq:generalized-smooth-variant}, then~\cref{ass:generalized-smooth-variant} admits the exact same form as Lemma~C.3 in~\citet{liu2025adagrad}, which is derived under condition~\eqref{eq:anisotropic-smooth}~\citep[Assumption~5.1]{liu2025adagrad}. According to the textbook analysis~\citep{nesterov2018lectures}, we can expand their proof to derive~\eqref{eq:generalized-smooth-variant}, which indicates that the type of gradient curvature bound in~\eqref{eq:coordinate-smooth} and~\eqref{eq:anisotropic-smooth} is \emph{stronger} than the function value quadratic bound in~\cref{ass:generalized-smooth,ass:generalized-smooth-variant}. To illustrate that the theoretical guarantee in~\cref{sec:vector-sign} still holds under~\cref{ass:generalized-smooth-variant}, note that the only difference between~\cref{ass:generalized-smooth,ass:generalized-smooth-variant} lies in the condition $\norm{\x^\prime-\x}_\infty\le1/\norm{\lb_1}_\infty$ versus $\norm{\x^\prime-\x}_{\lb_1^2}\le\sqrt{d}$. Delving into the analysis of SignSGD in~\cref{sec:proof-signsgd},~\cref{ass:generalized-smooth} is used under
\begin{align*}
    \eta\le1/\norm{\lb_1}_\infty\Longrightarrow\norm{\x_{t+1}-\x_t}_\infty=\eta\norm{\sign{\m_t}}_\infty\le1/\norm{\lb_1}_\infty.
\end{align*}
Since we have
\begin{align*}
    \eta\le1/\norm{\lb_1}_\infty\Longrightarrow\norm{\x_{t+1}-\x_t}_{\lb_1^2}=\eta\sqrt{\sum_{i=1}^d\brac{\sign{\m_{t,i}}^2\cdot\lb_{1,i}^2}}\le\frac{\norm{\lb_1}_2}{\norm{\lb_1}_\infty}\le\sqrt{d},
\end{align*}
so we can replace~\cref{ass:generalized-smooth} by~\cref{ass:generalized-smooth-variant}. The same arguments are still valid for Lion, which we omit here for brevity.

\subsection{Generalized Smoothness: Assumption~\ref{ass:generalized-smooth-matrix}\label{sec:discussion-gs-matrix}}

For completeness, we also provide the following assumption for Muon and Muonlight.
\renewcommand{\theass}{\arabic{ass}d}
\renewcommand{\theHass}{matsmooth.\arabic{ass}d}
\setcounter{ass}{1}
\begin{ass}[Variant of~\cref{ass:generalized-smooth-matrix}]\label{ass:generalized-smooth-matrix-variant}
    There exists non-negative constants $\norm{\Lb_0}_*,\norm{\Lb_1}_\op\in\R$ such that for all $\norm{\XB^\prime-\XB}_\op\le1/\norm{\Lb_1}_\op$: 
    \begin{align}\label{eq:generalized-smooth-matrix-variant}
    \norm{\nabla f(\XB^\prime)-\nabla f(\XB)}_*\le\brac{\norm{\Lb_0}_*+\norm{\Lb_1}_\op\norm{\nabla f(\XB)}_*}\norm{\XB^\prime-\XB}_\op.
    \end{align} 
\end{ass}
\cref{ass:generalized-smooth-matrix-variant} mirrors the original $(L_0, L_1)$-smoothness formulation in~\citet{Zhang2020Why} by utilizing scalar constants $(\norm{\Lb_0}_*, \norm{\Lb_1}_\op)$. This is a milder condition than \cref{ass:generalized-smooth-matrix}, as the latter necessitates the specification of two full matrices $\Lb_0$ and $\Lb_1$. To demonstrate that \cref{ass:generalized-smooth-matrix-variant} is a weaker requirement, consider the case where \cref{ass:generalized-smooth-matrix} holds. We have
\begin{align*}
    &\norm{\nabla f(\XB^\prime)-\nabla f(\XB)}_*\overset{\textnormal{\cref{lem:matrix-cauchy-schwarz}}}{\le}\sqrt{\norm{\Lb(\XB)}_*}\norm{\nabla f(\XB^\prime)-\nabla f(\XB)}_{\brac{\Lb(\XB)}^{-1}}\\\overset{\textnormal{\cref{ass:generalized-smooth-matrix}}}{\le}&\sqrt{\norm{\Lb(\XB)}_*}\norm{\XB^\prime-\XB}_{\Lb(\XB)}=\sqrt{\norm{\Lb(\XB)}_*\tr{\brac{\XB^\prime-\XB}^\top\Lb(\XB)\brac{\XB^\prime-\XB}}}\\\le&\sqrt{\norm{\Lb(\XB)}_*\cdot\brac{\norm{\Lb(\XB)}_*\norm{\brac{\XB^\prime-\XB}\brac{\XB^\prime-\XB}^\top}_\op}}\\\overset{\textnormal{\cref{lem:trace-property}}}{\le}&\norm{\Lb(\XB)}_*\norm{\XB^\prime-\XB}_\op\le\brac{\norm{\Lb_0}_*+\norm{\Lb_1}_\op\norm{\nabla f(\XB)}_*}\norm{\XB^\prime-\XB}_\op,
\end{align*}
suggesting that~\cref{ass:generalized-smooth-matrix} implies~\cref{ass:generalized-smooth-matrix-variant}. In order to be more consistent with prior work~\citep{an2025asgo,kovalev2025non} and address the matrix structures, we adopt the formulation in the main text.

\subsection{Heavy-Tailed Noise: Assumption~\ref{ass:heavy-tailed-noise}\label{sec:discussion-ht}}

Our theoretical justification mainly follows from~\citet{liu2025nonconvex}. The following lemma shows that~\cref{ass:heavy-tailed-noise} is meaningful for non-zero $\bsigma_1$.

\begin{lem}[Based on Example~A.1 in~\citet{liu2025nonconvex}]
    Given $\x_*=[x_{*,1},\cdots,x_{*,d}]\in\R^d$, define the separable function $h(\x):=\sum_{i=1}^dh_i(x_i),\forall\x=[x_{1},\cdots,x_{d}]\in\R^d$ with
    \begin{align*}
        h_i(x_i):=\frac{1}{2}\E\sqbrac{\brac{a_ix_i-b_i}^2},\quad a_i\sim\textnormal{Bernoulli}(0.5),b=a_ix_{*,i}+\xi_i,
    \end{align*}
    where $\xi_i$ is a centered random variable independent of $a_i$ and further satisfy $\E\sqbrac{\abs{\xi_i}^p}\le\sigma_i^p$ for some $\sigma_i\ge0$. Then, for any $h$ defined above, it holds that 
    \begin{enumerate}
        \item $\bsigma_1=\0$:~\cref{ass:heavy-tailed-noise} can not be satisfied.
        \item $\bsigma_1>\0$:~\cref{ass:heavy-tailed-noise} is satisfied with $\bsigma_{0,i}=\sigma2^{1-\frac{2}{p}},\bsigma_{1,i}=2^{-\frac{1}{p}}+2^{1-\frac{2}{p}},\forall i\in[d]$.
    \end{enumerate}
\end{lem}

\begin{proof}
    The stochastic gradient, as well as the true gradient at point $\x$ is given by
    \begin{align*}
        \g_i=a_i\brac{a_ix_i-b_i},\quad\nabla_if(\x)=\E[a_i^2]\x-\E[a_ib_i].
    \end{align*}
    Then, the function $h_i$ is reduced to the one-dimensional case in~\citet[Example~A.1]{liu2025nonconvex}. Thus, we apply their results to deduce that when $\bsigma_1=\0$, for all $h_i,i\in[d]$,~\cref{ass:heavy-tailed-noise} can not be satisfied. Also, their results indicate that for any $i\in[d]$, we can select $\bsigma_{0,i}^p=2^{p-2}\sigma^p$ and $\bsigma_{1,i}^p=0.5+2^{p-2}$. Due to our separate construction of $h$ and the independence between all coordinates, we can extend this conclusion to finish the second part of the proof.
\end{proof}

\subsection{Heavy-Tailed Noise: Assumption~\ref{ass:heavy-tailed-noise-matrix}\label{sec:discussion-ht-matrix}}
The following assumption is a relaxation of~\cref{ass:heavy-tailed-noise-matrix}.
\renewcommand{\theass}{\arabic{ass}c}
\renewcommand{\theHass}{matnoise.\arabic{ass}c}
\setcounter{ass}{3}
\begin{ass}[Variant of~\cref{ass:heavy-tailed-noise-matrix}]\label{ass:heavy-tailed-noise-matrix-variant}
    There exists $p\in(1,2]$, $\V_0\in\R^{m\times n}$ with full row rank, and $\norm{\V_1}_\op\in\R_+$ such that 
    \begin{align}\label{eq:heavy-tailed-noise-matrix-variant}
        \E\sqbrac{\left.\norm{\Gb_t^b-\nabla f(\XB_t)}^{p}_{\mabs{\V_0}^{-1}}\right|\F_{t-1}}\le\norm{\V_0}_*^{p/2}+\frac{\norm{\V_1}_{\op}^p\norm{\nabla f(\XB_t)}_*^p}{\norm{\V_0}_*^{p/2}},\quad\forall b\in[B].
    \end{align}  
\end{ass}
By the Cauchy-Schwarz inequality, $|\inner{\V_1}{\nabla f(\XB_t)}|^p \le \norm{\V_1}_\op^p \norm{\nabla f(\XB_t)}_*^p$, which implies that~\cref{ass:heavy-tailed-noise-matrix} is a stronger condition than~\cref{ass:heavy-tailed-noise-matrix-variant}. The latter is more amenable to empirical verification, as it replaces the full matrix $\V_1$ with a scalar coefficient $\norm{\V_1}_\op$, an approach we adopt in \cref{sec:experiments}. We also underline that in both assumptions, the term $\norm{\V_0}_*^{p/2}$ is introduced to maintain the homogeneity of the $p$th moment, ensuring that the scaling of the $\mabs{\V_0}^{-1}$-weighted metric remains consistent with the noise magnitude.

Crucially, the use of the $\norm{\cdot}_{\mabs{\V_0}^{-1}}$ weighted norm allows our model to capture the anisotropic structural properties of the noise~\citep{xie2025tale,an2025asgo}. Specifically,~\citet[Assumption~3]{an2025asgo} and~\citet[Assumption~3]{pan2025unbiased} assume a PSD constraint on the noise covariance: $\E[\brac{\Gb_t^b-\nabla f(\XB_t)}\brac{\Gb_t^b-\nabla f(\XB_t)}^\top|\F_{t-1}]\preceq \V_0\V_0^\top$. This condition directly implies $\E\sqbrac{\left. \norm{\Gb_t^b - \nabla f(\XB_t)}^2_{\mabs{\V_0}^{-1}} \right| \F_{t-1}} \le \norm{\V_0}_*$~\citep[Proposition~A.10]{xie2025tale}, illustrating how our framework encapsulates the geometry of the gradient's principal directions.

\subsection{Exact Newton--Schulz Oracle\label{sec:NewtonSchulz}}

In practice,~\cref{alg:muon,alg:muonlight} often runs in $q=5$ Newton--Schulz iterations~\citep{kovarik1970some,bjorck1971iterative}. The numerical error of the Newton--Schulz algorithm has been observed to exert little error on the optimization trajectory~\citep{jordan2024muon,liu2025muon}.~\citet{kim2026convergence} investigate this phenomenon theoretically and show that Muon with inexact Newton--Schulz converges at the same rate compared to the exact SVD realization, up to a constant factor which converges to $1$ \emph{double exponentially} in $q$. After reading their proof, it should be evident that their analysis can serve as a black box compatible with ours. It suffices to adapt~\eqref{eq:muon-onestep} in~\cref{sec:proof-muon} to equation~(5) in~\citet{kim2026convergence}. Ultimately, we only need to pay the same constant factor for the convergence of Muon without assuming exact Newton--Schulz oracle. The same technique can be extended to other numerical algorithms that compute $\msign{\cdot}$ in Muon, e.g., the recent hit \texttt{PolarExpress}~\citep{amsel2025polar}.

\section{Analysis for SignSGD and Lion\label{sec:vector-sign-analysis}}

To facilitate the theoretical analysis, we define
\begin{align}\label{eq:momentum-def}
    \eps_t:=\m_t-\nabla f(\x_t),\quad \n_t:=\g_t-\nabla f(\x_t),\quad \s_t:=\nabla f(\x_{t-1})-\nabla f(\x_t).
\end{align}

\subsection{Vector Calculus}

\begin{lem}[Minkowski's inequality]\label{lem:minkowski}
    For any $p\ge1,x,y\in\R_+$, it holds that $(x+y)^{1/p}\le x^{1/p}+y^{1/p}$.
\end{lem}

\begin{proof}
    Applying Minkowski's inequality $\norm{\x+\y}_p\le\norm{\x}_p+\norm{\y}_p$ for two-dimensional vectors $\x=(x^{1/p},0),\x=(0,y^{1/p})$ yields the result.
\end{proof}

\begin{lem}[Jensen's inequality for $\ell_p$-means]\label{lem:lp-mean}
    Let $p\in[1,+\infty]$ and let $X_1,\cdots,X_n$ be non-negative scalar random variables on some probability space, then
    \begin{align*}
        \E\sqbrac{\sum_{i=1}^nX_i^p}^{1/p}\le\brac{\sum_{i=1}^n\E\sqbrac{X_i^p}}^{1/p}
    \end{align*}
\end{lem}

\begin{proof}
    Let $Y = (X_1, \dots, X_n)$ be a random vector in $\mathbb{R}^n_+$. By Jensen's inequality as well as the convexity of $x\mapsto x^p,p\ge1$:
    \begin{align*}
        \brac{\E\sqbrac{\norm{Y}_p}}^{p}\le\E\sqbrac{\norm{Y}_p^p}=\E\sqbrac{\sum_{i=1}^n X_i^p}=\sum_{i=1}^n \E[X_i^p].
    \end{align*}
    Taking the \( 1/p \)-th power on both sides yields the result.
\end{proof}

\begin{lem}[Sign-difference Bound]\label{lem:sign-difference}
    For any $\x,\y\in\R^d$, it holds that $\inner{\x}{\sign{\x}-\sign{\y}}\le2\norm{\x-\y}_1$.
\end{lem}

\begin{proof}
    Let $\x=[\x_1,\cdots,\x_d]\in\R^d,\y=[\y_1,\cdots,\y_d]\in\R^d$, then
    \begin{align*}
        &\inner{\x}{\sign{\x}-\sign{\y}}=\sum_{i=1}^d\x_i\cdot\sqbrac{\sign{\x_i}-\sign{\y_i}}\le\sum_{i=1}^d2\abs{\x_i}\cdot\ind\brac{\sign{\x_i}\ne\sign{\y_i}}\\\le&\sum_{i=1}^d2\abs{\x_i-\y_i}\cdot\ind\brac{\sign{\x_i}\ne\sign{\y_i}}\le\sum_{i=1}^d2\abs{\x_i-\y_i}=2\norm{\x-\y}_1.
    \end{align*}
\end{proof}

\subsection{Technical Lemmas}
The following lemma with $\lb_1=\0$ is originally proven in~\citet[Lemma~F.3]{bernstein2018signsgd}] with some inaccuracies, which we fix below with a clean pathwise analysis and extend to general $\lb_1$ at the same time.
\begin{lem}\label{lem:smooth-grad-sign}
    Under~\cref{ass:generalized-smooth}, for any sign vector $\s\in\cbrac{-1,1}^d$ and any $\eta\le1/\norm{\lb_1}_\infty$, 
    \begin{align*}
        \norm{\nabla f(\x+\eta\s)-\nabla f(\x)}_1\le2\eta\brac{\norm{\lb_0}_1+\inner{\lb_1}{\abs{\nabla f(\x)}}}.
    \end{align*}
\end{lem}

\begin{proof}
    If $\eta=0$, the claim is immediate. Hence, assume $\eta>0$. Denote
    \[
        \z_t:=\x+t\eta\s,\qquad
        \q(t):=\nabla f(\z_t),\qquad
        Q(t):=\norm{\q(t)-\q(0)}_1,
    \]
    and define
    \[
        \DB_t:=\diag{\lb_0+\lb_1\odot\abs{\q(t)}},\qquad t\in[0,1].
    \]
    By~\cref{ass:generalized-smooth}, at every point on the segment where $f$ is twice-differentiable,
    \[
        -\DB_t\preceq \nabla^2 f(\z_t)\preceq \DB_t .
    \]
    Since $\q'(t)=\eta\nabla^2 f(\z_t)\s$ for almost every $t\in[0,1]$, we have
    \begin{align*}
        \frac{1}{\eta}\norm{\q'(t)}_1
        &= \norm{\nabla^2 f(\z_t)\s}_1
         = \max_{\u\in\cbrac{-1,1}^d}\inner{\u}{\nabla^2 f(\z_t)\s}.
    \end{align*}
    We next show that the last quantity is bounded by $\tr{\DB_t}$. Fix any
    $\u,\s\in\cbrac{-1,1}^d$, and set
    \[
        \a:=\frac{\u+\s}{2},\qquad
        \b:=\frac{\u-\s}{2}.
    \]
    Then $\u=\a+\b$, $\s=\a-\b$, and $\a_i^2+\b_i^2=1$ for every coordinate $i$.
    Using the symmetry of $\nabla^2 f(\z_t)$, we obtain
    \begin{align*}
        \inner{\u}{\nabla^2 f(\z_t)\s}
        &= \inner{\a+\b}{\nabla^2 f(\z_t)(\a-\b)} \\
        &= \inner{\a}{\nabla^2 f(\z_t)\a}
           -\inner{\b}{\nabla^2 f(\z_t)\b} \\
        &\le \inner{\a}{\DB_t\a}+\inner{\b}{\DB_t\b} \\
        &= \sum_{i=1}^d(\DB_t)_{ii}\brac{\a_i^2+\b_i^2}
         = \tr{\DB_t}.
    \end{align*}
    Therefore, for almost every $t\in[0,1]$,
    \[
        \frac{1}{\eta}\norm{\q'(t)}_1\le \tr{\DB_t}.
    \]
    Moreover,
    \begin{align*}
        \tr{\DB_t}
        &= \norm{\lb_0}_1+\inner{\lb_1}{\abs{\q(t)}} \\
        &\le \norm{\lb_0}_1+\inner{\lb_1}{\abs{\q(0)}}
            +\norm{\lb_1}_\infty\norm{\q(t)-\q(0)}_1 \\
        &= \norm{\lb_0}_1+\inner{\lb_1}{\abs{\nabla f(\x)}}
            +\norm{\lb_1}_\infty Q(t),
    \end{align*}
    where we used $\abs{\abs{a}-\abs{b}}\le \abs{a-b}$ coordinatewise. Hence, for every $t\in[0,1]$,
    \begin{align*}
        Q(t)
        &\le \int_0^t\norm{\q'(\tau)}_1\diff\tau \\
        &\le \eta t\brac{\norm{\lb_0}_1+\inner{\lb_1}{\abs{\nabla f(\x)}}}
        +\eta\norm{\lb_1}_\infty\int_0^t Q(\tau)\diff\tau .
    \end{align*}
    By Gr\"onwall's inequality,
    \begin{align*}
        Q(t)
        \le
        \eta\brac{\norm{\lb_0}_1+\inner{\lb_1}{\abs{\nabla f(\x)}}}
        \int_0^t e^{\eta\norm{\lb_1}_\infty(t-\tau)}\diff\tau .
    \end{align*}
    If $\norm{\lb_1}_\infty>0$, then taking $t=1$ gives
    \begin{align*}
        Q(1)
        &\le
        \brac{\norm{\lb_0}_1+\inner{\lb_1}{\abs{\nabla f(\x)}}}
        \frac{e^{\eta\norm{\lb_1}_\infty}-1}{\norm{\lb_1}_\infty} \\
        &\le
        \eta(e-1)\brac{\norm{\lb_0}_1+\inner{\lb_1}{\abs{\nabla f(\x)}}}
        \le
        2\eta\brac{\norm{\lb_0}_1+\inner{\lb_1}{\abs{\nabla f(\x)}}},
    \end{align*}
    where we used $\eta\norm{\lb_1}_\infty\le1$ and $e-1\le2$. If $\norm{\lb_1}_\infty=0$, the same conclusion follows directly from the preceding integral inequality with the last term equal to zero. Since
    \[
        Q(1)=\norm{\nabla f(\x+\eta\s)-\nabla f(\x)}_1,
    \]
    the proof is now complete.
\end{proof}

\begin{lem}[$p$th moment of mini-batch noise]\label{lem:batch-noise}
    Under~\cref{ass:unbiased,ass:heavy-tailed-noise}, the following holds for any $i\in[d],t\in[T]$:
    \begin{align*}
        \E\sqbrac{\left.\abs{\n_t}^p\right|\F_{t-1}}\le 2B^{1-p}\brac{\bsigma_{0,i}^p+\bsigma_{1,i}^p\abs{\nabla_if(\x_t)}^p},
    \end{align*}
    where $\n_t$ is defined in~\eqref{eq:momentum-def}.
\end{lem}

\begin{proof}
We denote $\n_{t,i}^b:=\g_{t,i}^b-\nabla_if(\x_t)$. So $\n_{t,i}=\frac{1}{B}\sum_{b=1}^B\n_{t,i}^b$. Under~\cref{ass:heavy-tailed-noise} and Jensen's inequality, we have
\begin{align*}
    \E\sqbrac{\left.\abs{\n_{t,i}}^p\right|\F_{t-1}}=&\E\sqbrac{\left.\abs{\frac{1}{B}\sum_{b=1}^B\n_{t,i}^b}^p\right|\F_{t-1}}\overset{\hypertarget{von-bahr-essen}{\textcolor{blue}{\#}}}{\le}2\E\sqbrac{\left.\frac{1}{B^p}\sum_{b=1}^B\abs{\n_{t,i}^b}^p\right|\F_{t-1}}\\&\le\frac{2}{B^p}\sum_{b=1}^B\brac{\bsigma_{0,i}^p+\bsigma_{1,i}^p\abs{\nabla_if(\x_t)}^p}=2B^{1-p}\brac{\bsigma_{0,i}^p+\bsigma_{1,i}^p\abs{\nabla_if(\x_t)}^p},
\end{align*}
where \hyperlink{von-bahr-essen}{\#} leverages von Bahr-Esseen inequality~\citep[Lemma~10]{pmlr-v258-hubler25a}.
\end{proof}

\begin{lem}[Regret Analysis of Diagonal AdaGrad]\label{lem:adagrad-regret}
    Given an arbitrary sequence of vectors\\ $\cbrac{\v_t}_{t=1}^T\subset\R^d,T\in\N$, there exists a sequence of vectors $\cbrac{\w_t}_{t=1}^T\subset\R^d$ such that (i) $\norm{\w_t}_\infty\le 1$, (ii) every $\w_t$ only depends on $\v_1,\cdots,\v_{t-1}$, and (iii) they further satisfy
    \begin{align*}
        \sum_{t=1}^T\inner{\v_t}{\w_t}\le\sum_{i=1}^d2\sqrt{2\sum_{t=1}^T\v_{t,i}^2}-\norm{\sum_{t=1}^T\v_t}_1.
    \end{align*}
\end{lem}

\begin{proof}
    The proof of this lemma stems from the regret analysis of AdaGrad-Norm~\citep{streeter2010less,duchi2011adaptive} in~\citet{liu2025nonconvex}, which is based on~\citet[Lemma~2]{rakhlin2017equivalence}. Here, we adapt their analysis to the diagonal version of AdaGrad~\citep{McMahanS10adagrad,duchi2011adaptive}. WLOG, we assume $\v_{1,i}^2>0$. Otherwise, define $\tau=\argmin\{t\in\N|\v_{t,i}^2>0\}$ and set $\w_{t,i}=0,\forall t\in[\tau-1]$, then start the proof at $\tau$. Define
    \begin{align}\label{eq:adagrad}
        \w_1:=\0,\quad\w_{t+1,i}:=\Pi_{[-1,1]}\sqbrac{\w_{t,i}-\gamma_t\v_{t,i}},\quad\gamma_{t,i}=\begin{cases}
            \sqrt{\frac{2}{\sum_{s=1}^t\v_{s,i}^2}},&t\ge1\\
            +\infty,&t=0
        \end{cases},\quad\forall i\in[d],
    \end{align}
    which ensures (i) and (ii) in~\cref{lem:adagrad-regret}. By standard analysis of projected online subgradient descent~\citep[Lemma~2.30]{orabona2019intro}, for any $\u_i\in[-1,1]$,
    \begin{align*}
        \v_{t,i}\cdot\brac{\w_{t,i}-\u_i}\le\frac{\brac{\u_i-\w_{t,i}}^2}{2\gamma_{t,i}}-\frac{\brac{\u_i-\w_{t+1,i}}^2}{2\gamma_{t,i}}+\frac{\gamma_{t,i}\v_{t,i}^2}{2}.
    \end{align*}
    Summing from $1$ to $T$:
    \begin{align*}
        &\sum_{t=1}^T\v_{t,i}\cdot\brac{\w_{t,i}-\u_i}\le\sum_{t=1}^T\brac{\frac{\brac{\u_i-\w_{t,i}}^2}{2\gamma_{t,i}}-\frac{\brac{\u_i-\w_{t+1,i}}^2}{2\gamma_{t,i}}+\frac{\gamma_{t,i}\v_{t,i}^2}{2}}\\=&\frac{\brac{\u_i-\w_{1,i}}^2}{2\gamma_{1,i}}-\frac{\brac{\u_i-\w_{T+1,i}}^2}{2\gamma_{T,i}}+\sum_{t=2}^T\brac{\u_i-\w_{t,i}}^2\brac{\frac{1}{2\gamma_{t,i}}-\frac{1}{2\gamma_{t-1,i}}}+\sum_{t=1}^T\frac{\gamma_{t,i}\v_{t,i}^2}{2}\\\le&\frac{1}{2\gamma_{1,i}}+\sum_{t=2}^T\brac{\frac{2}{\gamma_{t,i}}-\frac{2}{\gamma_{t-1,i}}}+\sum_{t=1}^T\gamma_{t,i}\brac{\frac{1}{2\gamma_{t,i}^2}-\frac{1}{2\gamma_{t-1,i}^2}}\\\le&\frac{2}{\gamma_{T,i}}+\sum_{t=1}^T2\brac{\frac{1}{\gamma_{t,i}}-\frac{1}{\gamma_{t-1,i}}}\le\frac{4}{\gamma_{T,i}}=2\sqrt{2\sum_{t=1}^T\v_{t,i}^2},
    \end{align*}
    where we utilize $\u_i\in[-1,1],\w_{1,i}=0,\gamma_t\le\gamma_{t-1},\u_i-\w_{t,i}\in[-2,2],1/\gamma_{0,i}=0$ in the second inequality. Rearranging the above relation and taking minimum over $\u_i\in[-1,1]$,
    \begin{align*}
        \sum_{t=1}^T\v_{t,i}\cdot\w_{t,i}\le2\sqrt{2\sum_{t=1}^T\v_{t,i}^2}+\min_{\u_i\in[-1,1]}\sum_{t=1}^T\v_{t,i}\cdot\u_{i}=2\sqrt{2\sum_{t=1}^T\v_{t,i}^2}-\abs{\sum_{t=1}^T\v_{t,i}}.
    \end{align*}
    Summing the regret across all coordinates $i\in[d]$ finishes the proof.
\end{proof}

\begin{lem}[$\ell_1$-norm vector martingale concentration,~\cref{lem:vector-concentration} in~\cref{sec:concentration-inequality}]\label{lem:MDS-L1-concentration}
    Given a sequence of integrable random vectors $\v_t\in\R^d,\forall t\in\N$ such that $\E\sqbrac{\v_t|\F_{t-1}}=\0$ where $\F_t=\sigma\brac{\v_1,\cdots,\v_t}$ is the natural filtration, then for any $p\in[1,2]$, there is 
    \begin{align*}
        \E\sqbrac{\norm{\sum_{t=1}^T\v_t}_1}\le2\sqrt{2}\sum_{i=1}^d\E\sqbrac{\norm{\v_{1:T,i}}_p},\quad\forall T\in\N,
    \end{align*}
    where $\v_{1:t,i}:=[\v_{1,i},\cdots,\v_{t,i}]\in\R^t$.
\end{lem}

\begin{proof}
    According to~\cref{lem:adagrad-regret}, there exists a sequence of random vectors $\cbrac{\w_t}_{t=1}^T\subset\R^d$ such that (i) $\norm{\w_t}_\infty\le 1$, (ii) every $\w_t\in\mathcal{F}_{t-1}$, and (iii) they further satisfy
    \begin{align*}
        \sum_{t=1}^T\inner{\v_t}{\w_t}\le\sum_{i=1}^d2\sqrt{2}\norm{\v_{1:T,i}}_2-\norm{\sum_{t=1}^T\v_t}_1.
    \end{align*}    
    Rearranging this inequality and taking expectations yields
    \begin{align*}
        \E\sqbrac{\norm{\sum_{t=1}^T\v_t}_1}\le2\sqrt{2}\sum_{i=1}^d\E\sqbrac{\norm{\v_{1:T,i}}_2}-\sum_{t=1}^T\E\sqbrac{\inner{\v_t}{\w_t}}\le2\sqrt{2}\sum_{i=1}^d\E\sqbrac{\norm{\v_{1:T,i}}_p},
    \end{align*}
    where the last step is due to $\norm{\cdot}_2\le\norm{\cdot}_p$ and the tower rule:
    \begin{align*}
        \E\sqbrac{\inner{\v_t}{\w_t}}=\E\sqbrac{\E\sqbrac{\inner{\v_t}{\w_t}|\F_{t-1}}}=\E\sqbrac{\inner{\E\sqbrac{\v_t|\F_{t-1}}}{\w_t}}=0.
    \end{align*}
\end{proof}

\begin{lem}[Stability with weight decay]\label{lem:stability-wd}
    Running~\cref{alg:lion} with 
    \begin{align*}
        \lambda\le\frac{1}{\eta}\brac{1-\frac{1}{2^{1/T}}},\text{ and }\norm{\x_1}_\infty\le\frac{a}{\lambda},\quad\forall a\in(0,1),
    \end{align*}
    ensures that
    \begin{align*}
        \norm{\x_t}_\infty\le\frac{a+1}{2\lambda},\quad\norm{\x_{t+1}-\x_t}_\infty\le\frac{a+3}{2}\eta,\quad\forall t\in[T].
    \end{align*}
\end{lem}

\begin{proof}
    We first show that the trajectory is bounded if the model is appropriately initialized with $\norm{\x_1}_\infty\le a/\lambda$. Denote $1-\eta\lambda$ by $q$, suggesting $q\in(1/2^{1/T},1)$ by the choice of $\lambda$. By the update rule of~\cref{alg:lion}, we have
    \begin{align*}
        &\norm{\x_t}_\infty=\norm{(1-\eta\lambda)\x_{t-1} - \eta    \sign{\v_{t-1}}}_\infty\le q\norm{\x_{t-1}}_\infty+\eta\\\le&q^2\norm{\x_{t-2}}_\infty+q\eta+\eta\le\cdots\le q^{t-1}\norm{\x_1}_\infty+\eta\sum_{i=0}^{t-2}q^i\le \frac{aq^{t-1}}{\lambda}+\frac{\eta(1-q^{t-1})}{1-q}\\=& \frac{1-(1-a)q^{t-1}}{\lambda}\le\frac{1-(1-a)\brac{\frac{1}{2}}^{\frac{t-1}{T}}}{\lambda}\le\frac{a+1}{2\lambda},
    \end{align*}
    which holds for any $t\in[T]$. Therefore, we can deduce
    \begin{align*}
        \norm{\x_{t+1}-\x_t}_\infty=\norm{\eta\lambda\x_t+\eta\sign{\v_t}}_\infty\le \eta\lambda\norm{\x_t}_\infty+\eta\le\frac{a+3}{2}\eta,\quad\forall t\in[T].
    \end{align*}
\end{proof}

\begin{lem}[\cref{lem:smooth-grad-sign} with weight decay]\label{lem:smooth-grad-sign-wd}
    For any sign vector $\s\in\cbrac{-1,1}^d$ and any $\eta\le1/((c+1)\norm{\lb_1}_\infty),c>0$, consider the sign-based update with weight decay:
    \begin{align*}
        \x^\prime:=\x+\eta\s^\prime,\quad\s^\prime=-\s-\lambda\x,\quad\forall\lambda\ge0,\x\in\R^d\text{ such that }\lambda\norm{\x}_\infty\le c. 
    \end{align*}
    Under~\cref{ass:generalized-smooth}, it holds that
    \begin{align*}
        \norm{\nabla f(\x^\prime)-\nabla f(\x)}_1\le(2+2c)\eta\brac{\norm{\lb_0}_1+\inner{\lb_1}{\abs{\nabla f(\x)}}}.
    \end{align*}
\end{lem}

\begin{proof}
    The constraint on $\x$ implies $\lambda\norm{\x}_\infty\le c$, where the case
    $\lambda=0$ is immediate. Hence
    \[
        \norm{\s^\prime}_\infty
        =\norm{-\s-\lambda\x}_\infty
        \le \norm{\s}_\infty+\lambda\norm{\x}_\infty
        \le 1+c .
    \]
    If $\eta=0$ or $\s^\prime=\0$, the claim is immediate. Thus, assume
    $\eta>0$ and $\s^\prime\ne\0$.

    Denote
    \[
        \z_t:=\x+t\eta\s^\prime,\qquad
        \q(t):=\nabla f(\z_t),\qquad
        Q(t):=\norm{\q(t)-\q(0)}_1,
    \]
    and define
    \[
        \DB_t:=\diag{\lb_0+\lb_1\odot\abs{\q(t)}},\qquad t\in[0,1].
    \]
    Since
    \[
        \norm{\z_t-\x}_\infty
        \le t\eta\norm{\s^\prime}_\infty
        \le \eta(1+c)
        \le \frac{1}{\norm{\lb_1}_\infty},
    \]
    whenever $\norm{\lb_1}_\infty>0$, \cref{ass:generalized-smooth} gives, at every point on the segment where
    $f$ is twice-differentiable,
    \[
        -\DB_t\preceq \nabla^2 f(\z_t)\preceq \DB_t .
    \]
    The same display is also the corresponding bound when $\norm{\lb_1}_\infty=0$.

    Since $\q'(t)=\eta\nabla^2 f(\z_t)\s^\prime$ for almost every $t\in[0,1]$, we bound
    $\norm{\nabla^2 f(\z_t)\s^\prime}_1$ as follows. Because
    $\s^\prime/\norm{\s^\prime}_\infty\in[-1,1]^d$, and because a linear function over
    $[-1,1]^d$ is maximized at a sign vector,
    \begin{align*}
        \norm{\nabla^2 f(\z_t)\s^\prime}_1
        &= \norm{\s^\prime}_\infty
        \max_{\u\in\cbrac{-1,1}^d}
        \inner{\u}{\nabla^2 f(\z_t)\frac{\s^\prime}{\norm{\s^\prime}_\infty}} \\
        &\le \norm{\s^\prime}_\infty
        \max_{\u,\tilde{\s}\in\cbrac{-1,1}^d}
        \inner{\u}{\nabla^2 f(\z_t)\tilde{\s}} .
    \end{align*}
    We next show that, for any $\u,\tilde{\s}\in\cbrac{-1,1}^d$,
    \[
        \inner{\u}{\nabla^2 f(\z_t)\tilde{\s}}\le \tr{\DB_t}.
    \]
    Indeed, set
    \[
        \a:=\frac{\u+\tilde{\s}}{2},\qquad
        \b:=\frac{\u-\tilde{\s}}{2}.
    \]
    Then $\u=\a+\b$, $\tilde{\s}=\a-\b$, and $\a_i^2+\b_i^2=1$ for every coordinate $i$.
    By symmetry of the Hessian,
    \begin{align*}
        \inner{\u}{\nabla^2 f(\z_t)\tilde{\s}}
        &= \inner{\a+\b}{\nabla^2 f(\z_t)(\a-\b)} \\
        &= \inner{\a}{\nabla^2 f(\z_t)\a}
           -\inner{\b}{\nabla^2 f(\z_t)\b} \\
        &\le \inner{\a}{\DB_t\a}
           +\inner{\b}{\DB_t\b} \\
        &= \sum_{i=1}^d(\DB_t)_{ii}\brac{\a_i^2+\b_i^2}
         = \tr{\DB_t}.
    \end{align*}
    Therefore, for almost every $t\in[0,1]$,
    \[
        \norm{\q'(t)}_1
        \le \eta\norm{\s^\prime}_\infty\tr{\DB_t}
        \le \eta(1+c)\tr{\DB_t}.
    \]
    Moreover,
    \begin{align*}
        \tr{\DB_t}
        &= \norm{\lb_0}_1+\inner{\lb_1}{\abs{\q(t)}} \\
        &\le \norm{\lb_0}_1+\inner{\lb_1}{\abs{\q(0)}}
            +\norm{\lb_1}_\infty\norm{\q(t)-\q(0)}_1 \\
        &= \norm{\lb_0}_1+\inner{\lb_1}{\abs{\nabla f(\x)}}
            +\norm{\lb_1}_\infty Q(t),
    \end{align*}
    where we used $\abs{\abs{a}-\abs{b}}\le\abs{a-b}$ coordinatewise. Hence, for every $t\in[0,1]$,
    \begin{align*}
        Q(t)
        &\le \int_0^t\norm{\q'(\tau)}_1\diff\tau \\
        &\le \eta(1+c)t
        \brac{\norm{\lb_0}_1+\inner{\lb_1}{\abs{\nabla f(\x)}}}
        +\eta(1+c)\norm{\lb_1}_\infty
        \int_0^t Q(\tau)\diff\tau .
    \end{align*}
    Applying Gr\"onwall's inequality yields
    \begin{align*}
        Q(t)
        \le
        \eta(1+c)
        \brac{\norm{\lb_0}_1+\inner{\lb_1}{\abs{\nabla f(\x)}}}
        \int_0^t
        e^{\eta(1+c)\norm{\lb_1}_\infty(t-\tau)}
        \diff\tau .
    \end{align*}
    If $\norm{\lb_1}_\infty>0$, then taking $t=1$ gives
    \begin{align*}
        Q(1)
        &\le
        \brac{\norm{\lb_0}_1+\inner{\lb_1}{\abs{\nabla f(\x)}}}
        \frac{e^{\eta(1+c)\norm{\lb_1}_\infty}-1}{\norm{\lb_1}_\infty} \\
        &\le
        \eta(1+c)(e-1)
        \brac{\norm{\lb_0}_1+\inner{\lb_1}{\abs{\nabla f(\x)}}} \\
        &\le
        (2+2c)\eta
        \brac{\norm{\lb_0}_1+\inner{\lb_1}{\abs{\nabla f(\x)}}},
    \end{align*}
    where we used $\eta(1+c)\norm{\lb_1}_\infty\le1$ and $e-1\le2$.
    If $\norm{\lb_1}_\infty=0$, the same conclusion follows directly from the preceding integral inequality with the last term equal to zero. Since
    \[
        Q(1)=\norm{\nabla f(\x^\prime)-\nabla f(\x)}_1,
    \]
    the proof is complete.
\end{proof}

\subsection{Proof of Theorem~\ref{thm:signsgd}\label{sec:proof-signsgd}}

For any $t\in[T]$, we have $\norm{\x_{t+1}-\x_t}_{\infty}=\eta\norm{\sign{\m_t}}_{\infty}\le1/\norm{\lb_1}_\infty$ by the choice of $\eta\le1/\norm{\lb_1}_\infty$. Thus, under~\cref{ass:generalized-smooth}, it holds that
\begin{align*}
        f(\x_{t+1})-f(\x_t)\le& 
        \inner{ \nabla f(\x_t)}{\x_{t+1} - \x_t} + \frac{1}{2}\norm{\x_{t+1}-\x_t}_{\lb_0+\lb_1\odot\abs{\nabla f(\x_t)}}^2  \\ =&\inner{\nabla f(\x_t)}{-\eta\sign{\m_t}}+\frac{1}{2}\sum_{i=1}^d\brac{\lb_0+\lb_1\abs{\nabla_if(\x_t)}}(\eta\sign{\m_{t,i}})^2\\=&\eta\inner{\nabla f(\x_t)}{-\sign{\nabla f(\x_t)}}+\eta\inner{\nabla f(\x_t)}{\sign{\nabla f(\x_t)}-\sign{\m_t}}\\&+\frac{\eta^2\norm{\lb_0}_1}{2}+\frac{\eta^2\inner{\lb_1}{\abs{\nabla f(\x_t)}}}{2}\\\le&-\eta\norm{\nabla f(\x_t)}_1+2\eta\norm{\m_t-\nabla f(\x_t)}_1+\frac{\eta^2\norm{\lb_0}_1}{2}+\frac{\eta^2\norm{\lb_1}_\infty\norm{\nabla f(\x_t)}_1}{2}\\\le&-\frac{\eta}{2}\norm{\nabla f(\x_t)}_1+2\eta\norm{\m_t-\nabla f(\x_t)}_1+\frac{\eta^2\norm{\lb_0}_1}{2}
        ,
\end{align*}
where the second inequality is due to~\cref{lem:sign-difference}; the last step utilizes $\eta\le1/\norm{\lb_1}_\infty$. Rearranging the above relation and summing up yields
\begin{align}\label{eq:Linf-mid}
        \E\sqbrac{\frac{1}{T}\sum_{t=1}^{T} \norm{\nabla f(\x_t)}_1} 
        \le \frac{2\Delta_f}{\eta T} +  4 \E\sqbrac{\frac{1}{T}\sum_{t=1}^{T} \norm{ \m_t-\nabla f(\x_t)}_1}+\eta\norm{\lb_0}_1,
\end{align}
where we use  $\Delta_{f}=f(\x_{1})-f_{*}\ge f(\x_1)-f(\x_{T+1})$. Next, we proceed to bound the deviation between the momentum and the true gradient~\citep{cutkosky2020momentum}. By~\eqref{eq:momentum-def}, we have
\begin{equation}\label{eq:eps-expand}
    \begin{aligned}
        &\eps_t=\m_t-\nabla f(\x_t)=\beta\m_{t-1}+(1-\beta)\g_t-\nabla f(\x_t)\\=&\beta\brac{\m_{t-1}-\nabla f(\x_{t-1})}+\beta\brac{\nabla f(\x_{t-1})-\nabla f(\x_{t})}+(1-\beta)\brac{\g_t-\nabla f(\x_t)}\\=&\beta\eps_{t-1}+\beta\s_t+(1-\beta)\n_t.
    \end{aligned}
\end{equation}
Applying the above relation recursively yields
    \begin{align*}
        \eps_t=\beta^{t-1}\n_1+(1-\beta)\sum_{k=2}^t\beta^{t-k}\n_k+\sum_{k=2}^t\beta^{t-k+1}\s_k,
    \end{align*}
    where we utilize $\eps_1=\m_1-\nabla f(\x_1)=\g_1-\nabla f(\x_1)=\n_1$. Next, we decompose $\eps_t$ into 
    \begin{align}\label{eq:ABCt}
    \E\sqbrac{\norm{\eps_t}_1}\le \underbrace{\E\sqbrac{\norm{\beta^{t-1}\n_1}_1}}_{\mathtt{A_t}}+\underbrace{\E\sqbrac{\norm{(1-\beta)\sum_{k=2}^t\beta^{t-k}\n_k}_1}}_{\mathtt{B_t}}+\underbrace{\E\sqbrac{\norm{\sum_{k=2}^t\beta^{t-k+1}\s_k}_1}}_{\mathtt{C_t}},
\end{align}
and bound these terms separately.

\paragraph{Initial noise $\mathtt{A_t}$}
We have
\begin{equation}\label{eq:signsgd-At}
    \begin{aligned}
        \mathtt{A_t}=&\beta^{t-1}\sum_{i=1}^d\E\sqbrac{\abs{\n_{1,i}}}\overset{\textnormal{\cref{lem:lp-mean}}}{\le}\beta^{t-1}\sum_{i=1}^d\brac{\E\sqbrac{\abs{\n_{1,i}}^p}}^{1/p}\\\overset{\textnormal{\cref{lem:batch-noise}}}{\le}&2\beta^{t-1}\sum_{i=1}^d\brac{B^{1-p}\brac{\bsigma_{0,i}^p+\bsigma_{1,i}^p\abs{\nabla_if(\x_1)}^p}}^{1/p}\\\overset{\textnormal{\cref{lem:minkowski}}}{\le}&2\beta^{t-1}\sum_{i=1}^dB^{\frac{1-p}{p}}\brac{\bsigma_{0,i}+\bsigma_{1,i}\abs{\nabla_if(\x_1)}}\\\le&2\beta^{t-1}B^{\frac{1-p}{p}}\brac{\norm{\bsigma_0}_1+\norm{\bsigma_1}_\infty\norm{\nabla f(\x_1)}_1}.
    \end{aligned}
\end{equation}

\paragraph{Cumulative noise $\mathtt{B_t}$}
Define $\v_k:=\beta^{t-k}\n_k,k\in[2,t]$. It holds that
\begin{align}\label{eq:Bt-initial-bound}
    \frac{\mathtt{B_t}}{1-\beta}\le\E\sqbrac{\norm{\sum_{k=2}^t\v_k}_1}\overset{\textnormal{\cref{lem:MDS-L1-concentration}}}{\le}2\sqrt{2}\sum_{i=1}^d\E\sqbrac{\sum_{k=2}^t\abs{\v_{k,i}}^p}^{1/p}.
\end{align}
Here, we cannot directly bound the above noise terms using similar procedures as in $\mathtt{A_t}$ due to technical reasons. As an alternative, we recursively expand the conditional expectation.
\begin{equation}\label{eq:recursion}
    \begin{aligned}
        &\E\sqbrac{\left.\brac{\sum_{k=2}^t\abs{\v_{k,i}}^p}^{1/p}\right|\F_{t-1}}\overset{\textnormal{\cref{lem:lp-mean}}}{\le}\brac{\sum_{k=2}^t\E\sqbrac{\left.\abs{\v_{t,i}}^p\right|\F_{t-1}}}^{1/p}\\=&\brac{\sum_{k=2}^{t-1}\abs{\v_{k,i}}^p+\E\sqbrac{\left.\abs{\v_{t,i}}^p\right|\F_{t-1}}}^{1/p}\\\overset{\textnormal{\cref{lem:batch-noise}}}{\le}&\brac{\sum_{k=2}^{t-1}\abs{\v_{k,i}}^p+2\beta^{p(t-t)}B^{1-p}\brac{\bsigma_{0,i}^p+\bsigma_{1,i}^p\abs{\nabla_if(\x_t)}^p}}^{1/p}\\\overset{\textnormal{\cref{lem:minkowski}}}{\le}&\brac{\sum_{k=2}^{t-1}\abs{\v_{k,i}}^p+2\beta^{p(t-t)}B^{1-p}\bsigma_{0,i}^p}^{1/p}+2^{\frac{1}{p}}\beta^{t-t}B^{\frac{1-p}{p}}\bsigma_{1,i}\abs{\nabla_if(\x_t)}.
    \end{aligned}
\end{equation}
Taking total expectations on both sides and applying the relation in~\eqref{eq:recursion} recursively:
\begin{align*}
    &\E\sqbrac{\sum_{k=2}^t\abs{\v_{k,i}}^p}^{1/p}\\\le&\E\sqbrac{\sum_{k=2}^{t-1}\abs{\v_{k,i}}^p+2\beta^{p(t-t)}B^{1-p}\bsigma_{0,i}^p}^{1/p}+2^{\frac{1}{p}}\beta^{t-t}B^{\frac{1-p}{p}}\bsigma_{1,i}\E\sqbrac{\abs{\nabla_if(\x_t)}}\\\overset{\textnormal{\cref{lem:lp-mean}}}{\le}&\brac{\E\sqbrac{\sum_{k=2}^{t-1}\abs{\v_{k,i}}^p+2\beta^{p(t-t)}B^{1-p}\bsigma_{0,i}^p}}^{1/p}+2^{\frac{1}{p}}\beta^{t-t}B^{\frac{1-p}{p}}\bsigma_{1,i}\E\sqbrac{\abs{\nabla_if(\x_t)}}\\=&\brac{\E\sqbrac{\E\sqbrac{\left.\sum_{k=2}^{t-1}\abs{\v_{k,i}}^p\right|\F_{t-2}}}+2\beta^{p(t-t)}B^{1-p}\bsigma_{0,i}^p}^{1/p}+2^{\frac{1}{p}}\beta^{t-t}B^{\frac{1-p}{p}}\bsigma_{1,i}\E\sqbrac{\abs{\nabla_if(\x_t)}}\\\overset{\eqref{eq:recursion}}{\le}&\brac{\E\sqbrac{\E\sqbrac{\left.\sum_{k=2}^{t-2}\abs{\v_{k,i}}^p\right|\F_{t-3}}}+2\beta^{p(t-(t-1))}B^{1-p}\bsigma_{0,i}^p+2\beta^{p(t-t)}B^{1-p}\bsigma_{0,i}^p}^{1/p}\\&+2^{\frac{1}{p}}\beta^{t-(t-1)}B^{\frac{1-p}{p}}\bsigma_{1,i}\E\sqbrac{\abs{\nabla_if(\x_{t-1})}}+2^{\frac{1}{p}}\beta^{t-t}B^{\frac{1-p}{p}}\bsigma_{1,i}\E\sqbrac{\abs{\nabla_if(\x_t)}}\\\le&\cdots\\\le&\brac{\sum_{k=2}^t2\beta^{p(t-k)}B^{1-p}\bsigma_{0,i}^p}^{1/p}+\sum_{k=2}^t2^{\frac{1}{p}}\beta^{t-k}B^{\frac{1-p}{p}}\bsigma_{1,i}\E\sqbrac{\abs{\nabla_if(\x_k)}}\\\le&\frac{2^{\frac{1}{p}}B^{\frac{1-p}{p}}\bsigma_{0,i}}{\brac{1-\beta^p}^{1/p}}+2^{\frac{1}{p}}B^{\frac{1-p}{p}}\sum_{k=2}^t\beta^{t-k}\bsigma_{1,i}\E\sqbrac{\abs{\nabla_if(\x_k)}}.
\end{align*}
Thus, combining with~\eqref{eq:Bt-initial-bound}, we conclude that
\begin{equation}\label{eq:signsgd-Bt}
    \begin{aligned}
    \mathtt{B_t}\le& 4\sqrt{2}(1-\beta)B^{\frac{1-p}{p}}\sum_{i=1}^d\brac{\frac{\bsigma_{0,i}}{\brac{1-\beta^p}^{1/p}}+\sum_{k=2}^t\beta^{t-k}\bsigma_{1,i}\E\sqbrac{\abs{\nabla_if(\x_k)}}}\\\le&4\sqrt{2}(1-\beta)B^{\frac{1-p}{p}}\brac{\frac{\norm{\bsigma_0}_1}{\brac{1-\beta}^{1/p}}+\sum_{k=2}^t\beta^{t-k}\norm{\bsigma_1}_\infty\E\sqbrac{\norm{\nabla f(\x_k)}_1}}\\\le&4\sqrt{2}B^{\frac{1-p}{p}}(1-\beta)^{\frac{p-1}{p}}\norm{\bsigma_0}_1+4\sqrt{2}B^{\frac{1-p}{p}}(1-\beta)\norm{\bsigma_1}_\infty\sum_{k=2}^t\beta^{t-k}\E\sqbrac{\norm{\nabla f(\x_k)}_1}.
    \end{aligned}
\end{equation}

\paragraph{Trajectory curvature $\mathtt{C_t}$}
Since $\cbrac{\x_t}_{t\in[T]}$ are generated by sign-based updates, we can apply~\cref{lem:smooth-grad-sign} to bound $\norm{\s_k}_1$:
\begin{equation}\label{eq:signsgd-Ct}
    \begin{aligned}
        \mathtt{C_t}\le&\E\sqbrac{\sum_{k=2}^t\beta^{t-k+1}\norm{\s_k}_1}\le\sum_{k=2}^t\beta^{t-k+1}\cdot \E\sqbrac{2\eta\brac{\norm{\lb_0}_1+\inner{\lb_1}{\abs{\nabla f(\x_k)}}}}\\\le&\frac{2\eta\beta\norm{\lb_0}_1}{1-\beta}+\sum_{k=2}^t2\eta\beta^{t-k+1}\norm{\lb_1}_\infty\E\sqbrac{\norm{\nabla f(\x_k)}_1}.
    \end{aligned}
\end{equation}
Combining the bounds for $\mathtt{A_t},\mathtt{B_t},\mathtt{C_t}$ (adding~\eqref{eq:signsgd-At},~\eqref{eq:signsgd-Bt} and~\eqref{eq:signsgd-Ct} together), we obtain
\begin{align*}
    \E\sqbrac{\norm{\eps_t}_1}\le&\beta^{t-1}B^{\frac{1-p}{p}}\brac{\norm{\bsigma_0}_1+\norm{\bsigma_1}_\infty\norm{\nabla f(\x_1)}_1}\\&+4\sqrt{2}B^{\frac{1-p}{p}}(1-\beta)^{\frac{p-1}{p}}\norm{\bsigma_0}_1+4\sqrt{2}B^{\frac{1-p}{p}}(1-\beta)\norm{\bsigma_1}_\infty\sum_{k=2}^t\beta^{t-k}\E\sqbrac{\norm{\nabla f(\x_k)}_1}\\&+\frac{2\eta\beta\norm{\lb_0}_1}{1-\beta}+\sum_{k=2}^t2\eta\beta^{t-k+1}\norm{\lb_1}_\infty\E\sqbrac{\norm{\nabla f(\x_k)}_1}.
\end{align*}
Summing from $1$ to $T$:
\begin{align*}
    \E\sqbrac{\frac{1}{T}\sum_{t=1}^T\norm{\eps_t}_1}\le&\frac{B^{\frac{1-p}{p}}}{T(1-\beta)}\brac{\norm{\bsigma_0}_1+\norm{\bsigma_1}_\infty\norm{\nabla f(\x_1)}_1}+4\sqrt{2}B^{\frac{1-p}{p}}(1-\beta)^{\frac{p-1}{p}}\norm{\bsigma_0}_1\\&+4\sqrt{2}B^{\frac{1-p}{p}}(1-\beta)\norm{\bsigma_1}_\infty\sum_{t=1}^T\sum_{k=2}^t\frac{\beta^{t-k}}{T}\E\sqbrac{\norm{\nabla f(\x_k)}_1}\\&+\frac{2\eta\beta\norm{\lb_0}_1}{1-\beta}+2\eta\norm{\lb_1}_\infty\sum_{t=1}^T\sum_{k=2}^t\frac{\beta^{t-k+1}}{T}\E\sqbrac{\norm{\nabla f(\x_k)}_1}.
\end{align*}
Note that
\begin{equation}\label{eq:switch-sum}
    \sum_{t=1}^T\sum_{k=2}^t\beta^{t-k}\E\sqbrac{\norm{\nabla f(\x_k)}_1}=\sum_{t=2}^T\brac{\sum_{k=0}^{T-t}\beta^{k}}\E\sqbrac{\norm{\nabla f(\x_t)}_1}\le\sum_{t=2}^T\frac{\E\sqbrac{\norm{\nabla f(\x_t)}_1}}{(1-\beta)}.
\end{equation}
Hence,
\begin{equation}
    \begin{aligned}
        &\E\sqbrac{\frac{1}{T}\sum_{t=1}^T\norm{\eps_t}_1}\\\le&\frac{B^{\frac{1-p}{p}}}{T(1-\beta)}\brac{\norm{\bsigma_0}_1+\norm{\bsigma_1}_\infty\norm{\nabla f(\x_1)}_1}+4\sqrt{2}B^{\frac{1-p}{p}}(1-\beta)^{\frac{p-1}{p}}\norm{\bsigma_0}_1+\frac{2\eta\beta\norm{\lb_0}_1}{1-\beta}\\&+4\sqrt{2}B^{\frac{1-p}{p}}\norm{\bsigma_1}_\infty\sum_{t=2}^T\frac{1}{T}\E\sqbrac{\norm{\nabla f(\x_k)}_1}+2\eta\norm{\lb_1}_\infty\sum_{t=2}^T\frac{\beta}{T(1-\beta)}\E\sqbrac{\norm{\nabla f(\x_t)}_1}\\=&4\sqrt{2}(1-\beta)^{\frac{p-1}{p}}B^{\frac{1-p}{p}}\norm{\bsigma_0}_1+\frac{2B^{\frac{1-p}{p}}\norm{\bsigma_0}_1+2B^{\frac{1-p}{p}}\norm{\bsigma_1}_\infty\norm{\nabla f(\x_1)}_1}{T(1-\beta)}+\frac{2\eta\beta\norm{\lb_0}_1}{1-\beta}\\&+\sum_{t=2}^T\brac{4\sqrt{2}B^{\frac{1-p}{p}}\norm{\bsigma_1}_\infty+\frac{2\eta\norm{\lb_1}_\infty\beta}{1-\beta}}\frac{1}{T}\E\sqbrac{\norm{\nabla f(\x_t)}_1}.
    \end{aligned}\label{eq:eps-mid-bound}
\end{equation}
By~\eqref{eq:signsgd-params}, we deduce that
\begin{align*}
    4\sqrt{2}B^{\frac{1-p}{p}}\norm{\bsigma_1}_\infty+\frac{2\eta\norm{\lb_1}_\infty\beta}{1-\beta}\le \frac{1}{16}+\frac{1}{16}=\frac{1}{8}.
\end{align*}
Plugging the above relation into the last line of~\eqref{eq:eps-mid-bound} and recall~\eqref{eq:Linf-mid}, we get
\begin{align*}
    &\E\sqbrac{\frac{1}{T}\sum_{t=1}^{T} \norm{\nabla f(\x_t)}_1} 
        \le \frac{2\Delta_f}{\eta T}+\eta\norm{\lb_0}_1 +\E\sqbrac{\frac{1}{2T}\sum_{t=1}^{T} \E\sqbrac{\norm{ \nabla f(\x_t)}_1}}\\&+16\sqrt{2}(1-\beta)^{\frac{p-1}{p}}B^{\frac{1-p}{p}}\norm{\bsigma_0}_1+\frac{8B^{\frac{1-p}{p}}\brac{\norm{\bsigma_0}_1+\norm{\bsigma_1}_\infty\norm{\nabla f(\x_1)}_1}}{T(1-\beta)}+\frac{8\eta\beta\norm{\lb_0}_1}{1-\beta},
\end{align*}
Rearranging the above relation yields
\begin{align*}
    \E\sqbrac{\frac{1}{T}\sum_{t=1}^{T} \norm{\nabla f(\x_t)}_1} 
        \le &\frac{4\Delta_f}{\eta T}+\frac{18\eta\beta\norm{\lb_0}_1}{1-\beta}+32\sqrt{2}(1-\beta)^{\frac{p-1}{p}}B^{\frac{1-p}{p}}\norm{\bsigma_0}_1\\&+\frac{16B^{\frac{1-p}{p}}\brac{\norm{\bsigma_0}_1+\norm{\bsigma_1}_\infty\norm{\nabla f(\x_1)}_1}}{T(1-\beta)}.
\end{align*}
In view of $\eta$ in~\eqref{eq:signsgd-params}, it holds that
\begin{equation}\label{eq:eta-bound}
    \begin{aligned}
        \frac{4\Delta_f}{\eta T}+\frac{18\eta\beta\norm{\lb_0}_1}{1-\beta}\le&12\sqrt{\frac{2\Delta_f\norm{\lb_0}_1}{T(1-\beta)}}+\frac{128\beta\Delta_f\norm{\lb_1}_\infty}{T(1-\beta)}+\frac{9\norm{\lb_0}_1}{16\norm{\lb_1}_\infty}\\\le&12\sqrt{\frac{2\Delta_f\norm{\lb_0}_1}{T(1-\beta)}}+\frac{256\beta\Delta_f\norm{\lb_1}_\infty}{T(1-\beta)}\\\le&\frac{12\sqrt{2}(\Delta_f\norm{\lb_0}_1)^{\frac{p-1}{3p-2}}\norm{\bsigma_0}_1^{\frac{p}{3p-2}}}{(BT)^{\frac{p-1}{3p-2}}}+\frac{256\Delta_f^{\frac{2p-2}{3p-2}}\norm{\bsigma_0}_1^{\frac{2p}{3p-2}}\norm{\lb_1}_\infty}{\norm{\lb_0}_1^{{\frac{p}{3p-2}}}(BT)^{\frac{2p-2}{3p-2}}},
    \end{aligned}
\end{equation}
where the second inequality holds when $\sqrt{\frac{2\Delta_f(1-\beta)}{9\beta\norm{\lb_0}_1T}}\ge\frac{1-\beta}{32\beta\norm{\lb_1}_\infty}$. Similarly, the rest terms are bounded by
\begin{equation}\label{eq:beta-bound}
    \begin{aligned}
        &32\sqrt{2}(1-\beta)^{\frac{p-1}{p}}B^{\frac{1-p}{p}}\norm{\bsigma_0}_1+\frac{16B^{\frac{1-p}{p}}\brac{\norm{\bsigma_0}_1+\norm{\bsigma_1}_\infty\norm{\nabla f(\x_1)}_1}}{T(1-\beta)}\\\le&\frac{32\sqrt{2}(\Delta_f\norm{\lb_0}_1)^{\frac{p-1}{3p-2}}\norm{\bsigma_0}_1^{\frac{p}{3p-2}}}{(BT)^{\frac{p-1}{3p-2}}}+\frac{16B^{\frac{1-p}{p}}\brac{\norm{\bsigma_0}_1+\norm{\bsigma_1}_\infty\norm{\nabla f(\x_1)}_1}\norm{\bsigma_0}_1^{\frac{2p}{3p-2}}}{\brac{\Delta_f\norm{\lb_0}_1}^{{\frac{p}{3p-2}}}(BT)^{\frac{2p-2}{3p-2}}}
    \end{aligned}
\end{equation}
Combining~\eqref{eq:beta-bound} with~\eqref{eq:eta-bound}, we finally obtain
\begin{align*}
    &\E\sqbrac{\frac{1}{T}\sum_{t=1}^{T} \norm{\nabla f(\x_t)}_1} 
        \le\frac{44\sqrt{2}(\Delta_f\norm{\lb_0}_1)^{\frac{p-1}{3p-2}}\norm{\bsigma_0}_1^{\frac{p}{3p-2}}}{(BT)^{\frac{p-1}{3p-2}}}\\&+\frac{16B^{\frac{1-p}{p}}\brac{\norm{\bsigma_0}_1+\norm{\bsigma_1}_\infty\norm{\nabla f(\x_1)}_1}\norm{\bsigma_0}_1^{\frac{2p}{3p-2}}}{\brac{\Delta_f\norm{\lb_0}_1}^{{\frac{p}{3p-2}}}(BT)^{\frac{2p-2}{3p-2}}}+\frac{256\Delta_f^{\frac{2p-2}{3p-2}}\norm{\bsigma_0}_1^{\frac{2p}{3p-2}}\norm{\lb_1}_\infty}{\norm{\lb_0}_1^{{\frac{p}{3p-2}}}(BT)^{\frac{2p-2}{3p-2}}},
\end{align*}
which implies a convergence rate of $O\brac{(\Delta_f\norm{\lb_0}_1)^{\frac{p-1}{3p-2}}\norm{\bsigma_0}_1^{\frac{p}{3p-2}}(BT)^{\frac{1-p}{3p-2}}}$.

\subsection{Proof of Theorem~\ref{thm:lion}}

By~\cref{lem:stability-wd}, $\norm{\x_{t+1}-\x_t}_{\infty}\le5\eta/3$ holds for any $t\in[T]$ provided that $\norm{\x_1}_\infty\le1/(3\lambda)$, which immediately implies $\norm{\x_{t+1}-\x_t}_{\infty}\le1/\norm{\lb_1}_\infty$ by the choice of $\eta$ in~\eqref{eq:lion-params}. Thus, under~\cref{ass:generalized-smooth}, it holds that
\begin{align*}
        f(\x_{t+1})\le& f(\x_t)+
        \inner{ \nabla f(\x_t)}{\x_{t+1} - \x_t} + \frac{1}{2}\norm{\x_{t+1}-\x_t}_{\lb_0+\lb_1\odot\abs{\nabla f(\x_t)}}^2  \\ \le&f(\x_t)+\inner{\nabla f(\x_t)}{-\eta\sign{\v_t}}+\inner{\nabla f(\x_t)}{-\eta\lambda\x_t}\\&+\frac{1}{2}\norm{\lb_0+\lb_1\odot\abs{\nabla f(\x_t)}}_1\norm{\x_{t+1}-\x_t}_\infty^2\\\overset{\textnormal{\cref{lem:stability-wd}}}{\le}&f(\x_t)+\eta\inner{\nabla f(\x_t)}{-\sign{\nabla f(\x_t)}}+\eta\inner{\nabla f(\x_t)}{\sign{\nabla f(\x_t)}-\sign{\v_t}}\\&+\eta\lambda\norm{\x_t}_\infty\norm{\nabla f(\x_t)}_1+\frac{25\eta^2}{9}\brac{\norm{\lb_0}_1+\inner{\lb_1}{\abs{\nabla f(\x_t)}}}\\\overset{\textnormal{\cref{lem:sign-difference,lem:stability-wd}}}{\le}&f(\x_t)-\eta\norm{\nabla f(\x_t)}_1+2\eta\norm{\v_t-\nabla f(\x_t)}_1+\frac{2\eta}{3}\norm{\nabla f(\x_t)}_1\\&+\frac{25\eta^2\norm{\lb_0}_1}{9}+\frac{25\eta^2\norm{\lb_1}_\infty\norm{\nabla f(\x_t)}_1}{9}\\\le&f(\x_t)-\frac{\eta}{6}\norm{\nabla f(\x_t)}_1+2\eta\norm{\v_t-\nabla f(\x_t)}_1+\frac{25\eta^2\norm{\lb_0}_1}{9}
        ,
\end{align*}
where the last step utilizes $\eta\le3/(50\norm{\lb_1}_\infty)$. Similar to~\eqref{eq:Linf-mid}, we arrive at
\begin{align}\label{eq:Linf-mid-lion}
        \frac{1}{T}\sum_{t=1}^{T}\E\sqbrac{ \norm{\nabla f(\x_t)}_1} 
        \le \frac{6\Delta_f}{\eta T} +  \frac{12}{T}\sum_{t=1}^{T} \E\sqbrac{ \norm{ \v_t-\nabla f(\x_t)}_1}+\frac{50\eta\norm{\lb_0}_1}{3}.
\end{align}
Recall the notations in~\eqref{eq:momentum-def} and apply the same decomposition as in~\eqref{eq:eps-expand} to bound $\E\sqbrac{\norm{ \v_t-\nabla f(\x_t)}_1}$:
\begin{equation}\label{eq:lion-decomposition}
    \begin{aligned}
        &\E\sqbrac{\norm{ \v_t-\nabla f(\x_t)}_1}=\E\sqbrac{\norm{\beta_1\m_{t-1}+(1-\beta_1)\g_t-\nabla f(\x_t)}_1}\\=&\E\sqbrac{\norm{\beta_1(\m_{t-1}-\nabla f(\x_{t-1}))+(1-\beta_1)(\g_t-\nabla f(\x_{t}))+\beta_1(\nabla f(\x_{t-1})-\nabla f(\x_t))}_1}\\\le&\beta_1\E\sqbrac{\norm{\eps_{t-1}}_1}+(1-\beta_1)\E\sqbrac{\norm{\n_t}_1}+\beta_1\E\sqbrac{\norm{\s_t}_1},
    \end{aligned}
\end{equation}
which holds for all $t\ge2$. The noise term $\n_t$ can be bounded by
\begin{align*}
    \E\sqbrac{\norm{\n_t}_1}=&\sum_{i=1}^d\E\sqbrac{\E\sqbrac{\abs{\n_{t,i}}|\F_{t-1}}}\overset{\textnormal{\cref{lem:lp-mean}}}{\le}\sum_{i=1}^d\E\sqbrac{\brac{\E\sqbrac{\left.\abs{\n_{t,i}}^p\right|\F_{t-1}}}^{1/p}}\\\overset{\textnormal{\cref{lem:batch-noise}}}{\le}&\sum_{i=1}^d\E\sqbrac{\brac{2B^{1-p}\brac{\bsigma_{0,i}^p+\bsigma_{1,i}^p\abs{\nabla_if(\x_t)}^p}}^{1/p}}\\\overset{\textnormal{\cref{lem:minkowski}}}{\le}&\sum_{i=1}^d\E\sqbrac{2^{\frac{1}{p}}B^{\frac{1-p}{p}}\brac{\bsigma_{0,i}+\bsigma_{1,i}\abs{\nabla_if(\x_t)}}}\\\le& 2^{\frac{1}{p}}B^{\frac{1-p}{p}}\brac{\norm{\bsigma_0}_1+\norm{\bsigma_1}_\infty\E\sqbrac{\norm{\nabla f(\x_t)}_1}}.
\end{align*}
To bound the trajectory curvature term $\s_t$, note that the trajectory of~\cref{alg:lion} satisfies the conditions in~\cref{lem:smooth-grad-sign-wd} with $c=2/3$ by~\cref{lem:stability-wd}. Therefore, we deduce that
\begin{align*}
    \E\sqbrac{\norm{\s_t}_1}\le \E\sqbrac{\frac{10\eta}{3}\brac{\norm{\lb_0}_1+\inner{\lb_1}{\abs{\nabla f(\x)}}}}\le\frac{10\eta}{3}\norm{\lb_0}_1+\frac{10\eta}{3}\norm{\lb_1}_\infty\E\sqbrac{\norm{\nabla f(\x)}_1}.
\end{align*}
Plugging the bounds for $\E\sqbrac{\norm{\n_t}_1}$ and $\E\sqbrac{\norm{\s_t}_1}$, then summing from $2$ to $T$, we get
\begin{equation}\label{eq:lion-vt-error-bound}
    \begin{aligned}
        &\frac{1}{T}\sum_{t=1}^{T} \E\sqbrac{ \norm{ \v_t-\nabla f(\x_t)}_1}=\frac{1}{T}\E\sqbrac{ \norm{ \v_1-\nabla f(\x_1)}_1}+\frac{1}{T}\sum_{t=2}^{T} \E\sqbrac{ \norm{ \v_t-\nabla f(\x_t)}_1}\\\le&\frac{1}{T}\E\sqbrac{ \norm{ \n_1}_1}+\frac{\beta_1}{T}\sum_{t=2}^{T} \E\sqbrac{\norm{\eps_{t-1}}_1}+\frac{(1-\beta_1)}{T}\sum_{t=2}^{T}\E\sqbrac{\norm{\n_t}_1}+\frac{\beta_1}{T}\sum_{t=2}^{T}\E\sqbrac{\norm{\s_t}_1}\\\le&\frac{\beta_1}{T}\sum_{t=1}^{T-1} \E\sqbrac{\norm{\eps_t}_1}+\frac{\beta_1}{T}\E\sqbrac{ \norm{ \n_1}_1}+\frac{2(1-\beta_1)}{T}\sum_{t=1}^{T} B^{\frac{1-p}{p}}\brac{\norm{\bsigma_0}_1+\norm{\bsigma_1}_\infty\E\sqbrac{\norm{\nabla f(\x_t)}_1}}\\&+\frac{\beta_1}{T}\sum_{t=2}^{T} \brac{\frac{10\eta}{3}\norm{\lb_0}_1+\frac{10\eta}{3}\norm{\lb_1}_\infty\E\sqbrac{\norm{\nabla f(\x_t)}_1}}\\\le&\frac{\beta_1}{T}\sum_{t=1}^{T-1} \E\sqbrac{\norm{\eps_t}_1}+\frac{2\beta_1\brac{\norm{\bsigma_0}_1+\norm{\bsigma_1}_\infty\norm{\nabla f(\x_1)}_1}}{B^{\frac{p-1}{p}}T}+\frac{2(1-\beta_1)\norm{\bsigma_0}_1}{B^{\frac{p-1}{p}}}+\frac{10\eta\beta_1}{3}\norm{\lb_0}_1\\&+\frac{2(1-\beta_1)\norm{\bsigma_1}_\infty}{B^{\frac{p-1}{p}}T}\sum_{t=1}^{T}\E\sqbrac{\norm{\nabla f(\x_t)}_1}+\frac{10\beta_1\eta\norm{\lb_1}_\infty}{3T}\sum_{t=2}^{T}\E\sqbrac{\norm{\nabla f(\x_t)}_1},
    \end{aligned}
\end{equation}
where the first inequality is due to $\v_1-\nabla f(\x_1)=\g_1-\nabla f(\x_1)=\n_1$. Next, we proceed by bounding the error term $\frac{\beta_1}{T}\sum_{t=1}^{T-1} \E\sqbrac{\norm{\eps_t}_1}$. Following the same recipe as in~\eqref{eq:eps-expand}, we immediately obtain $\eps_t=(1-\beta_2)\eps_{t-1}+(1-\beta_2)\s_t+\beta_2\n_t$. Expanding this relation recursively yields
    \begin{align*}
        \eps_t=\beta_2^{t-1}\n_1+(1-\beta_2)\sum_{k=2}^t\beta_2^{t-k}\n_k+\sum_{k=2}^t\beta_2^{t-k+1}\s_k,
    \end{align*}
    where we utilize $\eps_1=\m_1-\nabla f(\x_1)=\g_1-\nabla f(\x_1)=\n_1$. Next, we decompose $\eps_t$ into 
    \begin{align*}
    \E\sqbrac{\norm{\eps_t}_1}\le \underbrace{\E\sqbrac{\norm{\beta_2^{t-1}\n_1}_1}}_{\mathtt{A_t}}+\underbrace{\E\sqbrac{\norm{(1-\beta_2)\sum_{k=2}^t\beta_2^{t-k}\n_k}_1}}_{\mathtt{B_t}}+\underbrace{\E\sqbrac{\norm{\sum_{k=2}^t\beta_2^{t-k+1}\s_k}_1}}_{\mathtt{C_t}},
    \end{align*}
    which exhibits the almost identical form to that in~\eqref{eq:ABCt}, with the only difference in $\beta$. Thus, we can safely follow the derivations in~\cref{sec:proof-signsgd} to bound $\mathtt{A_t},\mathtt{B_t},\mathtt{C_t}$. For $\mathtt{A_t}$ and $\mathtt{B_t}$, we can simply replace the $\beta$ in~\eqref{eq:signsgd-At} and~\eqref{eq:signsgd-Bt} by $\beta_2$ to obtain
    \begin{align*}
        &\mathtt{A_t}\le\beta_2^{t-1}B^{\frac{1-p}{p}}\brac{\norm{\bsigma_0}_1+\norm{\bsigma_1}_\infty\norm{\nabla f(\x_1)}_1},\\
        &\mathtt{B_t}\le4\sqrt{2}B^{\frac{1-p}{p}}(1-\beta_2)^{\frac{p-1}{p}}\norm{\bsigma_0}_1+4\sqrt{2}B^{\frac{1-p}{p}}(1-\beta_2)\norm{\bsigma_1}_\infty\sum_{k=2}^t\beta_2^{t-k}\E\sqbrac{\norm{\nabla f(\x_k)}_1}.
    \end{align*}
    For $\mathtt{C_t}$, we can not directly substitute $\beta$ in~\eqref{eq:signsgd-Ct} with $\beta_2$ due to the presence of weight decay. Instead, we invoke~\cref{lem:smooth-grad-sign-wd} to deduce
    \begin{align*}
        \mathtt{C_t}\le&\E\sqbrac{\sum_{k=2}^t\beta_2^{t-k+1}\norm{\s_k}_1}\le\sum_{k=2}^t\beta_2^{t-k+1}\cdot \E\sqbrac{\frac{10}{3}\eta\brac{\norm{\lb_0}_1+\inner{\lb_1}{\abs{\nabla f(\x_k)}}}}\\\le&\frac{10\eta\beta_2\norm{\lb_0}_1}{3(1-\beta_2)}+\sum_{k=2}^t\frac{10}{3}\eta\beta_2^{t-k+1}\norm{\lb_1}_\infty\E\sqbrac{\norm{\nabla f(\x_k)}_1}.
    \end{align*}
    Combining the bounds for $\mathtt{A_t},\mathtt{B_t},\mathtt{C_t}$ as below:
    \begin{align*}
        &\frac{1}{T}\sum_{t=1}^T\E\sqbrac{\norm{\eps_t}_1}\\\le&4\sqrt{2}(1-\beta_2)^{\frac{p-1}{p}}B^{\frac{1-p}{p}}\norm{\bsigma_0}_1+\frac{2B^{\frac{1-p}{p}}\norm{\bsigma_0}_1+2B^{\frac{1-p}{p}}\norm{\bsigma_1}_\infty\norm{\nabla f(\x_1)}_1}{T(1-\beta_2)}+\frac{10\eta\beta_2\norm{\lb_0}_1}{3(1-\beta_2)}\\&+\sum_{t=2}^T\brac{4\sqrt{2}B^{\frac{1-p}{p}}\norm{\bsigma_1}_\infty+\frac{10\eta\norm{\lb_1}_\infty\beta_2}{3(1-\beta_2)}}\frac{1}{T}\E\sqbrac{\norm{\nabla f(\x_t)}_1}.
    \end{align*}
    Plugging the above relation into~\eqref{eq:lion-vt-error-bound} and simplify:
    \begin{align*}
         &\frac{1}{T}\sum_{t=1}^{T} \E\sqbrac{ \norm{ \v_t-\nabla f(\x_t)}_1}\\\le&4\sqrt{2}\beta_1(1-\beta_2)^{\frac{p-1}{p}}B^{\frac{1-p}{p}}\norm{\bsigma_0}_1+\frac{2\beta_1\brac{\norm{\bsigma_0}_1+\norm{\bsigma_1}_\infty\norm{\nabla f(\x_1)}_1}}{B^{\frac{p-1}{p}}T(1-\beta_2)}+\frac{10\eta\beta_1\beta_2\norm{\lb_0}_1}{3(1-\beta_2)}\\&+\sum_{t=2}^T\brac{4\sqrt{2}B^{\frac{1-p}{p}}\norm{\bsigma_1}_\infty+\frac{10\eta\norm{\lb_1}_\infty\beta_2}{3(1-\beta_2)}}\frac{\beta_1}{T}\E\sqbrac{\norm{\nabla f(\x_t)}_1} \\&+\frac{2\beta_1\brac{\norm{\bsigma_0}_1+\norm{\bsigma_1}_\infty\norm{\nabla f(\x_1)}_1}}{B^{\frac{p-1}{p}}T}+\frac{2(1-\beta_1)\norm{\bsigma_0}_1}{B^{\frac{p-1}{p}}}+\frac{10\eta\beta_1}{3}\norm{\lb_0}_1\\&+\frac{2(1-\beta_1)\norm{\bsigma_1}_\infty}{B^{\frac{p-1}{p}}T}\sum_{t=1}^{T}\E\sqbrac{\norm{\nabla f(\x_t)}_1}+\frac{10\beta_1\eta\norm{\lb_1}_\infty}{3T}\sum_{t=2}^{T}\E\sqbrac{\norm{\nabla f(\x_t)}_1}\\\le&\frac{4\sqrt{2}\beta_1\norm{\bsigma_0}_1}{(1-\beta_2)^{\frac{1-p}{p}}B^{\frac{p-1}{p}}}+\frac{2(1-\beta_1)\norm{\bsigma_0}_1}{B^{\frac{p-1}{p}}}+\frac{4\beta_1\brac{\norm{\bsigma_0}_1+\norm{\bsigma_1}_\infty\norm{\nabla f(\x_1)}_1}}{B^{\frac{p-1}{p}}T(1-\beta_2)}+\frac{10\eta\beta_1\norm{\lb_0}_1}{3(1-\beta_2)}\\&+\frac{1}{T}\sum_{t=1}^{T}\brac{4\sqrt{2}B^{\frac{1-p}{p}}\norm{\bsigma_1}_\infty+\frac{10\eta\beta_1\norm{\lb_1}_\infty}{3(1-\beta_2)}}\E\sqbrac{\norm{\nabla f(\x_t)}_1}\\\le&\frac{4\sqrt{2}\norm{\bsigma_0}_1}{(1-\beta_2)^{\frac{1-p}{p}}B^{\frac{p-1}{p}}}+\frac{2(1-\beta_1)\norm{\bsigma_0}_1}{B^{\frac{p-1}{p}}}+\frac{4\brac{\norm{\bsigma_0}_1+\norm{\bsigma_1}_\infty\norm{\nabla f(\x_1)}_1}}{B^{\frac{p-1}{p}}T(1-\beta_2)}+\frac{10\eta\norm{\lb_0}_1}{3(1-\beta_2)}\\&+\sum_{t=1}^T\frac{1}{24T}\E\sqbrac{\norm{\nabla f(\x_t)}_1},
    \end{align*}
    where the last inequality uses the hyperparameters in~\eqref{eq:lion-params}, as depicted below:
    \begin{align*}
    4\sqrt{2}B^{\frac{1-p}{p}}\norm{\bsigma_1}_\infty+\frac{10\eta\beta_1\norm{\lb_1}_\infty}{3(1-\beta_2)}\le \frac{1}{36}+\frac{1}{36}=\frac{1}{18}.
    \end{align*}
    Now, it suffices to combine the above error bound for $\frac{1}{T}\sum_{t=1}^{T} \E\sqbrac{ \norm{ \v_t-\nabla f(\x_t)}_1}$ with~\eqref{eq:Linf-mid-lion}, which gives
    \begin{align*}
        &\frac{1}{T}\sum_{t=1}^{T}\E\sqbrac{ \norm{\nabla f(\x_t)}_1} 
        \le \frac{12\Delta_f}{\eta T}+\frac{170\eta\norm{\lb_0}_1}{1-\beta_2}+\frac{144\sqrt{2}\norm{\bsigma_0}_1}{(1-\beta_2)^{\frac{1-p}{p}}B^{\frac{p-1}{p}}}\\&+\frac{72(1-\beta_1)\norm{\bsigma_0}_1}{B^{\frac{p-1}{p}}}+\frac{144\brac{\norm{\bsigma_0}_1+\norm{\bsigma_1}_\infty\norm{\nabla f(\x_1)}_1}}{B^{\frac{p-1}{p}}T(1-\beta_2)}.
    \end{align*}
    Choosing $B,\beta_1,\beta_2,\eta$ according to~\eqref{eq:lion-params} yields the desired bound.

\section{Analysis for Muon and Muonlight\label{sec:matrix-sign-analysis}}

As established in~\cref{sec:discussion-ht-matrix},~\cref{ass:heavy-tailed-noise-matrix} implies~\cref{ass:heavy-tailed-noise-matrix-variant}. Accordingly, throughout this appendix, we maintain~\cref{ass:heavy-tailed-noise-matrix} as our operative assumption while utilizing the relation~\eqref{eq:heavy-tailed-noise-matrix-variant} specified in~\cref{ass:heavy-tailed-noise-matrix-variant}.

\subsection{Matrix Calculus}

\cref{lem:trace-property,lem:msign-property} are standard textbook results~\citep{bhatia1996matrix,horn2012matrix,gupta18shampoo}, which we omit the proof here for simplicity.

\begin{lem}[Properties of trace and matrix norms]\label{lem:trace-property}
For any $\XB\in\R^{m\times n},\YB\in\R^{n\times r},\Lb,\Lb^\prime\in\SBB^n$
    \begin{enumerate}
        \item $\tr{\Lb}=\norm{\Lb}_*,\tr{\mabs{\XB}}=\tr{\brac{\XB\XB^\top}^{1/2}}=\norm{\XB}_*$.
        \item $\norm{\XB\YB}_*\le\min\cbrac{\norm{\XB}_\op\norm{\YB}_*,\norm{\XB}_*\norm{\YB}_\op}$.
    \end{enumerate}
\end{lem}

\begin{lem}[Properties of the $\msign{\cdot}$ operator]\label{lem:msign-property}
For any $\XB\in\R^{m\times n},\Lb\in\SBB^n,a>0$
    \begin{enumerate}
        \item $\inner{\XB}{\msign{\XB}}=\norm{\XB}_*$.
        \item $\msign{a\XB}=\msign{\XB}$.
        \item $\norm{\msign{\XB}}_{\Lb}^2=\tr{\brac{\msign{\XB}}^\top\Lb\cdot\msign{\XB}}\le\tr{\Lb}$.
    \end{enumerate}
\end{lem}

\begin{lem}[Matrix Extension of~\cref{lem:sign-difference}]\label{lem:polar-difference}
    For any $\XB,\YB\in\R^{m\times n}$, it holds that
    \begin{align*}
        \inner{\XB}{\msign{\XB}-\msign{\YB}}\le2\norm{\XB-\YB}_*.
    \end{align*}
\end{lem}

\begin{proof}
    By~\cref{lem:msign-property}, we compute
    \begin{align*}
        \inner{\XB}{\msign{\XB} - \msign{\YB}} 
        &= \inner{\XB}{\msign{\XB}}  - \inner{\XB}{\msign{\YB}}  \\ &= \norm{\XB}_* - \norm{\YB}_* + \inner{\YB-\XB}{\msign{\YB}}   \\
        &\le \norm{\XB - \YB}_* + \norm{\YB-\XB}_*\norm{\msign{\YB}}_\op\le2\norm{\XB-\YB}_*,
    \end{align*}
    where the first inequality uses the triangle inequality as well as the duality between nuclear norm and operator norm.
\end{proof}

\begin{lem}
    [Matrix Cauchy-Schwarz Inequality, Lemma~8 in~\citet{an2025asgo}]\label{lem:matrix-cauchy-schwarz}
    For any $\XB\in\R^{m\times n}$ and $\Lb\in\SBB^m,\Lb\succ\0$, it holds that
    \begin{align*}
        \norm{\XB}_*\le\sqrt{\norm{\Lb}_*\tr{\XB^\top\Lb^{-1}\XB}}=\sqrt{\norm{\Lb}_*\norm{\XB}^2_{\Lb^{-1}}}.
    \end{align*}
\end{lem}

\subsection{Technical Lemmas}

\begin{lem}[Descent Lemma]\label{lem:descent-lemma}
    Under~\cref{ass:generalized-smooth-matrix}, if $\norm{\XB^\prime-\XB}_\op\le 1/\norm{\Lb_1}_\op$, then:
    \begin{align*}
        f(\XB^\prime)\le f(\XB)+\inner{\nabla f(\XB)}{\XB^\prime-\XB}+\frac{1}{2}\norm{\XB^\prime-\XB}_{\Lb(\XB)}^2
    \end{align*}
\end{lem}

\begin{proof}
    Define $\XB(s)=s\XB^\prime+(1-s)\XB,\forall s\in[0,1]$, which implies $\norm{\XB(s)-\XB}_\op=s\norm{\XB^\prime-\XB}_\op\le1/\norm{\Lb_1}_\op$. By Taylor expansion, we have
    \begin{align*}
        f(\XB^\prime)=& f(\XB)+\inner{\nabla f(\XB)}{\XB^\prime-\XB}+\int_0^1\inner{\nabla f(\XB(s))-\nabla f(\XB)}{\XB^\prime-\XB}\diff{s}\\\le&f(\XB)+\inner{\nabla f(\XB)}{\XB^\prime-\XB}+\int_0^1\norm{\nabla f(\XB(s))-\nabla f(\XB)}_{\brac{\Lb(\XB)}^{-1}}\norm{\XB^\prime-\XB}_{\Lb(\XB)}\diff{s}\\\overset{\textnormal{\cref{ass:generalized-smooth-matrix}}}{\le}&f(\XB)+\inner{\nabla f(\XB)}{\XB^\prime-\XB}+\int_0^1\norm{\XB(s)-\XB}_{\Lb(\XB)}\norm{\XB^\prime-\XB}_{\Lb(\XB)}\diff{s}\\=&f(\XB)+\inner{\nabla f(\XB)}{\XB^\prime-\XB}+\int_0^1s\norm{\XB^\prime-\XB}_{\Lb(\XB)}^2\diff{s}\\=&f(\XB)+\inner{\nabla f(\XB)}{\XB^\prime-\XB}+\frac{1}{2}\norm{\XB^\prime-\XB}_{\Lb(\XB)}^2,
    \end{align*}
    where the first inequality is due to Cauchy-Schwarz inequality, and the fact that $\norm{\cdot}_{\Lb(\XB)}$ and $\norm{\cdot}_{\brac{\Lb(\XB)}^{-1}}$ are primal-dual norm pairs~\citep{nesterov2018lectures,an2025asgo}. 
\end{proof}

\begin{lem}[Matrix Version of~\cref{lem:batch-noise}]\label{lem:batch-noise-matrix}
    Under~\cref{ass:unbiased-matrix,ass:heavy-tailed-noise-matrix}, the following holds for any $t\in[T]$:
    \begin{align*}
        \E\sqbrac{\left.\norm{\NB_t}_{\mabs{\V_0}^{-1}}^p\right|\F_{t-1}}
        \le
        2B^{1-p}
        \brac{
            \norm{\V_0}_*^{p/2}
            +
            \frac{\norm{\V_1}_{\op}^p\norm{\bnabla_t}_*^p}{\norm{\V_0}_*^{p/2}}
        },
    \end{align*}
    where $\NB_t$ is defined in~\eqref{eq:notations-muon}.
\end{lem}

\begin{proof}
    Denote
    \[
        \NB_t^b:=\Gb_t^b-\nabla f(\XB_t),\qquad b\in[B].
    \]
    Then
    \[
        \NB_t=\frac{1}{B}\sum_{b=1}^B\NB_t^b .
    \]
    By~\cref{ass:unbiased-matrix}, conditionally on $\F_{t-1}$, the mini-batch samples
    $\Gb_t^1,\ldots,\Gb_t^B$ are independent and unbiased estimators of $\nabla f(\XB_t)$.
    Therefore, conditionally on $\F_{t-1}$, the centered noises
    $\NB_t^1,\ldots,\NB_t^B$ are independent and satisfy
    \[
        \E\sqbrac{\left.\NB_t^b\right|\F_{t-1}}=\0,\qquad b\in[B].
    \]
    Equivalently, if we define the within-batch filtration
    \[
        \F_{t,b}:=\sigma\brac{\F_{t-1},\NB_t^1,\ldots,\NB_t^b},
        \qquad b=0,1,\ldots,B,
    \]
    with $\F_{t,0}:=\F_{t-1}$, then
    \[
        \E\sqbrac{\left.\NB_t^b\right|\F_{t,b-1}}=\0,\qquad b\in[B],
    \]
    where we used the conditional independence and conditional unbiasedness from
    \cref{ass:unbiased-matrix}. Thus $\cbrac{\NB_t^b}_{b=1}^B$ forms a martingale difference
    sequence with respect to the within-batch filtration $\cbrac{\F_{t,b}}_{b=0}^B$.

    Since $\V_0$ has full row rank by~\cref{ass:heavy-tailed-noise-matrix},
    $\mabs{\V_0}$ is positive definite, and $\mabs{\V_0}^{-1/2}$ is well-defined. Define
    \[
        \ZB_t^b:=\mabs{\V_0}^{-1/2}\NB_t^b,\qquad b\in[B].
    \]
    Then
    \[
        \norm{\ZB_t^b}_\Fn
        =
        \norm{\NB_t^b}_{\mabs{\V_0}^{-1}},
        \qquad
        \norm{\frac{1}{B}\sum_{b=1}^B\ZB_t^b}_\Fn
        =
        \norm{\NB_t}_{\mabs{\V_0}^{-1}}.
    \]
    After vectorization, let
    \[
        \z_t^b:=\operatorname{vec}(\ZB_t^b)\in\R^{mn}.
    \]
    Then $\cbrac{\z_t^b}_{b=1}^B$ is also a martingale difference sequence with respect to
    $\cbrac{\F_{t,b}}_{b=0}^B$, and
    \[
        \norm{\z_t^b}_2=\norm{\NB_t^b}_{\mabs{\V_0}^{-1}},
        \qquad
        \norm{\frac{1}{B}\sum_{b=1}^B\z_t^b}_2
        =
        \norm{\NB_t}_{\mabs{\V_0}^{-1}}.
    \]
    Applying the vectorized von Bahr--Esseen inequality~\citep[Lemma~10]{pmlr-v258-hubler25a}
    conditionally on $\F_{t-1}$ gives
    \begin{align*}
        \E\sqbrac{
            \left.
            \norm{\NB_t}_{\mabs{\V_0}^{-1}}^p
            \right|\F_{t-1}
        }
        &=
        \E\sqbrac{
            \left.
            \norm{
                \frac{1}{B}\sum_{b=1}^B\z_t^b
            }_2^p
            \right|\F_{t-1}
        } \\
        &\le
        \frac{2}{B^p}
        \sum_{b=1}^B
        \E\sqbrac{
            \left.
            \norm{\z_t^b}_2^p
            \right|\F_{t-1}
        } \\
        &=
        \frac{2}{B^p}
        \sum_{b=1}^B
        \E\sqbrac{
            \left.
            \norm{\NB_t^b}_{\mabs{\V_0}^{-1}}^p
            \right|\F_{t-1}
        } .
    \end{align*}
    On the other hand, by~\cref{ass:heavy-tailed-noise-matrix},
    \begin{align*}
        \E\sqbrac{
            \left.
            \norm{\NB_t^b}_{\mabs{\V_0}^{-1}}^p
            \right|\F_{t-1}
        }
        &\le
        \norm{\V_0}_*^{p/2}
        +
        \frac{
            \abs{\inner{\V_1}{\nabla f(\XB_t)}}^p
        }{
            \norm{\V_0}_*^{p/2}
        } \\
        &\le
        \norm{\V_0}_*^{p/2}
        +
        \frac{
            \norm{\V_1}_{\op}^p\norm{\nabla f(\XB_t)}_*^p
        }{
            \norm{\V_0}_*^{p/2}
        } \\
        &=
        \norm{\V_0}_*^{p/2}
        +
        \frac{
            \norm{\V_1}_{\op}^p\norm{\bnabla_t}_*^p
        }{
            \norm{\V_0}_*^{p/2}
        },
    \end{align*}
    where the second inequality uses the duality between the operator norm and the nuclear norm.
    Combining the preceding two displays yields
    \begin{align*}
        \E\sqbrac{
            \left.
            \norm{\NB_t}_{\mabs{\V_0}^{-1}}^p
            \right|\F_{t-1}
        }
        &\le
        \frac{2}{B^p}
        \sum_{b=1}^B
        \brac{
            \norm{\V_0}_*^{p/2}
            +
            \frac{
                \norm{\V_1}_{\op}^p\norm{\bnabla_t}_*^p
            }{
                \norm{\V_0}_*^{p/2}
            }
        } \\
        &=
        2B^{1-p}
        \brac{
            \norm{\V_0}_*^{p/2}
            +
            \frac{
                \norm{\V_1}_{\op}^p\norm{\bnabla_t}_*^p
            }{
                \norm{\V_0}_*^{p/2}
            }
        },
    \end{align*}
    which completes the proof.
\end{proof}

\begin{lem}[Restatement of \cref{lem:matrix-concentration}]\label{lem:asgo-regret}
    Let $\cbrac{\Gb_t}_{t=1}^T\subset\R^{m\times n}$ be a matrix martingale difference sequence, i.e.,
    \[
        \E\sqbrac{\Gb_t\mid\F_{t-1}}=\0,
        \qquad
        \F_t=\sigma\brac{\Gb_1,\ldots,\Gb_t}.
    \]
    Then
    \begin{align*}
        \E\sqbrac{\norm{\sum_{t=1}^T\Gb_t}_*}
        \le
        2\sqrt{2}\,
        \E\sqbrac{\norm{\brac{\sum_{t=1}^T\Gb_t\Gb_t^\top}^{1/2}}_*}.
    \end{align*}
\end{lem}

\begin{proof}
    This lemma is motivated by the regret analysis of the ASGO optimizer~\citep{an2025asgo,xie2025structured}. Consider the closed convex set
    \[
        \W:=\cbrac{\WB\in\R^{m\times n}:\norm{\WB}_\op\le1}.
    \]
    To avoid the possible singularity of the adaptive matrix, for any $\epsilon>0$, define
    \[
        \bLambda_t^{(\epsilon)}
        :=
        \brac{\sum_{s=1}^t\Gb_s\Gb_s^\top}^{1/2}
        +\epsilon\IB,
        \qquad
        \bLambda_0^{(\epsilon)}:=\epsilon\IB .
    \]
    Then $\bLambda_t^{(\epsilon)}\succ\0$ for every $t\in[T]$. Imagine there is an online algorithm given by
    \[
        \WB_1^{(\epsilon)}=\0,\qquad
        \WB_{t+1}^{(\epsilon)}
        =
        \Pi_{\W}^{\bLambda_t^{(\epsilon)}}
        \brac{
            \WB_t^{(\epsilon)}
            -\eta\brac{\bLambda_t^{(\epsilon)}}^{-1}\Gb_t
        },
        \qquad t\in[T],
    \]
    where the generalized projection operator~\citep{hazan2007logarithmic,duchi2011adaptive} is defined as
    \[
        \Pi_{\W}^{\bLambda}(\WB)
        :=
        \argmin_{\WB^\prime\in\W}
        \norm{\WB-\WB^\prime}_{\bLambda}^2 .
    \]
    Clearly, $\norm{\WB_t^{(\epsilon)}}_\op\le1$, and $\WB_t^{(\epsilon)}$ only depends on
    $\Gb_1,\ldots,\Gb_{t-1}$. For any $\WB_*\in\W$, the non-expansiveness of the generalized projection operator~\citep[Lemma~8]{hazan2007logarithmic} implies
    \begin{align*}
        \norm{\WB_{t+1}^{(\epsilon)}-\WB_*}_{\bLambda_t^{(\epsilon)}}^2
        \le&
        \norm{
            \WB_t^{(\epsilon)}
            -\eta\brac{\bLambda_t^{(\epsilon)}}^{-1}\Gb_t
            -\WB_*
        }_{\bLambda_t^{(\epsilon)}}^2 \\
        =&
        \norm{\WB_t^{(\epsilon)}-\WB_*}_{\bLambda_t^{(\epsilon)}}^2
        +\eta^2\norm{\Gb_t}_{\brac{\bLambda_t^{(\epsilon)}}^{-1}}^2
        -2\eta\inner{\Gb_t}{\WB_t^{(\epsilon)}-\WB_*}.
    \end{align*}
    Rearranging and summing from $1$ to $T$ gives
    \begin{align*}
        &\sum_{t=1}^T\inner{\Gb_t}{\WB_t^{(\epsilon)}}
        \\\le&
        \sum_{t=1}^T
        \brac{
            \frac{
                \norm{\WB_t^{(\epsilon)}-\WB_*}_{\bLambda_t^{(\epsilon)}}^2
                -
                \norm{\WB_{t+1}^{(\epsilon)}-\WB_*}_{\bLambda_t^{(\epsilon)}}^2
            }{2\eta}
            +
            \frac{\eta}{2}
            \norm{\Gb_t}_{\brac{\bLambda_t^{(\epsilon)}}^{-1}}^2
            +
            \inner{\Gb_t}{\WB_*}
        } \\
        \le&
        \sum_{t=1}^T
        \frac{
            \norm{\WB_t^{(\epsilon)}-\WB_*}_{\bLambda_t^{(\epsilon)}-\bLambda_{t-1}^{(\epsilon)}}^2
        }{2\eta}
        +
        \frac{
            \norm{\WB_1^{(\epsilon)}-\WB_*}_{\bLambda_0^{(\epsilon)}}^2
        }{2\eta} \\
        &+
        \frac{\eta}{2}
        \sum_{t=1}^T
        \norm{\Gb_t}_{\brac{\bLambda_t^{(\epsilon)}}^{-1}}^2
        +
        \sum_{t=1}^T\inner{\Gb_t}{\WB_*}.
    \end{align*}
    Compared with the unregularized proof, the only additional term is
    \[
        \frac{
            \norm{\WB_1^{(\epsilon)}-\WB_*}_{\bLambda_0^{(\epsilon)}}^2
        }{2\eta}
        =
        \frac{\epsilon\norm{\WB_*}_\Fn^2}{2\eta},
    \]
    where we used $\WB_1^{(\epsilon)}=\0$ and $\bLambda_0^{(\epsilon)}=\epsilon\IB$.

    We proceed by bounding the remaining two terms as in the original regret argument. First, since the matrix square-root map is operator monotone,
    \[
        \bLambda_t^{(\epsilon)}-\bLambda_{t-1}^{(\epsilon)}
        =
        \brac{\sum_{s=1}^t\Gb_s\Gb_s^\top}^{1/2}
        -
        \brac{\sum_{s=1}^{t-1}\Gb_s\Gb_s^\top}^{1/2}
        \succeq\0.
    \]
    Therefore,
    \begin{align*}
        \sum_{t=1}^T
        \frac{
            \norm{\WB_t^{(\epsilon)}-\WB_*}_{\bLambda_t^{(\epsilon)}-\bLambda_{t-1}^{(\epsilon)}}^2
        }{2\eta}
        =&
        \frac{1}{2\eta}
        \sum_{t=1}^T
        \inner{
            \brac{\WB_t^{(\epsilon)}-\WB_*}\brac{\WB_t^{(\epsilon)}-\WB_*}^\top
        }{
            \bLambda_t^{(\epsilon)}-\bLambda_{t-1}^{(\epsilon)}
        } \\
        \le&
        \frac{1}{2\eta}
        \sum_{t=1}^T
        \norm{\brac{\WB_t^{(\epsilon)}-\WB_*}\brac{\WB_t^{(\epsilon)}-\WB_*}^\top}_\op
        \norm{\bLambda_t^{(\epsilon)}-\bLambda_{t-1}^{(\epsilon)}}_* \\
        \le&
        \frac{1}{2\eta}
        \sum_{t=1}^T
        \norm{\WB_t^{(\epsilon)}-\WB_*}_\op^2
        \tr{\bLambda_t^{(\epsilon)}-\bLambda_{t-1}^{(\epsilon)}} \\
        \le&
        \frac{2}{\eta}
        \sum_{t=1}^T
        \tr{\bLambda_t^{(\epsilon)}-\bLambda_{t-1}^{(\epsilon)}} \\
        =&
        \frac{2}{\eta}
        \norm{\brac{\sum_{t=1}^T\Gb_t\Gb_t^\top}^{1/2}}_*,
    \end{align*}
    where we used $\norm{\WB_t^{(\epsilon)}-\WB_*}_\op\le\norm{\WB_t^{(\epsilon)}}_\op+\norm{\WB_*}_\op\le2$ and the fact that
    $\bLambda_t^{(\epsilon)}-\bLambda_{t-1}^{(\epsilon)}\succeq\0$.

    For the second term, the matrix AdaGrad trace inequality~\citep[Eq.~(9)]{an2025asgo} gives
    \begin{align*}
        \frac{\eta}{2}
        \sum_{t=1}^T
        \norm{\Gb_t}_{\brac{\bLambda_t^{(\epsilon)}}^{-1}}^2
        \le
        \eta
        \sum_{t=1}^T
        \tr{\bLambda_t^{(\epsilon)}-\bLambda_{t-1}^{(\epsilon)}}
        =
        \eta
        \norm{\brac{\sum_{t=1}^T\Gb_t\Gb_t^\top}^{1/2}}_*.
    \end{align*}
    Combining the preceding bounds, we obtain
    \begin{align*}
        \sum_{t=1}^T\inner{\Gb_t}{\WB_t^{(\epsilon)}}
        \le
        \brac{\frac{2}{\eta}+\eta}
        \norm{\brac{\sum_{t=1}^T\Gb_t\Gb_t^\top}^{1/2}}_*
        +
        \sum_{t=1}^T\inner{\Gb_t}{\WB_*}
        +
        \frac{\epsilon\norm{\WB_*}_\Fn^2}{2\eta}.
    \end{align*}
    Now choose
    \[
        \WB_*
        \in
        \argmin_{\WB\in\W}
        \inner{\sum_{t=1}^T\Gb_t}{\WB}.
    \]
    By nuclear/operator norm duality,
    \[
        \sum_{t=1}^T\inner{\Gb_t}{\WB_*}
        =
        -\norm{\sum_{t=1}^T\Gb_t}_*.
    \]
    Taking $\eta=\sqrt{2}$ yields
    \begin{align}\label{eq:asgo-regret-eps}
        \sum_{t=1}^T\inner{\Gb_t}{\WB_t^{(\epsilon)}}
        \le
        2\sqrt{2}
        \norm{\brac{\sum_{t=1}^T\Gb_t\Gb_t^\top}^{1/2}}_*
        -
        \norm{\sum_{t=1}^T\Gb_t}_*
        +
        \frac{\epsilon\norm{\WB_*}_\Fn^2}{2\sqrt{2}}.
    \end{align}

    If $\cbrac{\Gb_t}_{t=1}^T$ is a martingale difference sequence, then $\WB_t^{(\epsilon)}$ is $\F_{t-1}$-measurable, and thus
    \begin{align*}
        \E\sqbrac{\inner{\Gb_t}{\WB_t^{(\epsilon)}}}
        =
        \E\sqbrac{
            \inner{
                \E\sqbrac{\Gb_t|\F_{t-1}}
            }{
                \WB_t^{(\epsilon)}
            }
        }
        =
        0.
    \end{align*}
    Taking expectations in~\eqref{eq:asgo-regret-eps} and rearranging gives
    \begin{align*}
        \E\sqbrac{\norm{\sum_{t=1}^T\Gb_t}_*}
        \le
        2\sqrt{2}
        \E\sqbrac{\norm{\brac{\sum_{t=1}^T\Gb_t\Gb_t^\top}^{1/2}}_*}
        +
        \frac{\epsilon}{2\sqrt{2}}
        \E\sqbrac{\norm{\WB_*}_\Fn^2}.
    \end{align*}
    Since $\norm{\WB_*}_\op\le1$, we have
    \[
        \norm{\WB_*}_\Fn^2\le\min\cbrac{m,n}.
    \]
    Thus the last term vanishes by taking $\epsilon\downarrow0$, and we conclude that
    \begin{align*}
        \E\sqbrac{\norm{\sum_{t=1}^T\Gb_t}_*}
        \le
        2\sqrt{2}
        \E\sqbrac{\norm{\brac{\sum_{t=1}^T\Gb_t\Gb_t^\top}^{1/2}}_*}.
    \end{align*}
\end{proof}

\begin{lem}[Stability with weight decay, matrix version of~\cref{lem:stability-wd}]\label{lem:stability-wd-matrix}
    Running~\cref{alg:muonlight}
    \begin{align*}
        \lambda\le\frac{1}{\eta}\brac{1-\frac{1}{2^{1/T}}},\text{ and }\norm{\XB_1}_\op\le\frac{a}{\lambda},\quad\forall a\in(0,1),
    \end{align*}
    ensures that
    \begin{align*}
        \norm{\XB_t}_\op\le\frac{a+1}{2\lambda},\quad\norm{\XB_{t+1}-\XB_t}_\op\le\frac{a+3}{2}\eta,\quad\forall t\in[T].
    \end{align*}
\end{lem}

\begin{proof}
    The proof of this lemma is analogous to that of~\cref{lem:stability-wd}. Denote $1-\eta\lambda$ by $q$, which implies $q\in(1/2^{1/T},1)$ as $\lambda\le\frac{1}{\eta}\brac{1-\frac{1}{2^{1/T}}}$. By the update rule of~\cref{alg:muonlight}, we have
    \begin{align*}
        &\norm{\XB_t}_\op=\norm{(1-\eta\lambda)\XB_{t-1} - \eta\msign{\BBT_{t-1}}}_\op\le q\norm{\XB_{t-1}}_\op+\eta\\\le&q^2\norm{\XB_{t-2}}_\op+q\eta+\eta\le\cdots\le q^{t-1}\norm{\XB_1}_\op+\eta\sum_{i=0}^{t-2}q^i\le\frac{aq^{t-1}}{\lambda}+\frac{\eta(1-q^{t-1})}{1-q}\\=& \frac{1-(1-a)q^{t-1}}{\lambda}\le\frac{1-(1-a)\brac{\frac{1}{2}}^{\frac{t-1}{T}}}{\lambda}\le\frac{a+1}{2\lambda},
    \end{align*}
    which holds for any $t\in[T]$. Therefore, we can deduce
    \begin{align*}
        \norm{\XB_{t+1}-\XB_t}_\op=\norm{\eta\lambda\XB_t+\eta\msign{\BBT_t}}_\op\le \eta\lambda\norm{\XB_t}_\op+\eta\le\frac{a+3}{2}\eta,\quad\forall t\in[T].
    \end{align*}
\end{proof}

\subsection{Proof of Theorem~\ref{thm:muon}\label{sec:proof-muon}}

\paragraph{Preliminaries} Throughout this section, we define
\begin{equation}\label{eq:notations-muon}
    \begin{aligned}
        &\bnabla_t:=\nabla f(\XB_t),\quad\MB_t:=(1-\beta)\BB_t,\\&\EB_t:=\MB_t-\bnabla_t,\quad\NB_t:=\Gb_t-\bnabla_t,\quad\SB_t:=\bnabla_{t-1}-\bnabla_t.
    \end{aligned}
\end{equation}
Under the above notation, the momentum term $\BB_t=\beta\BB_{t-1}+\Gb_t$ in~\cref{alg:muon} can be rewritten as the one with a damping factor, i.e., $\MB_t=\beta\MB_{t-1}+(1-\beta)\Gb_t$, allowing us to borrow the similar derivations for~\cref{thm:signsgd,thm:lion}. Also, under the exact Newton--Schulz oracle assumption, we have $\OB_t=\texttt{NewtonSchulz}(\BB_t)=\msign{\BB_t}$. Since $\msign{\cdot}$ is invariant to input scale according to \cref{lem:msign-property}, so it further implies $\OB_t=\msign{\MB_t}$. With the above preparations, we formally begin the proof.

By choosing $\eta\le1/\norm{\Lb_1}_\op$, we can ensure $\norm{\XB_{t+1}-\XB_t}_\op=\eta\norm{\OB_t}_\op\le 1/\norm{\Lb_1}_\op$. Recall the definitions and notations in~\eqref{eq:notations-muon}, according to~\cref{lem:descent-lemma}, we have
\begin{equation}
    \begin{aligned}
        f(\XB_{t+1})\le& f(\XB_t)-\eta\inner{\bnabla_t}{\msign{\BB_t}}+\frac{\eta^2}{2}\norm{\OB_t}_{\Lb(\XB_t)}^2\\\overset{\textnormal{\cref{lem:msign-property}}}{\le}&f(\XB_t)-\eta\inner{\bnabla_t}{\msign{\bnabla_t}}+\eta\inner{\bnabla_t}{\msign{\bnabla_t}-\msign{\MB_t}}\\&+\frac{\eta^2}{2}\tr{\mabs{\Lb_0}+\mabs{\bnabla_t\Lb_1^\top}}\\\overset{\textnormal{\cref{lem:trace-property,lem:msign-property,lem:polar-difference}}}{\le}&f(\XB_t)-\eta\norm{\bnabla_t}_*+2\eta\norm{\bnabla_t-\MB_t}_*+\frac{\eta^2\norm{\Lb_0}_*}{2}+\frac{\eta^2}{2}\norm{\Lb_1}_\op\norm{\bnabla_t}_*\\\le&f(\XB_t)-\frac{\eta}{2}\norm{\bnabla_t}_*+2\eta\norm{\bnabla_t-\MB_t}_*+\frac{\eta^2\norm{\Lb_0}_*}{2}.
    \end{aligned}\label{eq:muon-onestep}
\end{equation}
    Rearranging the above relation and summing from $1$ to $T$ yields
\begin{align}\label{eq:Linf-mid-muon}
        \E\sqbrac{\frac{1}{T}\sum_{t=1}^{T} \norm{\bnabla_t}_*} 
        \le \frac{2\Delta_f}{\eta T}+\eta\norm{\Lb_0}_* +  4 \E\sqbrac{\frac{1}{T}\sum_{t=1}^{T} \norm{ \MB_t-\bnabla_t}_*},
\end{align}
where we use  $\Delta_{f}=f(\XB_{1})-f_{*}\ge f(\XB_1)-f(\XB_{T+1})$. By~\eqref{eq:notations-muon}, we have that $\MB_t=\beta\MB_{t-1}+(1-\beta)\Gb_t$. Hence, we follow a similar expansion as in~\eqref{eq:eps-expand} to get
\begin{align*}
    &\EB_t=\MB_t-\bnabla_t=\beta\MB_{t-1}+(1-\beta)\Gb_t-\bnabla_t\\=&\beta\brac{\MB_{t-1}-\bnabla_{t-1}}+\beta\brac{\bnabla_{t-1}-\bnabla_{t}}+(1-\beta)\brac{\Gb_t-\bnabla_t}\\=&\beta\EB_{t-1}+\beta\SB_t+(1-\beta)\NB_t.
\end{align*}
which implies
\begin{align*}
        \EB_t=-\beta^t\bnabla_1+(1-\beta)\sum_{k=1}^t\beta^{t-k}\NB_k+\sum_{k=2}^t\beta^{t-k+1}\SB_k,
    \end{align*}
    where we utilize $\EB_1=\MB_1-\bnabla_1=(1-\beta)\Gb_1-\bnabla_1=(1-\beta)\NB_1-\beta\bnabla_1$. We then decompose $\EB_t$ into
    \begin{align*}
    \E\sqbrac{\norm{\EB_t}_*}\le \underbrace{\norm{\beta^t\bnabla_1}_*}_{\mathtt{A_t}}+\underbrace{\E\sqbrac{\norm{(1-\beta)\sum_{k=1}^t\beta^{t-k}\NB_k}_*}}_{\mathtt{B_t}}+\underbrace{\E\sqbrac{\norm{\sum_{k=2}^t\beta^{t-k+1}\SB_k}_*}}_{\mathtt{C_t}},
\end{align*}
and bound these terms separately.

\paragraph{Cumulative noise $\mathtt{B_t}$}
Define $\NBT_k:=\beta^{t-k}\NB_k,k\in[1,t]$. Under~\cref{ass:unbiased-matrix}, we can show that
\begin{equation}\label{eq:Bt-cancellation}
    \begin{aligned}
    &\frac{\mathtt{B_t}}{1-\beta}\overset{\textnormal{\cref{lem:asgo-regret}}}{\le}2\sqrt{2}\E\sqbrac{\norm{\brac{\sum_{k=1}^t\NBT_k\NBT_k^\top}^{1/2}}_*}\\\overset{\textnormal{\cref{lem:matrix-cauchy-schwarz}}}{\le}&2\sqrt{2}\E\sqbrac{\sqrt{\norm{\brac{\V_0\V_0^\top}^{1/2}}_*\tr{\brac{\sum_{k=1}^t\NBT_k\NBT_k^\top}^{1/2}\mabs{\V_0}^{-1}\brac{\sum_{k=1}^t\NBT_k\NBT_k^\top}^{1/2}}}}\\=&2\sqrt{2\norm{\V_0}_*}\cdot\E\sqbrac{\sqrt{\tr{\brac{\sum_{k=1}^t\NBT_k\NBT_k^\top}\mabs{\V_0}^{-1}}}}.
\end{aligned}
\end{equation}
We proceed by following the conditional expectation recursion in~\eqref{eq:recursion}:
\begin{align*}
    &\E\sqbrac{\left.\sqrt{\tr{\brac{\sum_{k=1}^t\NBT_k\NBT_k^\top}\mabs{\V_0}^{-1}}}\right|\F_{t-1}}\\=&\E\sqbrac{\left.\brac{\brac{\sum_{k=1}^t\tr{\NBT_k\NBT_k^\top\mabs{\V_0}^{-1}}}^{p/2}}^{1/p}\right|\F_{t-1}}\\\overset{\textnormal{\cref{lem:lp-mean}}}{\le}&\brac{\E\sqbrac{\left.\brac{\sum_{k=1}^t\norm{\NBT_k}_{\mabs{\V_0}^{-1}}^2}^{p/2}\right|\F_{t-1}}}^{1/p}\\\overset{\textnormal{\cref{lem:minkowski}}}{\le}&\brac{\brac{\sum_{k=1}^{t-1}\norm{\NBT_k}_{\mabs{\V_0}^{-1}}^2}^{p/2}+\E\sqbrac{\left.\brac{\tr{\NBT_t\NBT_t^\top\mabs{\V_0}^{-1}}}^{p/2}\right|\F_{t-1}}}^{1/p}\\\overset{\textnormal{\cref{lem:batch-noise-matrix}}}{\le}&\brac{\brac{\sum_{k=1}^{t-1}\norm{\NBT_k}_{\mabs{\V_0}^{-1}}^2}^{p/2}+2B^{1-p}\beta^{p(t-t)}\brac{\norm{\V_0}_*^{p/2}+\frac{\norm{\V_1}_{\op}^p\norm{\bnabla_t}_*^p}{\norm{\V_0}_*^{p/2}}}}^{1/p}\\\overset{\textnormal{\cref{lem:minkowski}}}{\le}&\brac{\brac{\sum_{k=1}^{t-1}\norm{\NBT_k}_{\mabs{\V_0}^{-1}}^2}^{p/2}+2B^{1-p}\beta^{p(t-t)}\norm{\V_0}_*^{p/2}}^{1/p}+\frac{2^{\frac{1}{p}}\beta^{(t-t)}\norm{\V_1}_{\op}\norm{\bnabla_t}_*}{B^{\frac{p-1}{p}}\norm{\V_0}_*^{1/2}}.
\end{align*}
Taking total expectations on both sides and applying the above relation recursively:
\begin{align*}
    &\E\sqbrac{\sqrt{\tr{\brac{\sum_{k=1}^t\NBT_k\NBT_k^\top}\mabs{\V_0}^{-1}}}}\le\underbrace{\left.\E\sqbrac{\frac{2\beta^{(t-k)}\norm{\V_1}_{\op}\norm{\bnabla_k}_*}{B^{\frac{p-1}{p}}\norm{\V_0}_*^{1/2}}}\right|_{k=t}}_{\Phi_k|_{k=t}}\\&+\E\sqbrac{\brac{\brac{\sum_{k=1}^{t-1}\norm{\NBT_k}_{\mabs{\V_0}^{-1}}^2}^{p/2}+2B^{1-p}\beta^{p(t-t)}\norm{\V_0}_*^{p/2}}^{1/p}}\\\overset{\textnormal{\cref{lem:minkowski}}}{\le}&\brac{\E\sqbrac{\brac{\sum_{k=1}^{t-1}\norm{\NBT_k}_{\mabs{\V_0}^{-1}}^2}^{p/2}+2B^{1-p}\beta^{p(t-t)}\norm{\V_0}_*^{p/2}}}^{1/p}+\Phi_t\\=&\brac{\E\sqbrac{\E\sqbrac{\left.\brac{\sum_{k=1}^{t-1}\norm{\NBT_k}_{\mabs{\V_0}^{-1}}^2}^{p/2}\right|\F_{t-2}}}+2B^{1-p}\beta^{p(t-t)}\norm{\V_0}_*^{p/2}}^{1/p}+\Phi_t\\\le&\brac{\E\sqbrac{\E\sqbrac{\left.\brac{\sum_{k=1}^{t-2}\norm{\NBT_k}_{\mabs{\V_0}^{-1}}^2}^{p/2}\right|\F_{t-3}}}+2B^{1-p}\brac{\beta^{p(t-(t-1))}+\beta^{p(t-t)}}\norm{\V_0}_*^{p/2}}^{1/p}\\&+\Phi_{t-1}+\Phi_t\le\cdots\\\le&\brac{2B^{1-p}\norm{\V_0}_*^{p/2}\sum_{k=1}^t\beta^{p(t-k)}}^{1/p}+\sum_{k=1}^t\Phi_k\\\le&\frac{2\norm{\V_0}_*^{p/2}}{B^{\frac{p-1}{p}}(1-\beta^p)^{1/p}}+\sum_{k=1}^t\E\sqbrac{\frac{2\beta^{t-k}\norm{\V_1}_{\op}\norm{\bnabla_k}_*}{B^{\frac{p-1}{p}}\norm{\V_0}_*^{1/2}}}.
\end{align*}
Hence, we can bound $\mathtt{B_t}$ as 
\begin{align}\label{eq:muon-Bt}
    \mathtt{B_t}\le\frac{4\sqrt{2}(1-\beta)^{\frac{p-1}{p}}\norm{\V_0}_*}{B^{\frac{p-1}{p}}}+\frac{4\sqrt{2}(1-\beta)\norm{\V_1}_{\op}}{B^{\frac{p-1}{p}}}\sum_{k=1}^t\E\sqbrac{\beta^{t-k}\norm{\bnabla_k}_*},
\end{align}
where we make use of $(1-\beta^p)^{-1/p}\le(1-\beta)^{-1/p}$.

\paragraph{Trajectory curvature $\mathtt{C_t}$}

Under~\cref{ass:generalized-smooth-matrix} with $\Lb(\XB)=\mabs{\Lb_0}+\mabs{\nabla f(\XB)\Lb_1^\top}$, together with the fact that $\norm{\XB_k-\XB_{k-1}}_\op=\eta\norm{\OB_{k-1}}_\op\le1/\norm{\Lb_1}_\op$, we have
\begin{align*}
    \mathtt{C_t}\le&\sum_{k=2}^t\beta^{t-k+1}\E\sqbrac{\norm{\SB_k}_*}\overset{\textnormal{\cref{lem:matrix-cauchy-schwarz}}}{\le}\sum_{k=2}^t\beta^{t-k+1}\E\sqbrac{\sqrt{\norm{\Lb(\XB_k)}_*}\norm{\SB_k}_{\brac{\Lb(\XB_k)}^{-1}}}\\\le&\sum_{k=2}^t\beta^{t-k+1}\E\sqbrac{\sqrt{\tr{\Lb(\XB_k)}}\norm{\eta\OB_{k-1}}_{\Lb(\XB_k)}}\\\overset{\textnormal{\cref{lem:msign-property}}}{\le}&\eta\sum_{k=2}^t\beta^{t-k+1}\E\sqbrac{\tr{\Lb(\XB_k)}}\\\overset{\textnormal{\cref{lem:trace-property}}}{\le}&\eta\sum_{k=2}^t\beta^{t-k+1}\tr{\mabs{\Lb_0}}+\eta\sum_{k=2}^t\beta^{t-k+1}\E\sqbrac{\norm{\Lb_1}_\op\norm{\bnabla_k}_*}\\\le&\frac{\eta\norm{\Lb_0}_*}{1-\beta}+\eta\norm{\Lb_1}_\op\sum_{k=2}^t\beta^{t-k+1}\E\sqbrac{\norm{\bnabla_k}_*}.
\end{align*}
Combining the bounds for $\mathtt{A_t},\mathtt{B_t},\mathtt{C_t}$, and summing from $1$ to $T$:
    \begin{align*}
    &\frac{1}{T}\sum_{t=1}^T\E\sqbrac{\norm{\EB_t}_*}\le\frac{1}{T}\sum_{t=1}^T\beta^t\norm{\bnabla_1}_*\\&+\frac{4\sqrt{2}(1-\beta)^{\frac{p-1}{p}}\norm{\V_0}_*}{B^{\frac{p-1}{p}}}+\frac{4\sqrt{2}(1-\beta)\norm{\V_1}_{\op}}{TB^{\frac{p-1}{p}}}\sum_{t=1}^T\sum_{k=1}^t\E\sqbrac{\beta^{(t-k)}\norm{\bnabla_k}_*}\\&+\frac{\eta\norm{\Lb_0}_*}{1-\beta}+\frac{\eta\norm{\Lb_1}_\op}{T}\sum_{t=1}^T\sum_{k=2}^t\beta^{t-k+1}\E\sqbrac{\norm{\bnabla_k}_*}\\\overset{\textnormal{\eqref{eq:switch-sum}}}{\le}&\frac{\norm{\bnabla_1}_*}{T(1-\beta)}+\frac{4\sqrt{2}(1-\beta)^{\frac{p-1}{p}}\norm{\V_0}_*}{B^{\frac{p-1}{p}}}+\frac{\eta\norm{\Lb_0}_*}{1-\beta}\\&+\frac{1}{T}\sum_{t=1}^T\brac{\frac{4\sqrt{2}\norm{\V_1}_{\op}}{B^{\frac{p-1}{p}}}+\frac{\eta\beta\norm{\Lb_1}_\op}{1-\beta}}\E\sqbrac{\norm{\bnabla_t}_*}\\\overset{\textnormal{\eqref{eq:muon-params}}}{\le}&\frac{\norm{\bnabla_1}_*}{T(1-\beta)}+\frac{4\sqrt{2}(1-\beta)^{\frac{p-1}{p}}\norm{\V_0}_*}{B^{\frac{p-1}{p}}}+\frac{\eta\norm{\Lb_0}_*}{1-\beta}+\sum_{t=1}^T\frac{1}{8T}\E\sqbrac{\norm{\bnabla_t}_*}.
\end{align*}
Plugging the above relation into~\eqref{eq:Linf-mid-muon}, we obtain
\begin{align*}
    \frac{1}{T}\sum_{t=1}^{T}\E\sqbrac{\norm{\bnabla_t}_*} 
        \le \frac{4\Delta_f}{\eta T}+ \frac{8\norm{\bnabla_1}_*}{T(1-\beta)}+\frac{32\sqrt{2}(1-\beta)^{\frac{p-1}{p}}\norm{\V_0}_*}{B^{\frac{p-1}{p}}}+\frac{10\eta\norm{\Lb_0}_*}{1-\beta}.
\end{align*}
Finally, choosing $B,\beta,\eta$ according to~\eqref{eq:muon-params} and following the similar steps as in~\cref{sec:proof-signsgd}, we can obtain the convergence rate $O\brac{(\Delta_f\norm{\Lb_0}_*)^{\frac{p-1}{3p-2}}\norm{\V_0}_*^{\frac{p}{3p-2}}(BT)^{\frac{1-p}{3p-2}}}$.

\subsection{Proof of Theorem~\ref{thm:muonlight}}

\paragraph{Preliminaries} For~\cref{alg:muonlight}, we define
\begin{equation}\label{eq:notations-muonlight}
    \begin{aligned}
        &\bnabla_t:=\nabla f(\XB_t),\quad\MB_t:=(1-\beta_2)\BB_t,\quad\MBT_t:=(1-\beta_2)\BBT_t,\\&\EB_t:=\MB_t-\bnabla_t,\quad\NB_t:=\Gb_t-\bnabla_t,\quad\SB_t:=\bnabla_{t-1}-\bnabla_t.
    \end{aligned}
\end{equation}
Based on the above notation, we have $\MB_t=\beta_2\MB_{t-1}+(1-\beta_2)\Gb_t$ and $\MBT_t=\beta_1\MB_t+(1-\beta_2)\Gb_t$. We also have $\OB_t=\texttt{NewtonSchulz}(\BBT_t)=\msign{\MBT_t}$. Now, we begin the proof as follows.

By~\cref{lem:stability-wd-matrix} and $\norm{\XB_1}_\op\le1/(3\lambda)$, $\norm{\XB_{t+1}-\XB_t}_{\op}\le5\eta/3$ holds for any $t\in[T]$, which immediately implies $\norm{\XB_{t+1}-\XB_t}_{\op}\le1/\norm{\Lb_1}_\op$ by the choice of $\eta$ in~\eqref{eq:muonlight-params}. Thus, we invoke~\cref{lem:descent-lemma} to deduce that
\begin{align*}
    f(\XB_{t+1})\le& f(\XB_t)-\eta\inner{\bnabla_t}{\msign{\BBT_t}+\lambda\XB_t}\\&+\frac{1}{2}\tr{\brac{\XB_{t+1}-\XB_t}^\top\Lb(\XB_t)\brac{\XB_{t+1}-\XB_t}}\\\overset{\textnormal{\cref{lem:msign-property}}}{\le}&f(\XB_t)-\eta\inner{\bnabla_t}{\msign{\bnabla_t}}+\eta\inner{\bnabla_t}{\msign{\bnabla_t}-\msign{\MBT_t}}\\&+\eta\lambda\norm{\XB_t}_\op\norm{\bnabla_t}_*+\frac{1}{2}\norm{\XB_{t+1}-\XB_t}_\op^2\norm{\Lb(\XB_t)}_*\\\overset{\textnormal{\cref{lem:trace-property,lem:msign-property,lem:polar-difference,lem:stability-wd-matrix}}}{\le}&f(\XB_t)-\eta\norm{\bnabla_t}_*+2\eta\norm{\bnabla_t-\MBT_t}_*+\frac{2\eta}{3}\norm{\bnabla_t}_*
    +\frac{25\eta^2}{9}\tr{\mabs{\Lb_0}+\mabs{\bnabla_t\Lb_1^\top}}
    \\\overset{\textnormal{\cref{lem:trace-property}}}{\le}&f(\XB_t)-\frac{\eta}{3}\norm{\bnabla_t}_*+2\eta\norm{\bnabla_t-\MBT_t}_*+\frac{25\eta^2\norm{\Lb_0}_*}{9}+\frac{25\eta^2\norm{\Lb_1}_\op\norm{\bnabla_t}_*}{9}\\\overset{\eqref{eq:muonlight-params}}{\le}&f(\XB_t)-\frac{8\eta}{25}\norm{\bnabla_t}_*+2\eta\norm{\bnabla_t-\MBT_t}_*+\frac{25\eta^2\norm{\Lb_0}_*}{9}, 
\end{align*}
where the second inequality uses Cauchy-Schwarz inequality and the fact that $\norm{\cdot}_\op$ is submultiplicative; the last step is due to $\eta\le3/(625\norm{\Lb_1}_\op)$. Rearranging the above relation and summing from $1$ to $T$ yields
\begin{align}\label{eq:Linf-mid-muonlight}
    \frac{1}{T}\sum_{t=1}^{T}\E\sqbrac{ \norm{\bnabla_t}_*} 
    \le \frac{25\Delta_f}{8\eta T}+\frac{625\eta}{72}\norm{\Lb_0}_*+\frac{25}{4T}\sum_{t=1}^{T}\E\sqbrac{ \norm{ \MBT_t-\bnabla_t}_*},
\end{align}
Recall the notations in~\eqref{eq:notations-muonlight} and decompose $\E\sqbrac{\norm{ \MBT_t-\bnabla_t}_*}$ as follows:
\begin{equation}\label{eq:muonlight-decomposition}
    \begin{aligned}
        &\E\sqbrac{\norm{\MBT_t-\bnabla_t}_*}=\E\sqbrac{\norm{\beta_1\MB_{t}+(1-\beta_2)\Gb_t-\bnabla _t}_*}\\=&\E\sqbrac{\norm{\beta_1(\MB_t-\bnabla_t)+(1-\beta_2)(\Gb_t-\bnabla_t)+(\beta_1-\beta_2)(\bnabla_t)}_*}\\\le&\beta_1\E\sqbrac{\norm{\EB_t}_*}+(1-\beta_2)\E\sqbrac{\norm{\NB_t}_*}+\abs{\beta_1-\beta_2}\E\sqbrac{\norm{\bnabla_t}_*},
    \end{aligned}
\end{equation}
which holds for all $t\in[T]$. The noise term $\NB_t$ can be bounded by
\begin{equation}\label{eq:muonlight-NBt}
    \begin{aligned}
        \E\sqbrac{\norm{\NB_t}_*}\overset{\textnormal{\cref{lem:matrix-cauchy-schwarz}}}{\le}&\E\sqbrac{\sqrt{\norm{\mabs{\V_0}}_*\tr{\NB_t^\top\mabs{\V_0}^{-1}\NB_t}}}\\=&\sqrt{\norm{\V_0}_*}\E\sqbrac{\E\sqbrac{\left.\norm{\NB_t}_{\mabs{\V_0}^{-1}}\right|\F_{t-1}}}\\\overset{\textnormal{\cref{lem:lp-mean}}}{\le}&\sqrt{\norm{\V_0}_*}\E\sqbrac{\brac{\E\sqbrac{\left.\norm{\NB_t}^p_{\mabs{\V_0}^{-1}}\right|\F_{t-1}}}^{1/p}}\\\overset{\textnormal{\cref{lem:batch-noise-matrix}}}{\le}&\sqrt{\norm{\V_0}_*}\E\sqbrac{\brac{2B^{1-p}\brac{\norm{\V_0}_*^{p/2}+\frac{\norm{\V_1}_{\op}^p\norm{\bnabla_t}_*^p}{\norm{\V_0}_*^{p/2}}}}^{1/p}}\\\overset{\textnormal{\cref{lem:minkowski}}}{\le}&2B^{\frac{1-p}{p}}\sqrt{\norm{\V_0}_*}\E\sqbrac{\norm{\V_0}_*^{1/2}+\frac{\norm{\V_1}_{\op}\norm{\bnabla_t}_*}{\norm{\V_0}_*^{1/2}}}\\=&2B^{\frac{1-p}{p}}\norm{\V_0}_*+2B^{\frac{1-p}{p}}\norm{\V_1}_{\op}\E\sqbrac{\norm{\bnabla_t}_*}.
    \end{aligned}
\end{equation}
As for the error term $\EB_t$, we mainly adopt the same procedures as in~\cref{sec:proof-muon}. Following~\eqref{eq:eps-expand}, $\EB_t$ could be expressed as $\EB_t=\beta_2\EB_{t-1}+\beta_2\SB_t+(1-\beta_2)\NB_t$, implying
    \begin{align*}
        \EB_t=-\beta_2^t\bnabla_1+(1-\beta_2)\sum_{k=1}^t\beta_2^{t-k}\NB_k+\sum_{k=2}^t\beta_2^{t-k+1}\SB_k,
    \end{align*}
    where we utilize $\EB_1=\MB_1-\bnabla_1=(1-\beta)\Gb_1-\bnabla_1=(1-\beta)\NB_1-\beta\bnabla_1$. Then, it holds that
    \begin{align*}
    \E\sqbrac{\norm{\EB_t}_*}\le \underbrace{\norm{\beta_2^t\bnabla_1}_*}_{\mathtt{A_t}}+\underbrace{\E\sqbrac{\norm{(1-\beta_2)\sum_{k=1}^t\beta_2^{t-k}\NB_k}_*}}_{\mathtt{B_t}}+\underbrace{\E\sqbrac{\norm{\sum_{k=2}^t\beta_2^{t-k+1}\SB_k}_*}}_{\mathtt{C_t}}.
\end{align*}
For $\mathtt{B_t}$, we can safely replace the $\beta$ in~\eqref{eq:muon-Bt} by $\beta_2$ to derive 
\begin{align*}
    \mathtt{B_t}\le\frac{4\sqrt{2}(1-\beta_2)^{\frac{p-1}{p}}\norm{\V_0}_*}{B^{\frac{p-1}{p}}}+\frac{4\sqrt{2}(1-\beta_2)\norm{\V_1}_{\op}}{B^{\frac{p-1}{p}}}\sum_{k=1}^t\E\sqbrac{\beta_2^{t-k}\norm{\bnabla_k}_*},
\end{align*}
    For $\mathtt{C_t}$, we need to cope with weight decay carefully. By~\cref{lem:stability-wd-matrix}, we have $\norm{\XB_t-\XB_{t-1}}_\op\le5\eta/3\le1/\norm{\Lb}_\op$. So we can utilize~\cref{ass:generalized-smooth-matrix} in the following way:
    \begin{align*}
    \mathtt{C_t}\le&\sum_{k=2}^t\beta_2^{t-k+1}\E\sqbrac{\norm{\SB_k}_*}\overset{\textnormal{\cref{lem:matrix-cauchy-schwarz}}}{\le}\sum_{k=2}^t\beta_2^{t-k+1}\E\sqbrac{\sqrt{\norm{\Lb(\XB_k)}_*}\norm{\SB_k}_{\brac{\Lb(\XB_k)}^{-1}}}\\\le&\sum_{k=2}^t\beta_2^{t-k+1}\E\sqbrac{\sqrt{\norm{\Lb(\XB_k)}_*}\norm{\XB_k-\XB_{k-1}}_{\Lb(\XB_k)}}\\\le&\sum_{k=2}^t\beta_2^{t-k+1}\E\sqbrac{\sqrt{\norm{\Lb(\XB_k)}_*\cdot\norm{\XB_k-\XB_{k-1}}_\op^2\cdot\norm{\Lb(\XB_k)}_*}}
    \\\overset{\textnormal{\cref{lem:stability-wd-matrix}}}{\le}&\sum_{k=2}^t\beta_2^{t-k+1}\E\sqbrac{\tr{\Lb(\XB_k)}\cdot\frac{5\eta}{3}}\\\overset{\textnormal{\cref{lem:trace-property}}}{\le}&\frac{5\eta}{3}\sum_{k=2}^t\beta_2^{t-k+1}\tr{\mabs{\Lb_0}}+\eta\sum_{k=2}^t\beta_2^{t-k+1}\E\sqbrac{\norm{\Lb_1}_\op\norm{\bnabla_k}_*}\\\le&\frac{5\eta\norm{\Lb_0}_*}{3(1-\beta_2)}+\frac{5\eta}{3}\norm{\Lb_1}_\op\sum_{k=2}^t\beta_2^{t-k+1}\E\sqbrac{\norm{\bnabla_k}_*},
\end{align*}
where the fourth inequality leverages Cauchy-Schwarz inequality and $\norm{AB}_\op\le\norm{A}_\op\norm{B}_\op,\forall A\in\R^{m\times n},B\in\R^{n\times r}$. Combining the bounds for $\mathtt{A_t},\mathtt{B_t},\mathtt{C_t}$, and summing from $1$ to $T$:
    \begin{align*}
    &\frac{1}{T}\sum_{t=1}^T\E\sqbrac{\norm{\EB_t}_*}\le\frac{1}{T}\sum_{t=1}^T\beta_2^t\norm{\bnabla_1}_*+\frac{5\eta\norm{\Lb_0}_*}{3(1-\beta_2)}+\frac{5\eta\norm{\Lb_1}_\op}{3T}\sum_{t=1}^T\sum_{k=2}^t\beta_2^{t-k+1}\E\sqbrac{\norm{\bnabla_k}_*}\\&+\frac{4\sqrt{2}(1-\beta_2)^{\frac{p-1}{p}}\norm{\V_0}_*}{B^{\frac{p-1}{p}}}+\frac{4\sqrt{2}(1-\beta_2)\norm{\V_1}_{\op}}{TB^{\frac{p-1}{p}}}\sum_{t=1}^T\sum_{k=1}^t\E\sqbrac{\beta_2^{(t-k)}\norm{\bnabla_k}_*}\\\overset{\textnormal{\eqref{eq:switch-sum}}}{\le}&\frac{\norm{\bnabla_1}_*}{T(1-\beta_2)}+\frac{4\sqrt{2}(1-\beta_2)^{\frac{p-1}{p}}\norm{\V_0}_*}{B^{\frac{p-1}{p}}}+\frac{5\eta\norm{\Lb_0}_*}{3(1-\beta_2)}\\&+\frac{1}{T}\sum_{t=1}^T\brac{\frac{4\sqrt{2}\norm{\V_1}_{\op}}{B^{\frac{p-1}{p}}}+\frac{5\eta\beta_2\norm{\Lb_1}_\op}{3(1-\beta_2)}}\E\sqbrac{\norm{\bnabla_t}_*}.
\end{align*}
Plugging the above bound for $\E\sqbrac{\norm{\EB_t}_*}$, the bound for $\E\sqbrac{\norm{\NB_t}_*}$ in~\eqref{eq:muonlight-NBt} into~\eqref{eq:muonlight-decomposition}:
\begin{align*}
    &\frac{1}{T}\sum_{t=1}^T\E\sqbrac{\norm{\MBT_t-\bnabla_t}_*}\le\frac{\beta_1}{T}\sum_{t=1}^T\E\sqbrac{\norm{\EB_t}_*}+\frac{1-\beta_2}{T}\sum_{t=1}^T\E\sqbrac{\norm{\NB_t}_*}+\frac{\abs{\beta_1-\beta_2}}{T}\sum_{t=1}^T\E\sqbrac{\norm{\bnabla_t}_*}\\\le&\frac{\beta_1\norm{\bnabla_1}_*}{T(1-\beta_2)}+\frac{4\sqrt{2}\beta_1(1-\beta_2)^{\frac{p-1}{p}}\norm{\V_0}_*}{B^{\frac{p-1}{p}}}+\frac{5\beta_1\eta\norm{\Lb_0}_*}{3(1-\beta_2)}+\frac{(1-\beta_2)\norm{\V_0}_*}{B^{\frac{p-1}{p}}}\\&+\frac{1}{T}\sum_{t=1}^T\brac{\frac{\brac{4\sqrt{2}\beta_1+1-\beta_2}\norm{\V_1}_{\op}}{B^{\frac{p-1}{p}}}+\frac{5\eta\beta_1\beta_2\norm{\Lb_1}_\op}{3(1-\beta_2)}}\E\sqbrac{\norm{\bnabla_t}_*}+\frac{3}{20T}\sum_{t=1}^T\E\sqbrac{\norm{\bnabla_t}_*}\\\overset{\eqref{eq:muonlight-params}}{\le}&\frac{\beta_1\norm{\bnabla_1}_*}{T(1-\beta_2)}+\frac{(4\sqrt{2}+1)\beta_1\norm{\V_0}_*}{(1-\beta_2)^{\frac{1-p}{p}}B^{\frac{p-1}{p}}}+\frac{5\beta_1\eta\norm{\Lb_0}_*}{3(1-\beta_2)}+\frac{159}{1000T}\sum_{t=1}^T\E\sqbrac{\norm{\bnabla_t}_*},
\end{align*}
where the first inequality is due to $\abs{\beta_1-\beta_2}\le3/20$ as indicated by~\eqref{eq:muonlight-params}; the last step is due to
\begin{align*}
    \frac{\brac{4\sqrt{2}\beta_1+1-\beta_2}\norm{\V_1}_{\op}}{B^{\frac{p-1}{p}}}+\frac{5\eta\beta_1\beta_2\norm{\Lb_1}_\op}{3(1-\beta_2)}\le\frac{1}{1000}+\frac{5\times3}{3\times 625}=\frac{9}{1000}.
\end{align*}
Therefore, we plug in this relation into the initial bound in~\eqref{eq:Linf-mid-muonlight} to obtain
\begin{align*}
    &\frac{1}{T}\sum_{t=1}^{T}\E\sqbrac{ \norm{\bnabla_t}_*} 
    \le \frac{25\Delta_f}{8\eta T}+\frac{625\eta}{72}\norm{\Lb_0}_*\\&+\frac{25\beta_1\norm{\bnabla_1}_*}{4T(1-\beta_2)}+\frac{25(4\sqrt{2}+1)\beta_1\norm{\V_0}_*}{4(1-\beta_2)^{\frac{1-p}{p}}B^{\frac{p-1}{p}}}+\frac{125\beta_1\eta\norm{\Lb_0}_*}{12(1-\beta_2)}+\frac{159}{160T}\sum_{t=1}^T\E\sqbrac{\norm{\bnabla_t}_*},\\\Longrightarrow&\frac{1}{T}\sum_{t=1}^{T}\E\sqbrac{ \norm{\bnabla_t}_*} 
    \le\frac{500\Delta_f}{\eta T}+\frac{1000\norm{\bnabla_1}_*}{T(1-\beta_2)}+\frac{6658\norm{\V_0}_*}{(1-\beta_2)^{\frac{1-p}{p}}B^{\frac{p-1}{p}}}+\frac{1875\eta\norm{\Lb_0}_*}{1-\beta_2}.
\end{align*}
Setting $B,\beta_1,\beta_2,\eta$ according to~\eqref{eq:muonlight-params} completes the proof.

\section{Experimental Details\label{sec:experiments-add}}

We present the omitted details in~\cref{sec:experiments}.

\subsection{Validation of Heavy-Tailed Noise}
We first verify the validity of our proposed noise assumptions in the LLM pretraining regimes. To this end, we train nanoGPT~\citep{nanogpt} on the C4 dataset~\citep{JMLR:v21:20-074} using various optimizers and sample the stochastic gradient noise across different mini-batches. We first verify that the noise distributions are heavy-tailed in LLMs, with visualizations in~\cref{figs:error_distribution}. Then, we estimate the tail index $p$ following the methods in~\citet{simsekli19tailindex}. After that, we can approximate~\cref{ass:heavy-tailed-noise,ass:heavy-tailed-noise-matrix}, where the detailed approach is postponed to~\cref{sec:noise-validation}.~\cref{figs:assumption_4a_verification_start=0_j=1555144} depicts the relationship between the noise norm and the gradient magnitude for a randomly selected dimension $j\in[d]$ over consecutive iterations, while \cref{figs:assumption_4c_verification_start=0} illustrates the corresponding relationship for the matrix setting. We observe that in both cases, the expected noise norm scales linearly with the gradient norm. These empirical findings strongly corroborate our noise models in~\cref{ass:heavy-tailed-noise,ass:heavy-tailed-noise-matrix}, confirming that the noise magnitude is directly proportional to the gradient magnitude during training.

\subsection{Practical Efficiency of Sign-Based Optimizers}

Following the setup detailed in~\cref{sec:experimental-setup}, we train the nanoGPT model on the C4 dataset using Normalized SGD (NSGD)~\citep{nesterov1984minimization,cutkosky2020momentum}, AdamW~\citep{kingma15adam,loshchilov2019adamw}, SignSGD~\citep{bernstein2018signsgd}, Lion~\citep{chen2023symbolic}, Muon~\citep{jordan2024muon}, and Muonlight~\citep{liu2025muon}, where the results are shown in~\cref{fig:train_loss,fig:val_loss,fig:val_acc}. We observe a substantial performance gap between NSGD and sign-based methods (Lion and Muon), providing strong empirical evidence that sign-based optimizers are significantly more efficient than NSGD in the presence of heavy-tailed noise. This confirms the theoretical advantages established in \cref{sec:comparisons}. Furthermore, the matrix sign methods (Muon and Muonlight) consistently outperform AdamW, aligning with recent empirical findings in~\citet{liu2025muon,shah2025practical,wen2025fantastic,semenov2025benchmarking}.

\subsection{Experimental Setup and Methodology\label{sec:experimental-setup}}

We conduct all experiments for the NanoGPT model using PyTorch and Distributed Data Parallel (DDP) across four NVIDIA Pro6000 GPUs (96GB VRAM each). Our implementation extends the codebase provided by~\citet{semenov2025benchmarking}, incorporating additional modules for noise visualization. All models are trained for 10k steps with a global batch size of 512 sequences. This is achieved via a per-GPU batch size of 128 and 4 gradient accumulation steps per round. We use a sequence length of 512, totaling approximately $1\times$ the Chinchilla-optimal token count as proposed by~\citet{hoffmann2022chinchilla}. For the AdamW, Lion, and SignSGD optimizers, we adopt the standard hyperparameters for 124M-parameter models suggested by~\citet{semenov2025benchmarking}, specifically a learning rate $\eta = 10^{-3}$ and weight decay $\lambda = 0.1$. Meanwhile, we change the learning rates for Muon and Muonlight to $2\times lr$ of AdamW, as suggested in~\citet{jordan2024muon}. Finally, we employ a linear warm-up period of $10\%$ iterations at the start of pretraining. 


\begin{figure}[t]
\centering
\vspace{20pt}
\includegraphics[width=\columnwidth]{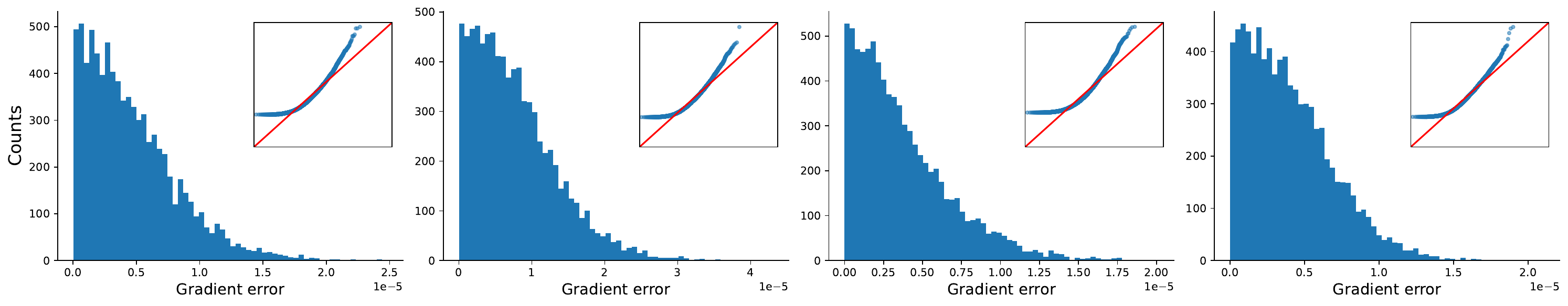}
\vspace{-10pt}
\caption{Noise histograms of nanoGPT on C4 at initialization sampled from different coordinates. Q-Q plots are shown at the top-right of each histogram, visualizing the distribution relative to a Gaussian (the red diagonal reference line).}
\label{figs:error_distribution}
\end{figure}

\subsection{Estimating the Heavy-Tail Index \texorpdfstring{$p$}{p}\label{sec:tail-index}}
Following the methodology of~\citet{zhang2020whyheavytail}, we estimate the heavy-tail index $p$ using the estimator introduced by~\citet{simsekli19tailindex}. To visualize this, we randomly sample coordinates from the model and plot the resulting noise histograms. These results, including quantile-quantile (Q-Q) plots, are presented in~\cref{figs:error_distribution}. In these Q-Q plots, a Gaussian distribution would align perfectly with the red diagonal reference line; the fact that our observed curves deviate above the diagonal indicates that the model experiences heavy-tailed noise at these coordinates.

\subsection{Validating the Noise Model\label{sec:noise-validation}}
To empirically characterize the evolution of noise throughout the pretraining phase, we save model checkpoints at fixed intervals. At each checkpoint, we compute the heavy-tail coefficient $p$ by applying the tail index estimator to gradients calculated over several mini-batches. Simultaneously, we estimate the gradient noise by aggregating stochastic gradients over subsequent iterations. We then plot the distribution of the empirical gradient norms against their expected values.~\cref{figs:assumption_4a_verification_start=0_j=1555144,figs:assumption_4c_verification_start=0} show the noise distribution in nanoGPT's attention layer on the C4 dataset with checkpoint at iteration $t=0$, while~\cref{figs:assumption_4a_verification_start=4000_j=884834,figs:assumption_4c_verification_start=4000} show the case when iteration $t=4000$.\footnote{Note that~\cref{ass:heavy-tailed-noise-matrix,ass:heavy-tailed-noise-matrix-variant} are equivalent in some sense, experimental on one implies the other.} These trends are consistent across all layers of the network.

\begin{figure}[t]
\centering
\includegraphics[width=\columnwidth]{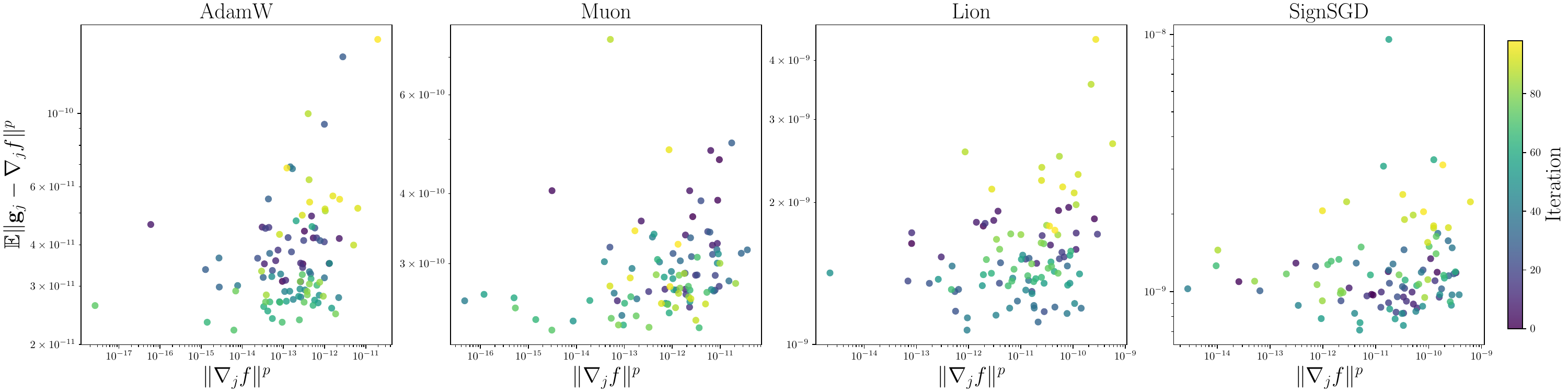}
\vspace{-10pt}
\caption{Verification of ~\cref{ass:heavy-tailed-noise}. \textbf{x-axis}: $|\nabla_j f|^p$, \textbf{y-axis}: $\E \sqbrac{\abs{\mathbf{g}_j - \nabla_j f}^p}$.}
\label{figs:assumption_4a_verification_start=4000_j=884834}
\end{figure}

\begin{figure}[t]
\centering
\includegraphics[width=\columnwidth]{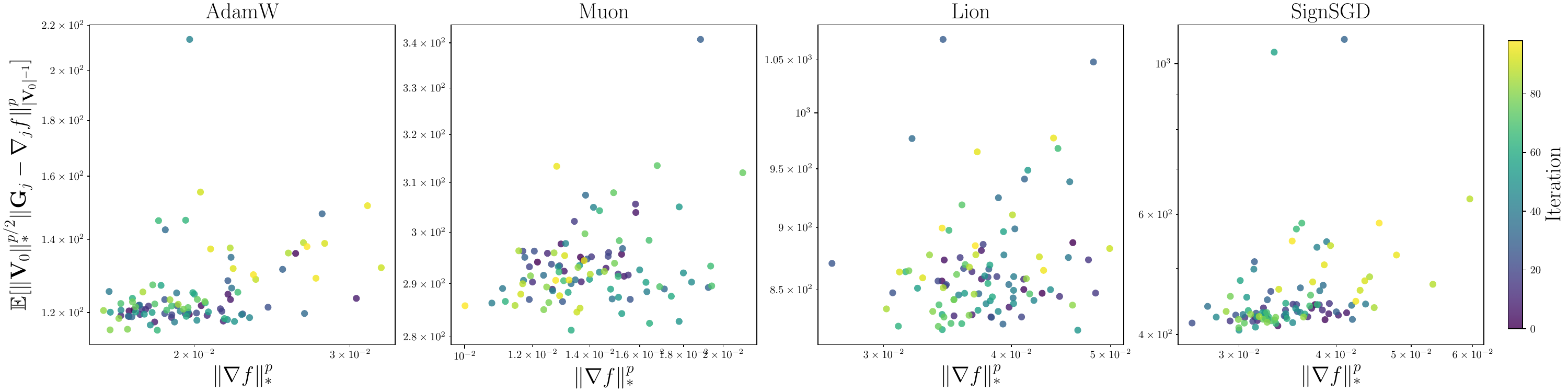}
\vspace{-10pt}
\caption{Verification of~\cref{ass:heavy-tailed-noise-matrix}. \textbf{x-axis}: $\norm{\nabla f}_*^p$, \textbf{y-axis}: $\E \sqbrac{\norm{\mathbf{V}_0}_*^{p/2} \norm{\mathbf{G} - \nabla f}^p_{\mabs{\mathbf{V}_0}^{-1}}}$.}
\label{figs:assumption_4c_verification_start=4000}
\end{figure}

\end{document}

%% file: macro.tex
\usepackage{booktabs}
\usepackage{multirow}
\usepackage{algorithm,algorithmic}
\usepackage{times}

\usepackage{amsmath,amsfonts,bm}
\usepackage{hyperref}
\usepackage{url}
\usepackage[capitalize,noabbrev]{cleveref}
\usepackage{apptools}
\usepackage{enumitem}

\usepackage{graphicx}
\usepackage{subfig}
\usepackage{fontawesome5}
\usepackage[T1]{fontenc}
\usepackage{fancybox}

\AtAppendix{\counterwithin{lem}{section}}

\newtheorem{thm}{Theorem}

\newtheorem{lem}{Lemma}

\newtheorem{ass}{Assumption}

\def \y {\mathbf{y}}
\def \YB {\mathbf{Y}}
\def \E {\mathbb{E}}
\def \x {\mathbf{x}}
\def \XB {\mathbf{X}}

\def \v {\mathbf{v}}
\def \V {\mathbf{V}}
\def \g {\mathbf{g}}

\def \DB {\mathbf{D}}

\def \OB {\mathbf{O}}

\def \z {\mathbf{z}}
\def \ZB {\mathbf{Z}}

\def \u {\mathbf{u}}
\def \U {\mathbf{U}}

\def \w {\mathbf{w}}
\def \WB {\mathbf{W}}
\def \R {\mathbb{R}}

\def \F {\mathcal{F}}

\def \W {\mathcal{W}}

\def \diff {\mathrm{d}}

\def \q {\mathbf{q}}

\def \IB {\mathbf{I}}

\def \a {\mathbf{a}}
\def \b {\mathbf{b}}

\def \BB {\mathbf{B}}
\def \BBT {\widetilde{\BB}}

\def \SBB {\mathbb{S}}

\def \Gb {\mathbf{G}}

\def \Fn {\textnormal{F}}

\def \s {\mathbf{s}}
\def \SB {\mathbf{S}}

\def \lb {\mathbf{l}}
\def \Lb {\mathbf{L}}

\def \m {\mathbf{m}}
\def \MB {\mathbf{M}}
\def \MBT {\widetilde{\MB}}

\def \n {\mathbf{n}}
\def \NB {\mathbf{N}}
\def \NBT {\widetilde{\NB}}
\def \K {\mathcal{K}}

\def \EB {\mathbf{E}}

\def \N {\mathbb{N}}

\def \ind {\mathbb{I}}

\newcommand{\inner}[2]{\left\langle #1, #2 \right\rangle}

\DeclareMathOperator*{\argmin}{argmin}

\DeclareMathOperator*{\sgn}{sign}
\DeclareMathOperator*{\msgn}{msign}
\DeclareMathOperator*{\diagop}{diag}

\DeclareMathOperator*{\rk}{rank}

\DeclareMathOperator*{\trace}{tr}

\newcommand{\norm}[1]{\left\Vert#1\right\Vert}

\newcommand{\abs}[1]{\left|#1\right|}
\newcommand{\mabs}[1]{\left|#1\right|_{\textnormal{m}}}
\newcommand{\tr}[1]{\trace\left(#1\right)}
\newcommand{\diag}[1]{\diagop\left(#1\right)}
\newcommand{\sign}[1]{\sgn\left(#1\right)}
\newcommand{\msign}[1]{\msgn\left(#1\right)}
\newcommand{\rank}[1]{\rk\left(#1\right)}

\def \eps {\bm{\epsilon}}

\def \bsigma {\bm{\sigma}}
\def \bSigma {\bm{\Sigma}}
\def \bnabla {\bm{\nabla}}
\def \bLambda {\bm{\Lambda}}
\def \op {\textnormal{op}}
\newcommand{\sqbrac}[1]{\left[#1\right]}
\newcommand{\brac}[1]{\left(#1\right)}
\newcommand{\cbrac}[1]{\left\{#1\right\}}

\crefname{ass}{Assumption}{Assumptions}
\crefname{defn}{Definition}{Definitions}
\crefname{lem}{Lemma}{Lemmas}
\crefname{thm}{Theorem}{Theorems}
\crefname{appendix}{Appendix}{Appendices}

\def\ceil#1{\left\lceil #1 \right\rceil}

\def\1{\bm{1}}
\def\0{\bm{0}}

\usepackage{etoolbox}

\newcommand{\squeezeTheorem}{\vspace{-2.5pt}} 

\BeforeBeginEnvironment{thm}{\squeezeTheorem}
\AfterEndEnvironment{thm}{\squeezeTheorem}
\BeforeBeginEnvironment{lem}{\squeezeTheorem}
\AfterEndEnvironment{lem}{\squeezeTheorem}
\BeforeBeginEnvironment{ass}{\squeezeTheorem}
\AfterEndEnvironment{ass}{\squeezeTheorem}
\BeforeBeginEnvironment{defn}{\squeezeTheorem}
\AfterEndEnvironment{defn}{\squeezeTheorem}
\BeforeBeginEnvironment{rem}{\squeezeTheorem}
\AfterEndEnvironment{rem}{\squeezeTheorem}
\BeforeBeginEnvironment{cor}{\squeezeTheorem}
\AfterEndEnvironment{cor}{\squeezeTheorem}

\allowdisplaybreaks[4]

\usepackage{hyperref}

\makeatletter
\let\oldass\ass
\let\endoldass\endass
\renewenvironment{ass}[1][]{%
  \phantomsection 
  \oldass[#1]%
}{%
  \endoldass
}
\makeatother